\documentclass[letterpaper]{article} 
\usepackage{aaai24}  
\usepackage{times}  
\usepackage{helvet}  
\usepackage{courier}  
\usepackage[hyphens]{url}  
\usepackage{graphicx} 
\usepackage{natbib}  
\usepackage{caption} 
\usepackage{algorithm}
\usepackage{algorithmic}
\usepackage{amsmath}
\usepackage{amsthm}
\usepackage{amssymb}
\usepackage{diagbox}
\usepackage{makecell}
\usepackage{pifont}
\usepackage{graphicx} 
\usepackage{tikz}
\usepackage{xcolor}   
\usepackage{amsthm,amsmath,amssymb,amsfonts}
\usepackage{mathtools}
\usepackage{booktabs}
\usepackage{multirow}

\usepackage{subcaption}
\definecolor{mycolor}{RGB}{120, 150, 193}
\usetikzlibrary{arrows.meta}
\tikzset{>=latex}

\usepackage{newfloat}
\usepackage{listings}

\newtheorem{proposition}{Proposition}
\newtheorem{definition}{Definition}

\newtheorem{remark}{Remark}

\newtheorem{lemma}{Lemma}

\newcommand{\xb}{\mathbf{x}}
\newcommand{\ab}{\mathbf{a}}

\newcommand{\vb}{\mathbf{v}}
\newcommand{\yb}{\mathbf{y}}

\newcommand{\wb}{\mathbf{w}}

\newcommand{\dr}{\mathrm{d}}
\newcommand{\PM}{\mathcal{P}}
\newcommand{\LM}{\mathcal{L}}
\newcommand{\mut}{\tilde \mu}
\newcommand{\EBB}{\mathbb{E}}
\newcommand{\FM}{\mathcal{F}}
\newcommand{\Kb}{\mathbf{K}}
\newcommand{\RBB}{\mathbb{R}}
\newcommand{\NAME}{GAD}

\DeclareCaptionStyle{ruled}{labelfont=normalfont,labelsep=colon,strut=off} 
\floatstyle{ruled}
\newfloat{listing}{tb}{lst}{}
\floatname{listing}{Listing}
\title{\NAME-PVI : A General Accelerated Dynamic-Weight Particle-based Variational Inference Framework}

\author{
    Fangyikang Wang\textsuperscript{\rm 1},
    Huminhao Zhu\textsuperscript{\rm 1},
    Chao Zhang\thanks{Corresponding author.}\textsuperscript{\rm 1,\rm 2},
    Hanbin Zhao\textsuperscript{\rm 1,\rm 2},
    Hui Qian\textsuperscript{\rm 1,\rm 3}
}
\affiliations{

    \textsuperscript{\rm 1}College of Computer Science and Technology, Zhejiang University\\
    \textsuperscript{\rm 2}Advanced Technology Institute, Zhejiang University\\
    \textsuperscript{\rm 3}State Key Lab of CAD\&CG, Zhejiang University\\
    \{wangfangyikang, zhuhuminhao, zczju, zhaohanbin, qianhui\}@zju.edu.cn
}

\usepackage{bibentry}

\begin{document}

\maketitle

\begin{abstract}
    Particle-based Variational Inference (ParVI) methods  approximate the target distribution by iteratively evolving finite weighted particle systems. Recent advances of ParVI methods reveal the benefits
    of accelerated position update strategies and dynamic weight adjustment approaches.
    In this paper, we propose the first ParVI framework that possesses both accelerated position update
    and dynamical weight adjustment simultaneously,
    named the General Accelerated Dynamic-Weight Particle-based Variational Inference (GAD-PVI) framework.
    Generally, GAD-PVI simulates the semi-Hamiltonian gradient flow on a novel Information-Fisher-Rao space, which yields an additional decrease on the local functional dissipation.
    GAD-PVI is compatible with different dissimilarity functionals and associated smoothing approaches under three information metrics.
    Experiments on both synthetic and real-world data demonstrate
    the faster convergence and reduced approximation error of GAD-PVI methods over the state-of-the-art.
\end{abstract}

\section{Introduction}
Particle-based Variational Inference (ParVI) methods have gained significant attention in the Bayesian inference literature
 owing to their effectiveness in providing approximations of the target posterior distribution $\pi$
 \cite{liu2016stein,zhu2020variance,De-randomizing_MCMC,dpvi,li2023sampling}.
The essence of ParVI lies in deterministically evolving a system of finite weighted particles 
by simulating the probability space gradient flow of 
certain dissimilarity functional $\FM(\mu) := \mathcal D(\mu|\pi)$ 
vanishing at $\mu = \pi$ \cite{liu2019understanding}.
Since the seminal work Stein Variational Gradient Descent (SVGD) \cite{liu2016stein},
classical ParVI focus on simulating the \emph{first-order} gradient flow in \emph{Wasserstein} space.
By using different dissimilarity and associated smoothing approaches,
various effective ParVI methods have been proposed, including
the BLOB method\cite{chen2018unified}, the GFSD method\cite{liu2019understanding}, 
and the KSDD method\cite{korba2021kernel}.


\begin{table*}[h]
	\centering
    \small
	\begin{tabular}{l|cccc}
		\toprule
        \diagbox{Methods}{Features} 
        
        & \makecell[c]{Accelerated \\position update}  
        & \makecell[c]{Dynamic \\weight adjustment} 
        & \makecell[c]{Dissimilarity and \\associated smoothing approach}
        & \makecell[c]{Underlying \\probability space}  \\

		\hline
		\makecell[l]{SVGD\gape{\footnotemark}  \cite{liu2016stein}
        } 
        & \makecell[c]{\ding{55}}  
        & \makecell[c]{\ding{55}} 
        & KL-RKHS
        & Wasserstein  \\

		\makecell[l]{BLOB \cite{craig2016blob}
        } 
        & \makecell[c]{\ding{55}}  
        & \makecell[c]{\ding{55}} 
        & KL-BLOB
        & Wasserstein  \\

		\makecell[l]{KSDD \cite{korba2021kernel} 
        } 
        
        & \makecell[c]{\ding{55}}  
        & \makecell[c]{\ding{55}} 
        & KSD-KSDD 
        & Wasserstein  \\
    
        \hline
        \makecell[l]{   ACCEL \cite{taghvaei2019accelerated}\\ 
        }  
        & \makecell[c]{\ding{51}}  
        & \makecell[c]{\ding{55}} 
        & KL-GFSD
        & Wasserstein  \\

        \makecell[l]{WNES, WAG \cite{liu2019understanding}\\ 
        }  
        & \makecell[c]{\ding{51}}  
        & \makecell[c]{\ding{55}} 
        & General
        & Wasserstein  \\
    
        \makecell[l]{AIG \cite{AIG} 
        }  
        
        & \makecell[c]{\ding{51} }
        & \makecell[c]{\ding{55}} 
        & KL-GFSD
        &  \makecell[c]{Information (General)} \\  
        
        \hline
        \makecell[l]{DPVI \cite{dpvi}\\ 
        }  
        
        & \ding{55} 
        & \ding{51} & General
        & WFR  \\  
    
        \hline
        \makecell[l]{\NAME-PVI (Ours)  
        } 
        
        & \makecell[c]{\ding{51} }
        & \makecell[c]{\ding{51}} 
        & General
        & \makecell[c]{IFR (General)} \\  
        
	\bottomrule
	\end{tabular}
	\caption{Feature-by-Feature comparison of different ParVIs.}
	\label{gmm_w2}
\end{table*}

To improve the efficiency of ParVIs,
recent works explore different aspects of the underlying geometry structures in the probability space
and design two types of refined particle systems with either \emph{accelerated position update}
or \emph{dynamic weight adjustment}.
\begin{itemize}
\item \emph{Accelerated position update.}
By considering the second-order Riemannian information of the Wasserstein probability space,
different accelerated position update strategies have been proposed\cite{liu2019understanding,taghvaei2019accelerated}:
\citet{liu2019understanding} follows the accelerated gradient descend methods in the Wasserstein propability space \cite{liu2017accelerated,zhang2018estimate} and derives the WNES and WAG methods,
which update the particles' positions with an extra momentum;
the ACCEL method\cite{taghvaei2019accelerated} directly discretizes the Hamiltonian gradient flow in the Wasserstein space and update the position with the damped velocity field, 
which effectively decrease the Hamiltonian potential of the particle system. 
Later, \citet{AIG} consider the Hamiltonian gradient flow for general information probability space\cite{density_manifold},
and derive novel accelerated position update strategies according to the Kalman-Wasserstein/Stein Hamiltonian flow.
They theoretically show that the Hamiltonian flow usually has a faster convergence to the equilibrium compared with the original first-order counterpart under mild condition.
Numerous experimental studies demonstrate that these accelerated position update strategies usually drift the particle system to the target distribution more efficiently\cite{liu2019understanding,taghvaei2019accelerated,carrillo2019convergence,AIG}.  
\item \emph{Dynamic weight adjustment.}
Delving into the orthogonality structure of the Wasserstein-Fisher-Rao(WFR) space,
\citet{dpvi} develop the first dynamic-weight ParVI (DPVI) methods.
Specifically, they derive effective dynamical weight adjustment approaches by mimicing the reaction variational step in a JKO splitting scheme of first-order WFR gradient flow\cite{gallouet2017jko,rotskoff2019global}.
Compared with the commonly used fixed weight strategy, these dynamical weight adjustment schemes usaully lead to less approximation error,
especially when the number of particles is limited\cite{dpvi}. 
\end{itemize}



\noindent\textbf{Contribution:}In this paper, we propose the first ParVI methods which possess both accelerated position update and dynamical weight adjustment simultaneously.
Specifically, we first construct a novel Information-Fisher-Rao (IFR) probability space, which augment the origianl information space with an orthogonal Fisher-Rao structure.\footnotetext{SVGD methods can be also viewed as using KL dissimilarity without Smoothing Approach under Stein-Metric\cite{nusken2021stein}.}
Then, we originate a novel Semi-Hamiltonian IFR (SHIFR) flow in this space, which simplifies the influence of the kinetic energy on the velocity feild in the Hamiltonian IFR flow
\footnote{Though the Hamiltonian IFR flow seems a natural choice, it is generally infeasible to obtain practical algorithm by discretizing this flow. Please check the Appendix A.2\ref{full_Hamiltonian} for a detailed discussion of IFR Hamiltonian flow.}.
By discretizing the SHIFR flow, a practical General Accelerated Dynamic-weight 
Particle-based Variational Inference (GAD-PVI) framework is proposed.
The main contribution of our paper are listed as follows:

\begin{itemize}
    \item 
    We investigate the convergence property of the SHIFR flow
    and show that the target distribution $\pi$
    is the stationary distribution of the proposed semi-Hamiltonian flow for proper dissimilarity functional
    $\mathcal D(\cdot|\pi)$.
    Moreover, our theoretical result also shows that the augmented Fisher-Rao structure yields an additional decrease on the  local functional dissipation, compared to the Hamiltonian flow in the vanilla information space.


    \item 
    We derive an effective finite-particle approximation
    to the SHIFR flow, which directly evolves the position, weight, and velocity of the particles via a set of ordinary differential equations.
    The finite particle system is compatible with different dissimilarity and associated smoothing approaches.
    We prove that the mean-field limit of the proposed particle system 
    converges to the exact SHIFR flow under mild condition.
    

    \item 
    By adopting explicit Euler discretization to the finite-particle system,
    we architect the General Accelerated Dynamic-weight Particle-based Variational 
    Inference (GAD-PVI) framework,
    which update positions in an acceleration manner and dynamically adjust weights.
    We derive nine GAD-PVI instances by using three different dissimilarity functionals and associated smoothing approaches(KL-BLOB, KL-GFSD and KSD-KSDD) on the Wasserstein/Kalman-Wasserstein/Stein IFR space, respectively.
\end{itemize}
    We evaluate our algorithms on various synthetic and real-world tasks.
    The empirical results demonstrate the superiority of our GAD-PVI methods.

~\\
\noindent
\textbf{Notation.} 
Given a probability measure $\mu$ on $\RBB^d$, we denote $\mu \in \PM_2(\RBB^d)$ if its second moment is finite.
For a given functional $\FM(\cdot):\PM_2(\RBB^d) \to \RBB$, 
$\frac{\delta \FM(\tilde{\mu})}{\delta \mu}(\cdot):\RBB^d \to \RBB$ denote its first variation at $\mu=\tilde{\mu}$.
We use $C(\RBB^n)$ to denote the set of continuous functions map from $\mathbb{R}^n$ to $\mathbb{R}$.
We denote $\xb^i \in \RBB^d$ as the $i$-th particle, for $i\in \{1...M\}$.
We denote the Dirac delta distribution with point mass located at $\xb^i$ as $\delta_{\xb^i}$, 
and use $f*g:\RBB^d \to \RBB$ to denote the convolution operation 
between $f:\RBB^d \to \RBB$ and $g:\RBB^d \to \RBB$. 
Besides, we use $\nabla$ and $\nabla \cdot ()$ to denote 
the gradient and the divergence operator, respectively.
We denote a general information probability space as $(\mathcal{P}(\RBB^n),G(\mu))$,
where $G(\mu)[\cdot]$ denotes the one-to-one information metric tensor mapping elements in the tangent space $T_\mu\mathcal{P}(\RBB^n) \subset C(\RBB^n)$ 
to the cotangent space $T^*_\mu\mathcal{P}(\RBB^n) \subset C(\RBB^n)$.
The inverse map of $G(\mu)[\cdot]$ is denoted 
as $G^{-1}(\mu)[\cdot]:T^*_\mu\mathcal{P}(\RBB^n)\to T_\mu\mathcal{P}(\RBB^n)$. 
\section{Preliminaries}
When dealing with Bayesian inference tasks, 
variational inference methods approximate the target posterior $\pi$ with 
an easy-to-sample distribution $\mu$, and recast the inference task as an optimization problem 
over $\mathcal{P}_2(\mathbb{R}^d)$ \cite{ranganath2014black}: 
\begin{align}
    \text{min}_{\mu\in \mathcal{P}_2(\RBB^n)}\mathcal{F}(\mu):=\mathcal{D}(\mu|\pi).
\end{align}
To solve this optimization problem, 
Particle-based Variational Inference (ParVI) methods generally simulate the gradient 
flow of $\mathcal{F}(\mu)$ in certain probability space with a finite particle system,
which transport the initial empirical distribution towards the target 
distribution $\pi$ iteratively.
Given an information metric tensor $G(\mu)[\cdot]$, 
the gradient flow in the information probability space $(\mathcal{P}(\RBB^n),G(\mu))$ takes the following form\cite{ambrosio2008gradient}:
\begin{align}\label{gf_general}
    \partial_t \mu_t = -G(\mu_t)^{-1} \left[\frac{\delta \mathcal{F}(\mu_t) }{\delta \mu}\right].
\end{align}


\subsection{Wasserstein Gradient Flow and Classical ParVIs}
Since the seminal work Stein Variational Gradient Descent (SVGD)\cite{liu2016stein},
many ParVI methods focus on flows in the Wasserstein space, where the inverse of the Wasserstein metric tensor writes
\begin{align}\label{Wasserstein_metric}
    G^{W}(\mu)^{-1}\left[\Phi\right] = -\nabla \cdot (\mu\nabla\Phi),\Phi \in T_\mu^* \mathcal{P}(\RBB^n),
\end{align}
and the Wasserstein gradient flow is defined as
\begin{align}\label{gf_Wasserstein}
    \partial_t \mu_t = 
    \nabla \cdot (\mu_t\nabla \frac{\delta \mathcal{F}(\mu_t) }{\delta \mu}).
\end{align}
Based on the probability flow \eqref{gf_Wasserstein} on the density,
existing ParVIs maintain a set of particles ${\xb^i_t}$ and directly modify the particle position according to the following ordinary differential equation
\begin{align}\label{ode_wgf}
    \mathrm{d}\xb^i_t = \nabla \frac{\delta \mathcal{F}(\tilde\mu_t) }{\delta \mu} (\xb^i_t)\mathrm{d} t,
\end{align}
where $\tilde{\mu}_t = \sum_{i=1}^M w^i_t \delta_{\xb^i_t}$  denotes the empirical distribution.
Since the first total variation $\frac{\delta \mathcal{F}(\tilde\mu_t)}{\delta \mu}$  of $\mathcal{F}$ might be not well-defined for the discrete empirical distribution,
various ParVI methods have proposed by choosing different dissimilarity $\mathcal{F}$ and associated smoothing approaches for $\frac{\delta \mathcal{F}(\tilde{\mu_t}) }{\delta \mu}$, e.g., KL-BLOB\cite{chen2018unified}, KL-GFSD\cite{liu2019understanding}, 
and KSD-KSDD\cite{korba2021kernel}.

\subsection{Hamiltonian Gradient Flows and Accelerated ParVIs}
The following Hamiltonian gradient flow in the general information probability space has recently been utilized to derive more efficient ParVI methods
\begin{align}\label{AIG_general}
   \left\{
	\begin{aligned}
		&\partial_t\mu_t= \frac{\delta}{\delta\Phi}\!\mathcal{H}(\mu_t,\Phi_t), \\
        &\partial_t\Phi_t \! =\!- \! \gamma_t\Phi_t \! - \! \frac{\delta}{\delta\mu}\!\mathcal{H}(\mu_t,\Phi_t), 
    \end{aligned}\right.
\end{align}
where $\Phi_t$ denote the Hamiltonian velocity and $\mathcal{H}(\mu_t,\Phi_t) = \frac{1}{2}\int\Phi_tG(\mu_t)^{-1}\left[\Phi_t\right] dx + \mathcal{F}(\mu_t)$ denotes the Hamiltonian potential.
Note that the Hamiltonian flow \eqref{AIG_general} can be regarded as the second-order accelerated version 
of the information gradient flow \eqref{gf_general}, and usually converges faster to the equilibrium of the target distribution under mild condition\cite{carrillo2019convergence,taghvaei2019accelerated,AIG}. 
Though the form of the Hamiltonian flow \eqref{AIG_general} seems complicated,
it induces a simple augmented particle system $(\xb_t^i,\vb_t^i)$,
which evolves the position $\xb_t^i$  and velocity $\vb_t^i$ of particles simultaneously.
As the position update rule of $\xb_t^i$ also uses the extra velocity information, 
the induced system is said to have an accelerated position update.
By discretizing the continuous particle system, several accelerated ParVI methods have been proposed,
which converge faster to the target distribution in numerous real-world Bayesian inference tasks \cite{taghvaei2019accelerated,AIG}.

\subsection{Wasserstein-Fisher-Rao Flow and Dynamic-weight ParVIs}
Recently, the Wasserstein-Fisher-Rao (WFR) Flow has been used to derive effective dynamic weight adjustment approaches to mitigate the fixed-weight restriction of ParVIs\cite{dpvi}. 
The inverse of WFR metric tensor is
\begin{align}\label{wfr_metric}
    G^{WFR}(\mu)^{-1}\left[\Phi\right] = -\nabla \cdot (\mu\nabla\Phi)+(\Phi-\int \Phi)\mu,
\end{align}
where $\Phi \in T^*_\mu\mathcal{P}(\RBB^n)$,
and the WFR gradient flow writes:
\begin{align}\label{gf_wfr}
    \small
    \partial_t \mu_t \!= \!
    \underbrace{\nabla \!\cdot\! (\mu_t\nabla \frac{\delta \mathcal{F}(\mu_t) }{\delta \mu}) }_{\text{Wasserstein transport}}
    \!-\! \underbrace{(\frac{\delta\mathcal{F}(\mu_t)}{\delta \mu} \!-\! \int{\frac{\delta\mathcal{F}(\mu_t)}{\delta \mu}}\mathrm{d}\mu_t)\mu_t}_{\text{Fisher-Rao variational distortion}}.
\end{align}
Since the WFR space can be regarded as orthogonal sum of the Wasserstein space and the Fisher-Rao space,
\citet{dpvi} mimic a JKO splitting scheme for the WFR flow, 
which deal with the position and the weight with the Wasserstein transport and the Fisher-Rao variational distortion, respectively. 
Given a set of particles with position $\xb^i_t$ and weight $\wb_t^i$, the Fisher-Rao distortion can be approximated by the following ode
\begin{align}\label{FR-distortion}
    \frac{\mathrm{d}}{\mathrm{d}t} w^i_t &=\textstyle 
        -\left(\frac{\delta\mathcal{F}(\tilde{\mu}_t)}{\delta\mu}(\xb^i_t) 
        - \sum_{i=1}^M w^i_t \frac{\delta\mathcal{F}(\tilde{\mu}_t)}{\delta\mu}(\xb^i_t)\right)w^i_t.
\end{align}
According to the ode \eqref{FR-distortion}, \citet{dpvi} derive two dynamical weight-adjustment scheme and propose the Dynamic-Weight Particle-Based Variational Inference (DPVI) framework, which is compatible with several dissimilarity functionals and associated smoothing
approaches.

\section{Methodology}
In this section, we present our General Accelerated Dynamic-weight Particle-based Variational Inference (GAD-PVI) framework, detailted in Algorithm \ref{alg:algorithm1}.
We first introduce a novel augmented Information-Fisher-Rao space,
and originate the Semi-Hamiltonian-Information-Fisher-Rao (SHIFR) flow in the space.
The theoretical analysis on SHIFR shows that it usually possesses an additional decrease on the local functional dissipation compared to the Hamiltonian flow in the original information space.
Then, effective finite-particle systems, which directly evolve the position, weight, and velocity of the particles via a set of ordinary differential equations, are constructed based on SHIFR flows in several IFR spaces with different underlying information metric tensors.
We demonstrate that the mean-field limit of the constructed particle system exactly converges to the SHIFR flow in the corresponding probability space.
Next, we develop the GAD-PVI framework by discretizing these continuous-time finite-particles formulations, 
which enables simultaneous accelerated updates of particles' positions and dynamic adjustment of particles' weights.
We present nine effective GAD-PVI algorithms that use different underlying information metric tensors, 
dissimilarity functionals and the associated finite-particle smoothing approaches.

\subsection{Information-Fisher-Rao Space and Semi-Hamiltonian-Information-Fisher-Rao Flow}
To define the augmented Information-Fisher-Rao probability space, we introduce the Information-Fisher-Rao metric tensor $G^{IFR}(\mu)$, whose inverse is defined as follows.
\begin{align}\label{ifr_metric}
    G^{IFR}(\mu)^{-1}\left[\Phi\right] = G^{I}(\mu)^{-1}\left[\Phi\right]+(\Phi-\int \Phi \mathrm{d}\mu)\mu,
\end{align}
where $\Phi \in T^*_\mu\mathcal{P}(\RBB^n)$ and $G^{I}(\mu)$ denotes certain underlying information metric tensor.
Note that $G^{IFR}(\mu)$ is formed by the inf-convolution of $G^{I}(\mu)$ and Fisher-Rao metric tensor.

Based on $G^{IFR}(\mu)$, we introduce the following novel semi-Hamiltonian flow of $\FM$  on the Information-Fisher-Rao space $(\mathcal{P}(\RBB^n),G^{IFR}(\mu))$
\begin{align}\label{GAD-FLOW}
    \left\{
	\begin{aligned}
		&\partial_t\mu_t=\frac{\delta}{\delta\Phi}\!\mathcal{H}^{IFR}(\mu_t,\Phi_t), \\
        &\partial_t\Phi_t\!=\!-\!\gamma_t\Phi_t\!-\!\frac{1}{2}\frac{\delta}{\delta\mu}\!\left(\!\int\Phi_tG^I(\mu_t)^{-1}\left[\Phi_t\right] dx\!\right)\!\!-\!\frac{\delta \mathcal{F}(\mu_t) }{\delta \mu}.
    \end{aligned}\right.
\end{align}
where $\Phi_t$ denote the Hamiltonian velocity and 
\begin{align}
\mathcal{H}^{IFR}(\mu_t,\Phi_t) &= 
\underbrace{\frac{1}{2}\!\!\int\Phi_tG^I(\mu_t)^{-1}\left[\Phi_t\right] dx}_{\text{Information kinetic energy}}
\!\\
&+\underbrace{\frac{1}{2}\int\Phi_t(\Phi_t-\int\Phi d\mu_t)d\mu_t}_{\text{Fisher-Rao kinetic energy}}
+\underbrace{\frac{\delta \mathcal{F}(\mu_t) }{\delta \mu}}_{\text{potential energy}},\nonumber
\end{align}
denotes the Hamiltonian potential in the IFR space.
Compared to the full Hamiltonian flow of $\FM$ in the IFR space, 
the  SHIFR flow \eqref{GAD-FLOW} ignores the influence of the Fisher-Rao kinetic energy on the Hamiltonian field $\Phi_t$.
Later, we will show that SHIFR can be directly transformed into a particle system consisting of odes on the positions, velocities and weights of particles for proper underlying information metric tensor, while it is generally infeasible to obtain such a direct particle system according to the corresponding full Hamiltonian flow because it is difficult to handle the Fisher-Rao kinetic energy.
As the kinetic energy term vanishes when near the equilibrium of the flow, therefore it is acceptable for the SHIFR flow to neglect this intractable term and still has the target distribution $\pi$ as its stationary distribution. Moreover, this semi-Hamiltonian flow would 
converge faster compare to the Hamiltonian flow in the original information space on account of extra local descending property. 
Due to the limit of space, we defer the discussion of 
the stationary analysis and functional dissipation quantitative analysis of the SHIFR flow to Appendix A.4. Please refer to Proposition 2
and Proposition 3
for details.

With different underlying information metric tensor $G^I(\mu)$ in $\mathcal{H}^{IFR}(\mu_t,\Phi_t)$, we can obtain different SHIFR flows.
Suitable $G^I(\mu)$ includes the Wasserstein metric tensor,the Kalman-Wasserstein metric tensor (KW-metric) 
and the Stein metric tensor (S-metric).
For instance, the  SHIFR flow with Wasserstein metric (Wasserstein-SHIFR flow) writes:
\begin{align}\label{Gad-W-flow}
    \left\{
    \begin{aligned}
        &\partial_t \mu_t \!=-  \nabla\cdot\left(\mu_t\nabla \Phi_t\right) \!
        -\! \left(\frac{\delta\mathcal{F}(\mu_t)}{\delta\mu} \!-
        \! \int{\frac{\delta\mathcal{F}(\mu_t)}{\delta\mu}\mathrm{d}\mu_t}\right)\mu_t,\\
        &\partial_t\Phi_t \!=\!-\! \gamma_t\Phi_t\!-\!\left\lVert \nabla\Phi_t\right\rVert^2\!-\!\frac{\delta\mathcal{F}(\mu_t)}{\delta\mu}.
    \end{aligned}\right.
\end{align}
Note that in the subsequent section, we focus on the Wasserstein-SHIFR flow, and defer the detailed formulations with respect to KW-SHIFR and S-SHIFR to the Appendix B.1 and B.2 due to limited space.

\subsection{Finite-Particles Formulations to  SHIFR flows}
Now, we derive the finite-particle approximation to the SHIFR flow, which directly evolves the position $\xb^i_t$, weight $w^i_t$, and velocity $\vb^i_t$ of the particles. 
Specifically, we construct the following ordinary differential equation system to simulate the Wasserstein-SHIFR flow \eqref{Gad-W-flow}
\begin{align}\label{particle_formulation_w}
    \small
    \left\{
	\begin{aligned}
		\mathrm{d}\xb^i_t &=   \vb^i_t  \mathrm{d}t, \\
        \mathrm{d} \vb^i_t &= (-\gamma \vb^i_t - \nabla\frac{\delta\mathcal{F}(\tilde{\mu}_t)}{\delta\mu}(\xb^i_t) )\mathrm{d}t, \\
		\mathrm{d} w^i_t &=\textstyle -\left(\frac{\delta\mathcal{F}(\tilde{\mu}_t)}{\delta\mu}(\xb^i_t) 
        - \sum_{i=1}^M w^i_t \frac{\delta\mathcal{F}(\tilde{\mu}_t)}{\delta\mu}(\xb^i_t)\right)w^i_t\mathrm{d}t,\\
		\tilde{\mu}_t   &= \textstyle\sum_{i=1}^M w^i_t\delta_{\xb^i_t}.
	\end{aligned}\right.
\end{align}

The following proposition demonstrates that the mean-field limit of the particle system \eqref{particle_formulation_w} corresponds precisely to the Wasserstein-SHIFR flow in \eqref{Gad-W-flow}.
\begin{proposition}\label{proposition_1}
    Suppose the empirical distribution $\mut^M_0$ of $M$ weighted particles weakly converges to a 
    distribution $\mu_0$ when $M \to \infty$. 
    Then, the path of \eqref{particle_formulation_w} starting from $\mut^M_0$ and $\Phi_0$ with initial velocity $\mathbf{0}$
    weakly converges to a solution of the Wasserstein-SHIFR gradient flow \eqref{Gad-W-flow}
    starting from $\mu_t|_{t=0} = \mu_0 = $ and $\Phi_t|_{t=0} = \mathbf{0}$ as $M\to\infty$:
\end{proposition}

\begin{algorithm}[tb]
    \caption{General Accelerated Dynamic-weight Particle-based Variational Inference (GAD-PVI) framework}
    \label{alg:algorithm1}
    \textbf{Input}: 
    Initial distribution $\tilde{\mu}_0 = \sum_{i=1}^M w^i_0 \delta_{\xb^i_0}$, position adjusting step-size $\eta_{pos}$, 
    weight adjusting step-size $\eta_{wei}$, velocity field adjusting step-size $\eta_{vel}$,
    velocity damping parameter $\gamma$.	
    \begin{algorithmic}[1] 
    \STATE Choose a suitable functional $\FM$ and 
    its smoothing strategy $U_{\tilde{\mu}}$ 
     from KL-BLOB/KL-GFSD/KSD-KSDD
    \FOR{$k = 0,1,...,T-1$}
    \FOR{$i = 1,2,...,M$}
    \STATE Update positions $\xb^i_{k+1}$'s according to \eqref{dynamic_pos_w}.
    \ENDFOR 
    \FOR{$i = 1,2,...,M$}
    \STATE  Adjust velocity field $\vb^i_{k+1}$'s according to \eqref{dynamic_ac_w}.
    \ENDFOR
    \FOR{$i = 1,2,...,M$}
    \STATE  Adjust weights $w^i_{k+1}$'s according to \eqref{dynamic_weight}.
    \ENDFOR 
    \ENDFOR
    \STATE \textbf{Output}:$\tilde{\mu}_T = \sum_{i=1}^M w^i_T \delta_{\xb^i_T}$.
    \end{algorithmic}
    \end{algorithm}

\subsection{GAD-PVI Framework}
Generally, it is impossible to obtain an analytic solution of the continuous finite-particles formulations \eqref{particle_formulation_w}, 
thus a numerical integration method is required to derive an approximate solution.
Note that any numerical solver, such as the implicit Euler method \cite{platen2010numerical} 
and higher-order Runge-Kutta method \cite{butcher1964implicit} can be used.
Here, we follow the tradition of ParVIs to adopt the first-order explicit Euler discretization \cite{suli2003introduction} 
since it is efficient and easy-to-implement \cite{dpvi},
and propose our GAD-PVI framework, 
as listed in Algorithm \ref{alg:algorithm1}.

\paragraph{Dissimilarity Functionals and Smoothing Approaches}
To develop practical GAD-PVI methods, we must first select a dissimilarity functional $\FM$. 
The commonly used underlying functionals are 
KL-divergence \cite{liu2016stein,liu2019understanding,AIG} 
and Kernel-Stein-Discrepancy \cite{korba2020non}. 
Once a dissimilarity functional $\FM$ has been chosen, 
we need to select a smoothing approach to approximate 
the first variation of the empirical approximation, 
as the value of $\frac{\delta\FM(\cdot)}{\mu}$ at an 
empirical distribution $\tilde{\mu} = \sum_{i=1}^M w^i \delta_{\xb^i}$ is generally not well-defined. 
Smoothing strategies allow us to approximate the first variation value at the discrete empirical distribution.
Generally, a smoothed approximation to the first total variation is denoted as $U_{\tilde{\mu}}(\cdot) \thickapprox \frac{\delta\mathcal{F}(\tilde{\mu})}{\delta\mu}(\cdot)$.
The commonly used smoothing approaches in the ParVI area, 
namely BLOB (with KL-divergence as $\FM$) \cite{craig2016blob}, 
GFSD (with KL-divergence as $\FM$) \cite{liu2019understanding}, 
and KSDD (with Kernel Stein Discrepancy as $\FM$) \cite{korba2021kernel}, 
are all compatible with our GAD-PVI framework.

Here, we describe the dissimilarity functional KL-divergence and the associated BLOB smoothing approach as an example.
The first total variation of the KL is  
\begin{align}
\frac{\delta\FM(\mu)}{\delta\mu}(\cdot):=\frac{\delta KL(\mu|\pi)}{\delta\mu}(\cdot)=- \log{\pi}(\cdot)+\log\mu(\cdot).
\end{align}
As $\log{\mu(\xb)}$ is ill-defined for the discrete empirical distribution $\tilde{\mu}_k$, 
BLOB smoothing approach reformulate the intractable term $\log{\mu}$ as $\frac{\delta}{\delta \mu}\mathbb{E}_{\mu}\left[\log{\mu}\right]$ and smooth the density with a kernel function $K$, resulting in the approximation
    \begin{align}
        \log{\tilde \mu} &\approx  \frac{\delta}{\delta \tilde{\mu}} \mathbb{E}_{\tilde{\mu}}\left[\log{(\tilde{\mu}*K)}\right]
        \\
        &:=  \textstyle\log{\sum_{i=1}^Mw^iK(\cdot,\xb^i)} + \displaystyle\sum_{i=1}^M\frac{w^iK(\cdot, \xb^i)}{\sum_{j=1}^M w^j K(\xb^i, \xb^j)}.\nonumber
    \end{align}
 for a discrete density $\tilde{\mu}=\sum_{i=1}^{M} w^i\xb^i $.
This leads to the following approximation results: 
    \begin{align}
        \begin{aligned}
            U_{\tilde{\mu}_k} (\xb) = & - \log{\pi(\xb)} + \textstyle\log{\sum_{i=1}^M w^i_k K(\xb,\xb^i_{k})}  \\
            &+ \sum_{i=1}^M{\frac{w^i_kK(\xb,\xb^i_{k})}{\sum_{j=1}^M w^j_kK(\xb^i_{k},\xb^j_{k})}}.\label{blob-u}
        \end{aligned}
    \end{align}
    Details regarding other dissimilarity functionals and smoothing approaches are included in the Appendix B.3.

\paragraph{Updating rules}
Once the functional $\FM$ and its empirical 
approximation of the first variation 
$U_{\tilde{\mu}} \thickapprox \frac{\delta\mathcal{F}(\tilde{\mu})}{\delta\mu}$ is decided,
we adopt a Jacobi-type strategy to 
update the position $\xb^i_k$, velocity field $\vb^i_k$ and the weight $w^i_{k}$, i.e.,
the calculations in the $k+1$-th iteration are totally based on the variables obtained in the $k$-th iteration.
Therefore, starting from $M$ weighted particles located at $\{\xb_0^i\}_{i=1}^M$ 
with weights $\{w_0^i\}_{i=0}^M$ and $\{\vb_0^i=0\}_{i=0}^M$, GAD-PVI w.r.t. the Wasserstein-SHIFR flow first 
updates the positions of particles according to the following rule: 
\begin{align}\label{dynamic_pos_w}
    \xb^i_{k+1} = \xb^i_{k} + \eta_{pos} \vb_{k}^i.
\end{align}
Then, it adjusts the velocity field as 
\begin{align}\label{dynamic_ac_w}
	\vb_{k+1}^i=(1-\gamma\eta_{vel})\vb_k^i-\eta_{vel}\nabla U_{\tilde{\mu}_k}(\xb_k^i),
\end{align}
and particles' weights as following:
\begin{align}\label{dynamic_weight}
	w^i_{k+1} \!= w^i_k \!-\! \eta_{wei} (U_{\tilde{\mu}_k}(\xb_k^i)\!-\!
    \sum\nolimits_{j = 1}^{M}w_k^j U_{\tilde{\mu}_k}(\xb_k^j) ).
\end{align}
Here $\tilde{\mu}_k = \sum\nolimits_{i=1}^M w^i_k\delta_{\xb^i_{k}}$ denotes the empirical distribution, and 
$\eta_{pos}/\eta_{vel}/\eta_{wei}$ are the discretization stepsizes.
It can be verified that the total mass of $\tilde{\mu}_k$ is 
conserved
and $\tilde{\mu}_k$ remains a valid probability distribution 
during the whole procedure of GAD-PVI, i.e. $\sum_i w^i_k = 1$ for all $k$.
The detailed updating rules of GAD-PVI w.r.t. the  KW-SHIFR and S-SHIFR can be found in Appendix B.3. 


Notice that, compared to the classical ParVIs, the position acceleration scheme and 
dynamic-weight scheme only bring \emph{little}
extra computational cost, because the number of time-complexity-bottleneck operation, i.e.
calculation of $U_{\tilde{\mu}}$ and $\nabla U_{\tilde{\mu}}$, remains the same.

\paragraph{An alternative Weight Adjusting Approach.}
Except for Continuous Adjusting (CA) strategy,
the Duplicate/Kill (DK) strategy, which is a probabilistic discretization strategy 
to the Fisher-Rao part of \eqref{Gad-W-flow}, can also be adopted in GAD-PVI.
This strategy duplicates/kills particle $\xb^i_{k+1}$ according to an exponential clock with instantaneous rate: 
    \begin{align}\label{dk}
        R^i_{k+1} \!=\! -\! \eta_{wei} (\frac{\delta\mathcal{F}(\tilde{\mu}_k)}{\delta\mu}(\xb_k^i)\!-\!
        \sum_{j = 1}^{M}w_k^j \frac{\delta\mathcal{F}(\tilde{\mu_k})}{\delta\mu}(\xb_k^j) ).
    \end{align}
Specifically, if $R^i_{k+1}>0$, duplicate the particle $\xb^i_{k+1}$ with probability $1-\exp{(-R^i_{k+1})}$, and kill another one with 
    uniform probability to conserve the total mass; if $R^i_{k+1}<0$, kill the particle $\xb^i_{k+1}$ with probability $1 - \exp{(R^i_{k+1})}$, 
    and duplicate another one with uniform probability.
By replacing the CA strategy \eqref{dynamic_weight} 
in the GAD-PVI framework, we could obtain the DK variants of GAD-PVI methods.

\paragraph{GAD-PVI instances.}
With different underlying information metric tensors (W-metric, KW-metic and S-metric), weight adjustment approaches(CA and DK) and dissimilarity functionals/associated smoothing approaches(KL-BLOB, KL-GFSD and KSD-KSDD), we can derive 18 different intances of GAD-PVI, named as WGAD/KWGAD/SGAD-CA/DK-BLOB/GFSD/KSDD.

\section{Experiments}
In this section, we conduct empirical studies with our GAD-PVI algorithms. Here, we focus on the instances of GAD-PVI w.r.t. the W-SHIFR flows, i.e.,WGAD-CA/DK-BLOB/GFSD.
The experimental results on methods w.r.t. the KW-SHIFR and S-SHIFR flows are provided in the Appendix C\ref{additional_result}. 
Note that we do not include GAD-PVI methods with the KSDD smoothing approaches, as they are more computationally expensive and have been widely reported to be less stable\cite{korba2020non,dpvi}. 
We include four classes of methods as our baseline:
classical ParVI algorithms (SVGD, GFSD and BLOB),
the Nesterov accelerated ParVI algorithms (WNES-BLOB/GFSD),
the Hamiltonian accelerated ParVI algorithms
(WAIG-BLOB/GFSD) and
the Dynamic-weight ParVI algorithms (DPVI-CA/DK-BLOB/GFSD).

We compare the performance of these algorithms on two simulations,
 i.e., a 10-D Single-mode Gaussian model (SG) and a Gaussian mixture model (GMM),
and two real-world applications, i.e. Gaussian Process (GP) regression and Bayesian neural network (BNN).
For all the algorithms, the particles' weights are initialized to be equal.
In the first three experiments, we tune the parameters to achieve the best $W_2$ distance.
In the BNN task, we split $1/5$ of the training set as our validation set to tune the parameters.
Note that,
the position step-size are tuned via grid search for the fixed-weight ParVI algorithms, 
then used in the corresponding dynamic-weight algorithms.
The acceleration parameters and weight adjustment parameters
are tuned via grid search for each specific algorithm.
We repeat all the experiments 10 times and report the average results.
Due to limited space, only parts of the results are reported in this section.
We refer readers to the Appendix C\ref{additional_result} for the results on SG and additional results for GMM, GP and BNN.

	
	
	
	

\begin{figure}[ht!]
    \centering
    \begin{subfigure}{\linewidth}
        \includegraphics[width=.48\linewidth]{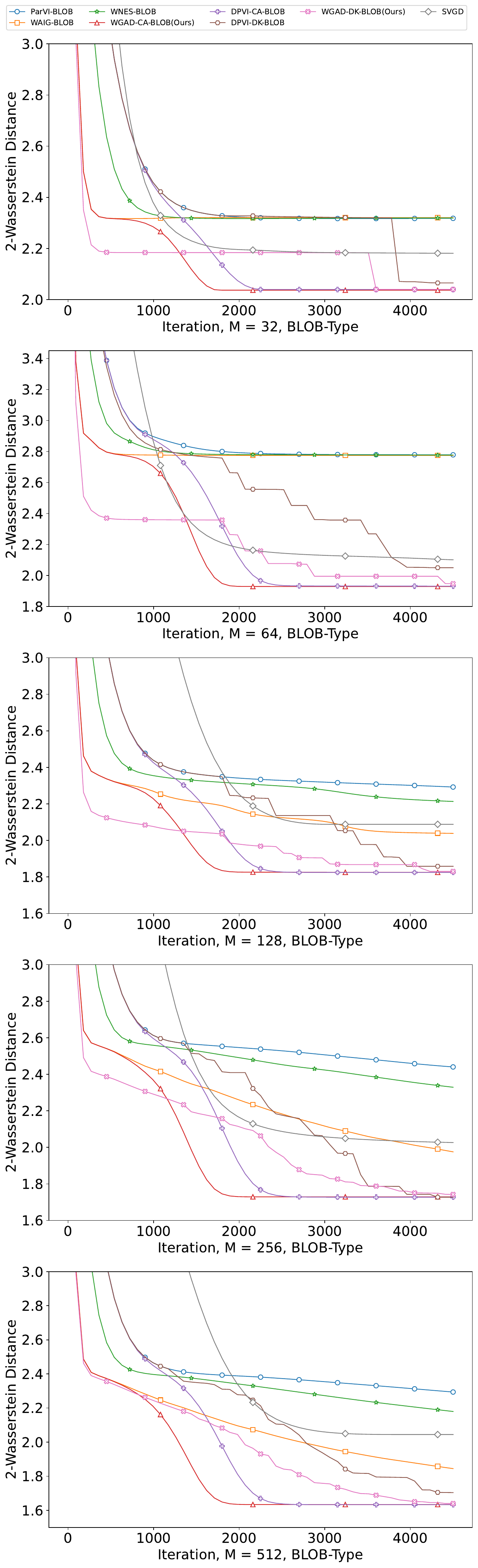}
        \includegraphics[width=.48\linewidth]{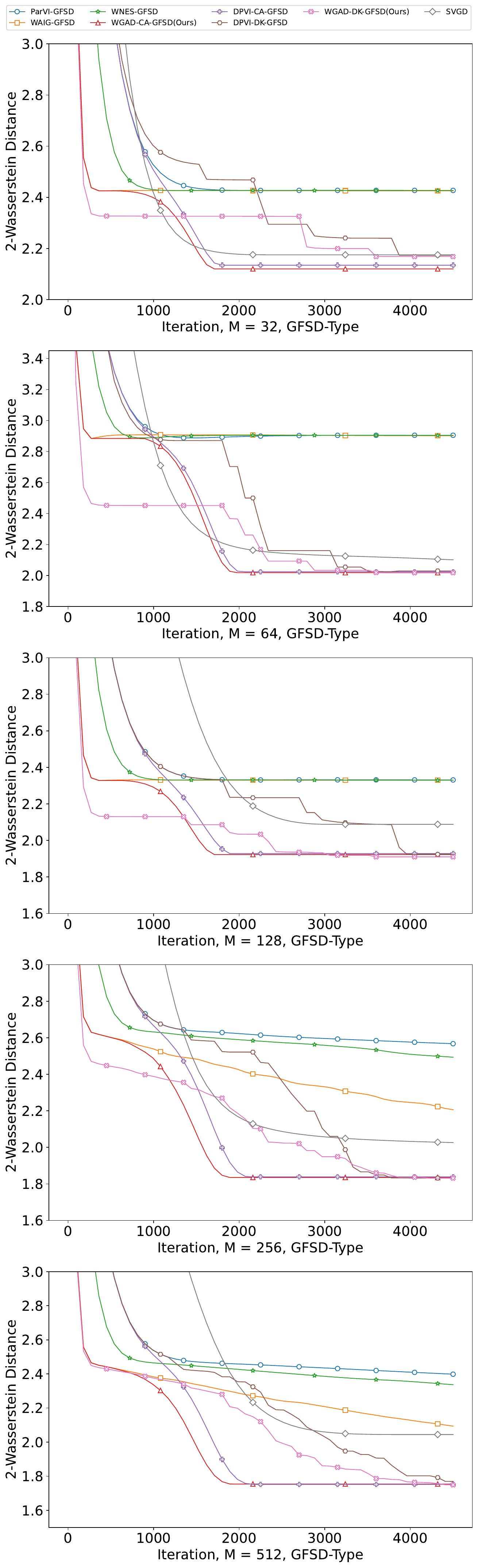}
    \end{subfigure}\qquad
    \caption{$W_2$ distance to the target w.r.t. iterations in the Gaussian Mixture Model task.}
    \vspace{-5mm}
    \label{figure_mg}
    \end{figure}

\subsection{Gaussian Mixture Model}\label{gmm_exp}
We consider approximating a 10-D Gaussian mixture model with two components, weighted by $1/3$ and $2/3$ respectively.
We run all algorithms with particle number $M\in\{32, 64, 128, 256, 512\}$.

In Figure \ref{figure_mg}, we report the 2-Wasserstein ($W_2$) distance between the 
empirical distribution generated by each algorithm and the target distribution  w.r.t. iterations
of different ParVI methods. 
We generate 5,000 samples from the target distribution $\pi$ as reference to evaluate the $W_2$ distance 
by using the POT library \footnote{http://jmlr.org/papers/v22/20-451.html}. 
The results demonstrate that our GAD-PVI algorithms consistently outperform their counterpart with only one (or none) of the accelerated position update strategy and dynamic weight adjustment approach.
Besides, the CA weight-adjustment approach usually result a lower $W_2$ compared to the DK scheme, and WGAD-CA-BLOB/GFSD usually have the fastest convergence and the lowest final $W_2$ distance to the target.

\begin{table}[tb]
	\centering
	 \resizebox{0.45\textwidth}{!}{
	\begin{tabular}{c|cc}
		\toprule
		\multirow{2}*{Algorithm} & \multicolumn{2}{c}{Smoothing Strategy}  \\
		~    & BLOB &GFSD \\
		\hline
		ParVI & 	1.570e-01 $\pm $ 2.210e-04  & 2.143e-01 $\pm $ 7.424e-04  \\
		WAIG & 	1.572e-01 $\pm $ 2.070e-04  & 2.142e-01 $\pm $ 7.048e-04 \\
		WNES & 	1.571e-01 $\pm $ 3.011e-04  & 2.138e-01 $\pm $ 7.771e-04 \\
        DPVI-DK & 	1.568e-01 $\pm $ 1.496e-03  & 2.142e-01 $\pm $ 2.712e-03 \\   
		DPVI-CA & 	1.285e-01 $\pm $ 2.960e-04 & 1.638e-01 $\pm $ 4.332e-04 \\
        \textbf{W\NAME-DK} & 	1.561e-01 $\pm $ 1.155e-03  & 2.142e-01 $\pm $ 1.501e-03 \\
        \textbf{W\NAME-CA} & 	\textbf{1.274e-01 $\pm $ 2.964e-04}  & \textbf{1.626e-01  $\pm $ 4.842e-04 }\\

        
		\bottomrule
	\end{tabular}}
	\caption{Averaged $W_2$ for the GP task with dataset LIDAR.}
	\label{gp_w2_table}
\end{table}

\begin{figure*}[ht!]
    \captionsetup[subfigure]{labelformat=empty}
    \centering
    \tiny
    \begin{subfigure}{.19\linewidth}
        \makebox[0pt][r]{\makebox[15pt]{\raisebox{25pt}{\rotatebox[origin=c]{90}{ParVI-BLOB}}}}%
        \includegraphics[width=\linewidth,height= .6\linewidth]{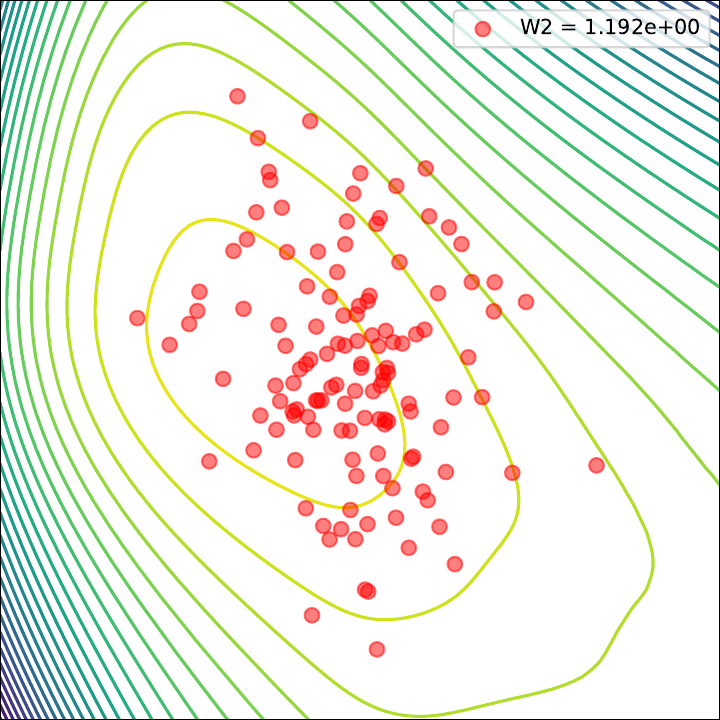}
    \end{subfigure}
    \begin{subfigure}{.19\linewidth}
        \includegraphics[width=\linewidth,height= .6\linewidth]{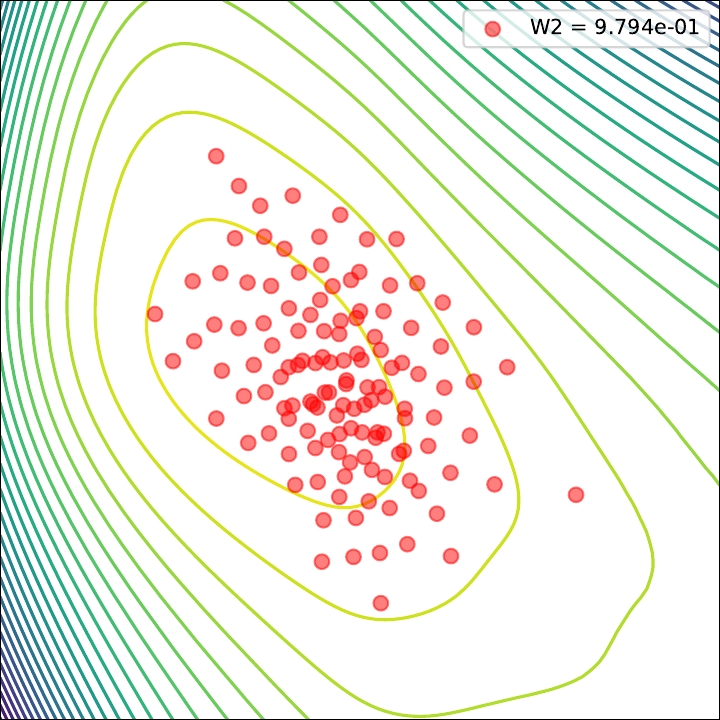}
    \end{subfigure}
    \begin{subfigure}{.19\linewidth}
        \includegraphics[width=\linewidth,height= .6\linewidth]{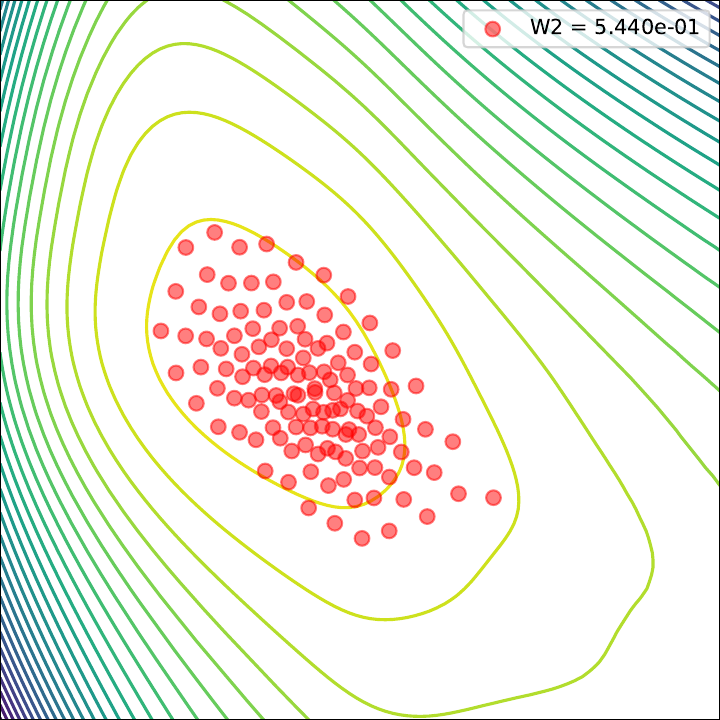}
    \end{subfigure}
    \begin{subfigure}{.19\linewidth}
        \includegraphics[width=\linewidth,height= .6\linewidth]{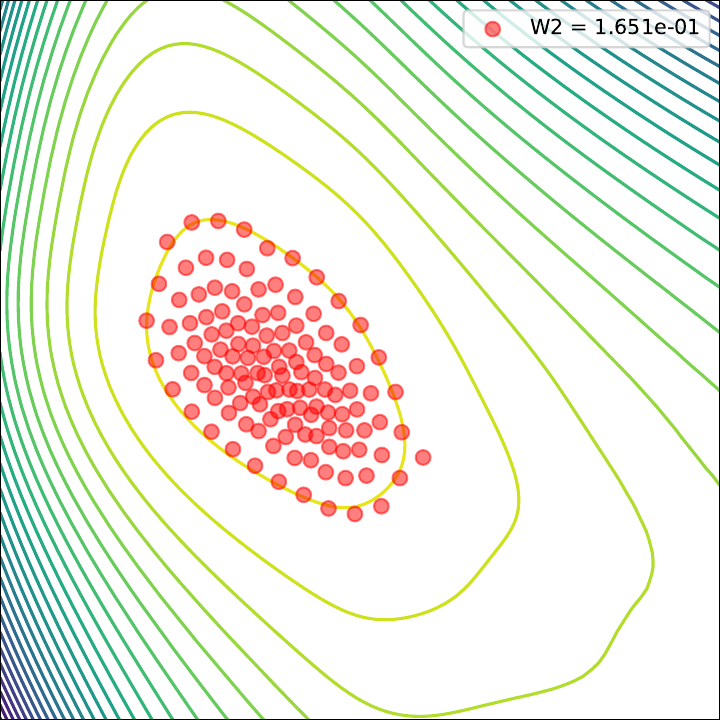}
    \end{subfigure}
    \begin{subfigure}{.19\linewidth}
        \includegraphics[width=\linewidth,height= .6\linewidth]{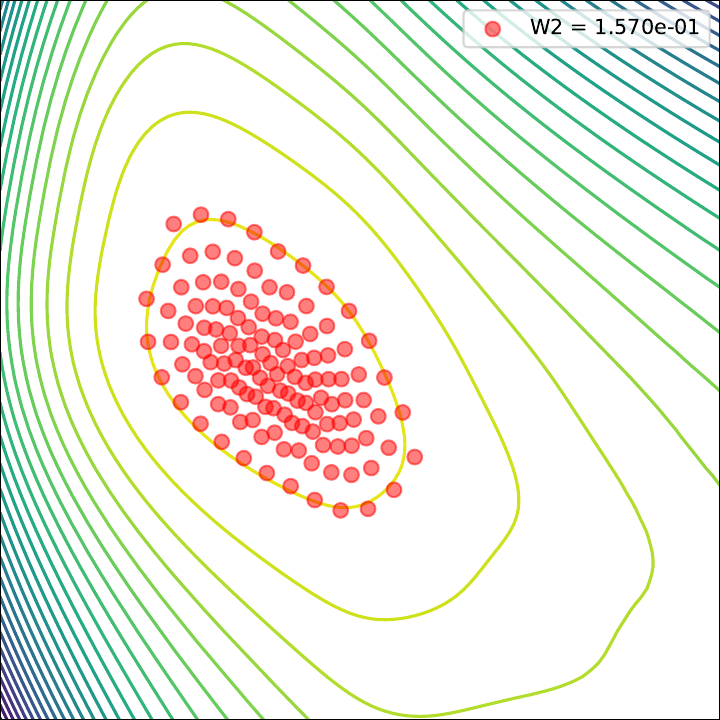}
    \end{subfigure}

    \makebox[0pt][r]{\makebox[15pt]{\raisebox{30pt}{\rotatebox[origin=c]{90}{WAIG-BLOB}}}}%
    \begin{subfigure}{.19\linewidth}
        \includegraphics[width=\linewidth,height= .6\linewidth]{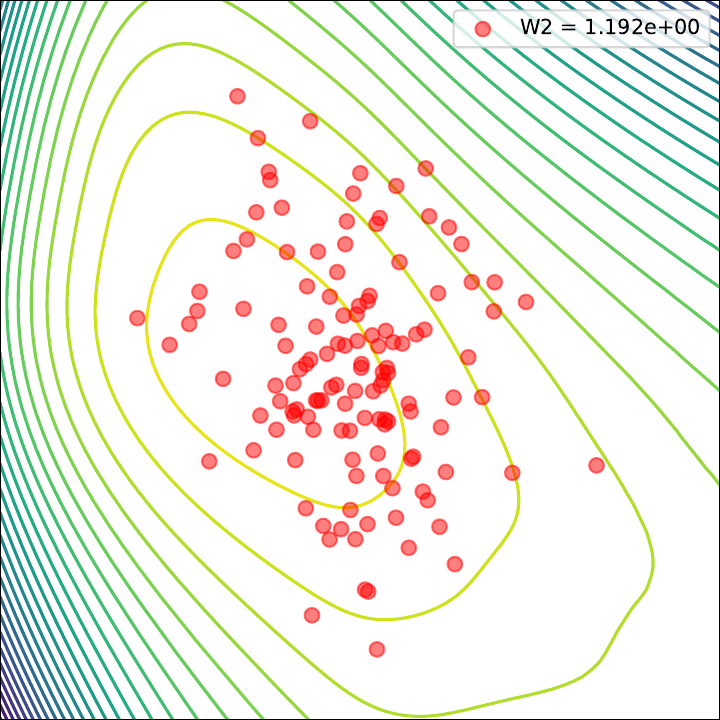}
    \end{subfigure}
    \begin{subfigure}{.19\linewidth}
        \includegraphics[width=\linewidth,height= .6\linewidth]{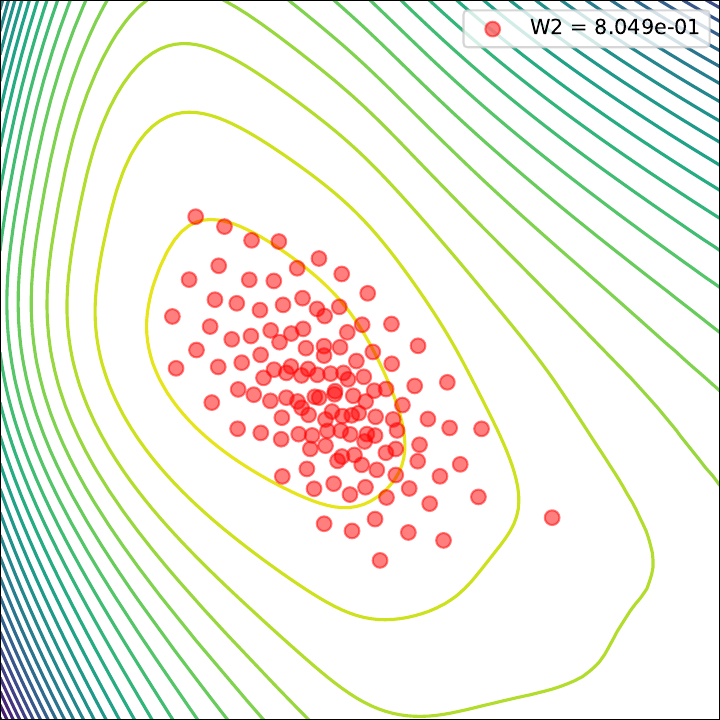}
    \end{subfigure}
    \begin{subfigure}{.19\linewidth}
        \includegraphics[width=\linewidth,height= .6\linewidth]{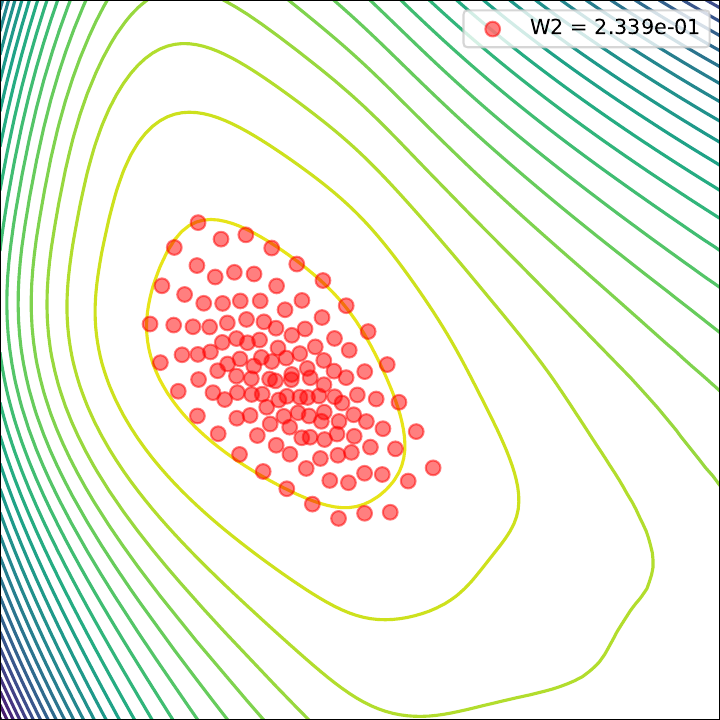}
    \end{subfigure}
    \begin{subfigure}{.19\linewidth}
        \includegraphics[width=\linewidth,height= .6\linewidth]{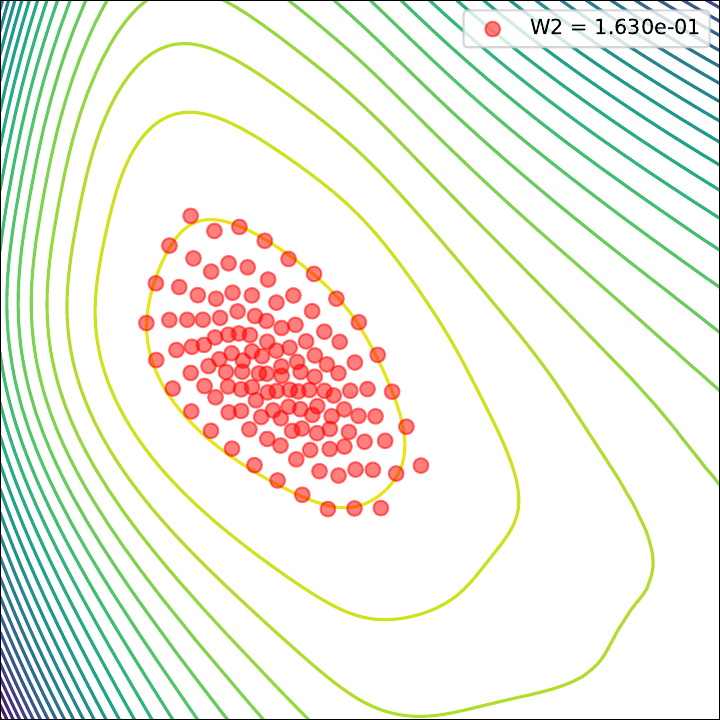}
    \end{subfigure}
    \begin{subfigure}{.19\linewidth}
        \includegraphics[width=\linewidth,height= .6\linewidth]{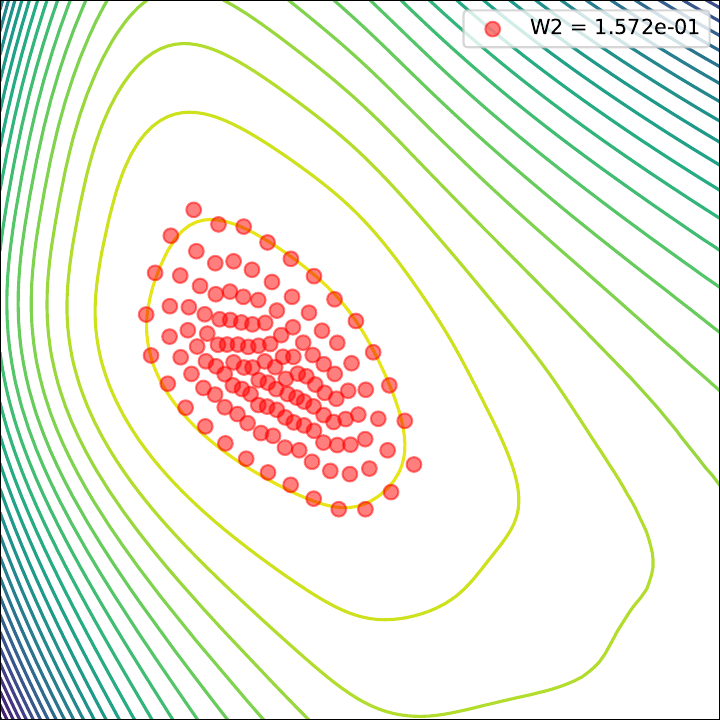}
    \end{subfigure}

    \makebox[0pt][r]{\makebox[15pt]{\raisebox{30pt}{\rotatebox[origin=c]{90}{DPVI-CA-BLOB}}}}%
    \begin{subfigure}{.19\linewidth}
        \includegraphics[width=\linewidth,height= .6\linewidth]{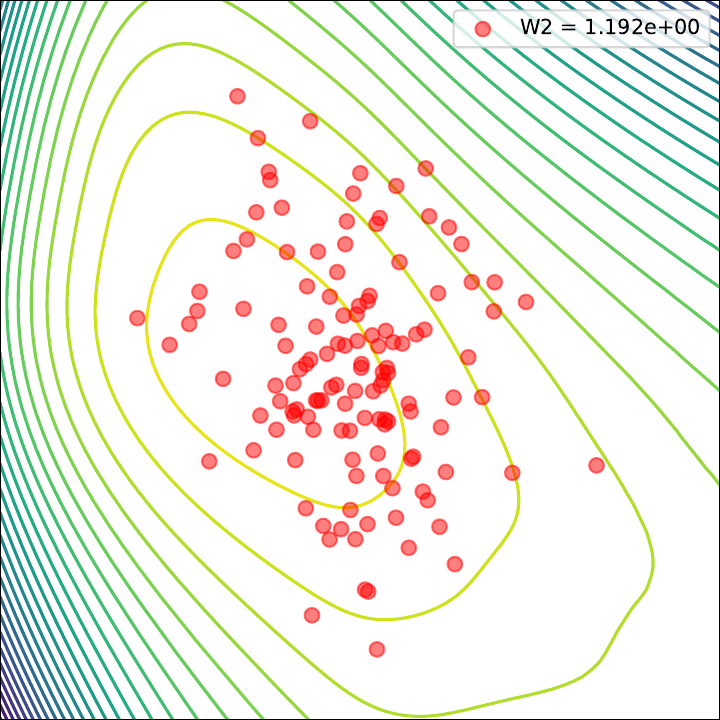}
    \end{subfigure}
    \begin{subfigure}{.19\linewidth}
        \includegraphics[width=\linewidth,height= .6\linewidth]{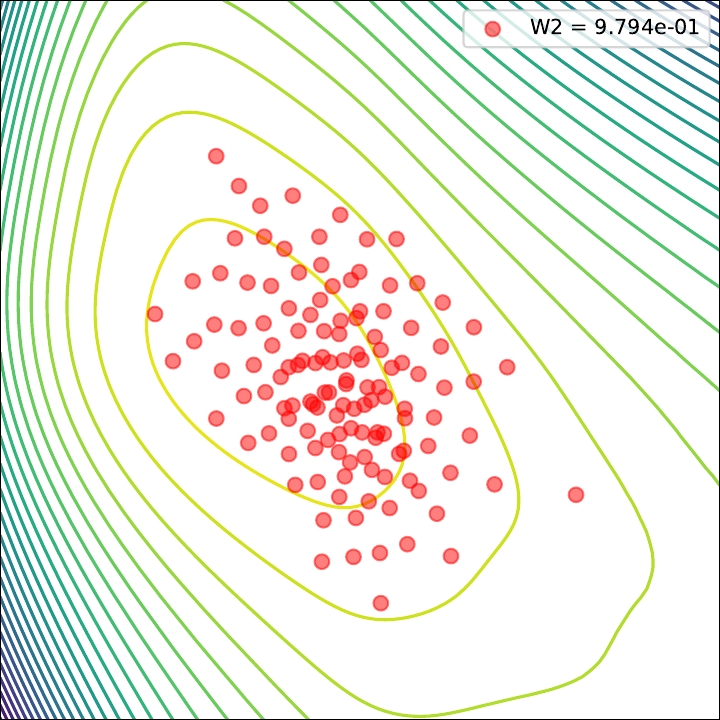}
    \end{subfigure}
    \begin{subfigure}{.19\linewidth}
        \includegraphics[width=\linewidth,height= .6\linewidth]{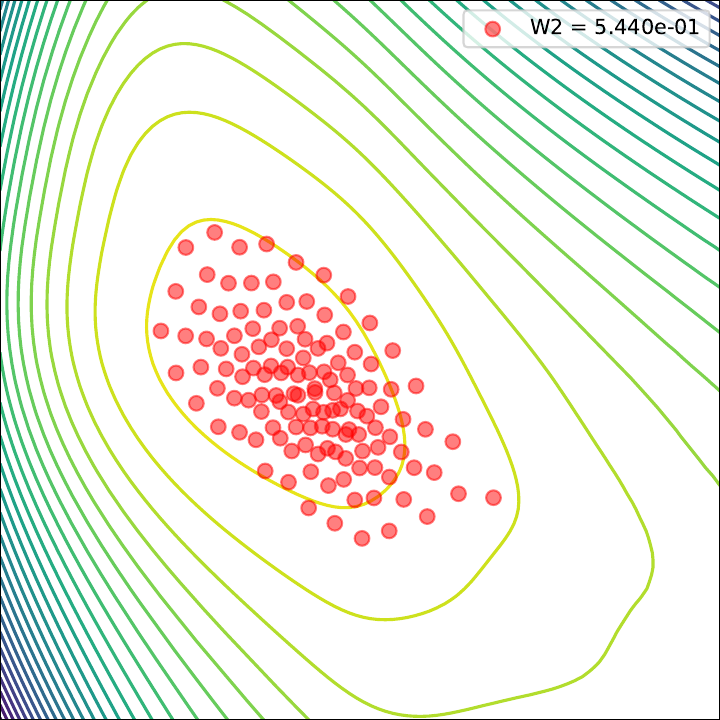}
    \end{subfigure}
    \begin{subfigure}{.19\linewidth}
        \includegraphics[width=\linewidth,height= .6\linewidth]{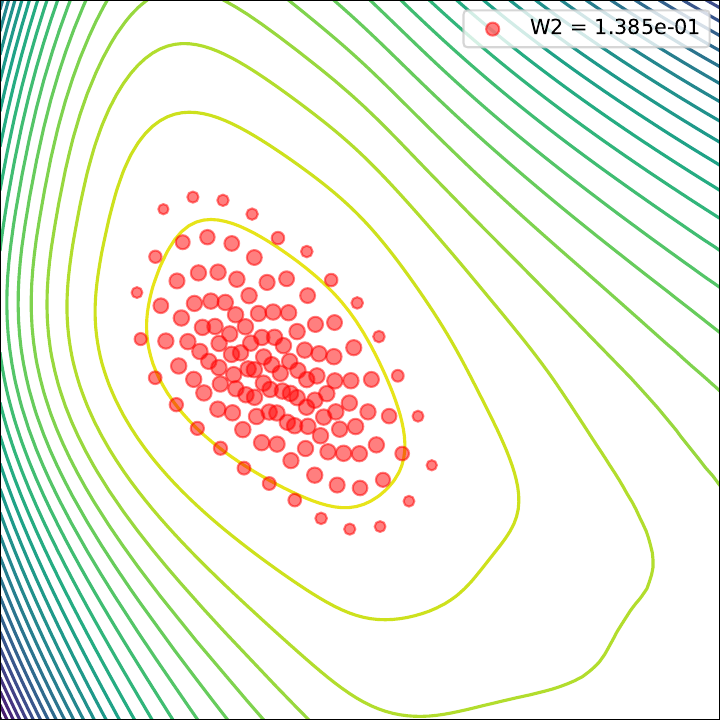}
    \end{subfigure}
    \begin{subfigure}{.19\linewidth}
        \includegraphics[width=\linewidth,height= .6\linewidth]{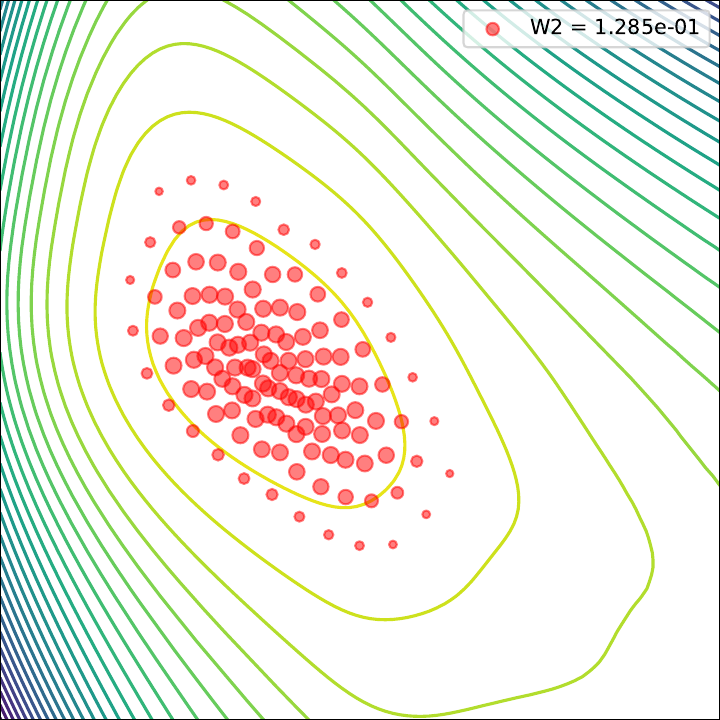}
    \end{subfigure}

    \makebox[0pt][r]{\makebox[15pt]{\raisebox{40pt}{\rotatebox[origin=c]{90}{WGAD-CA-BLOB}}}}%
    \begin{subfigure}{.19\linewidth}
        \includegraphics[width=\linewidth,height= .6\linewidth]{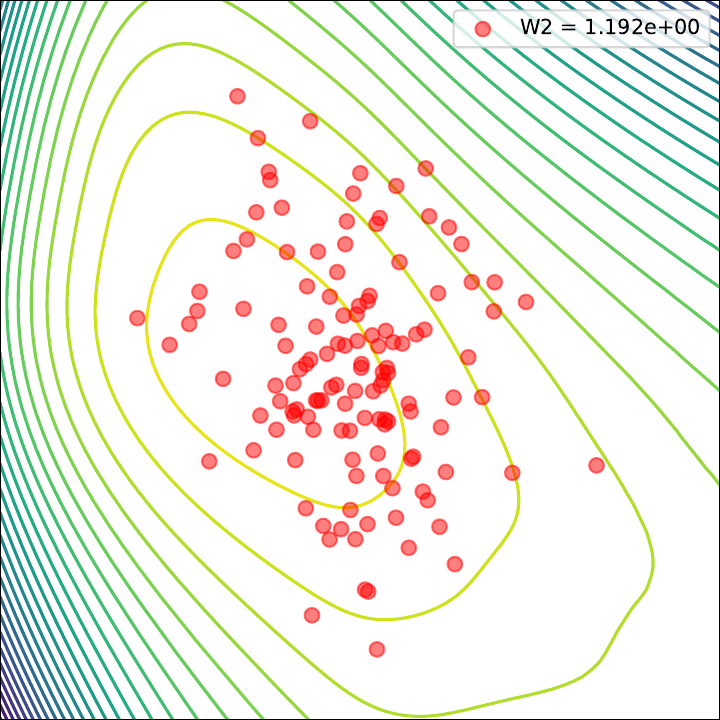}
        \caption{Iteration 0}
    \end{subfigure}
    \begin{subfigure}{.19\linewidth}
        \includegraphics[width=\linewidth,height= .6\linewidth]{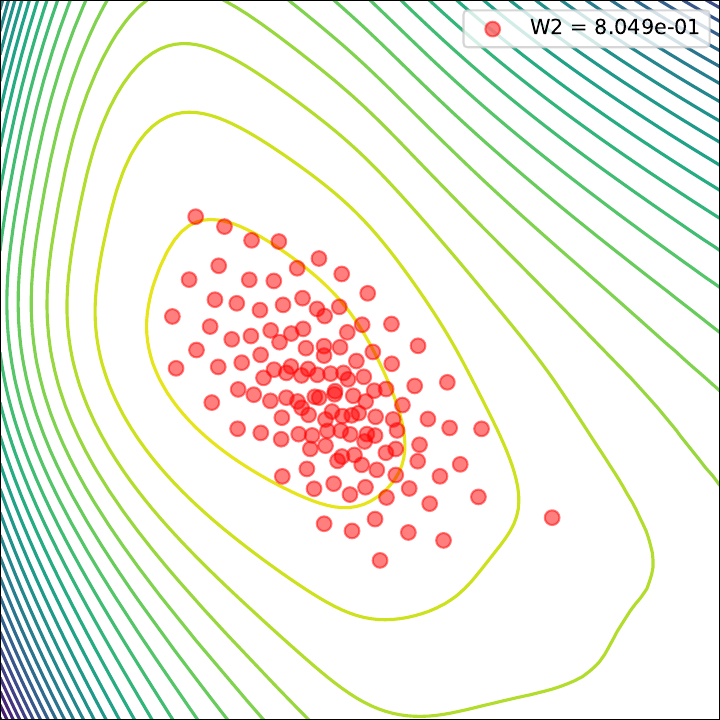}
        \caption{Iteration 100}
    \end{subfigure}
    \begin{subfigure}{.19\linewidth}
        \includegraphics[width=\linewidth,height= .6\linewidth]{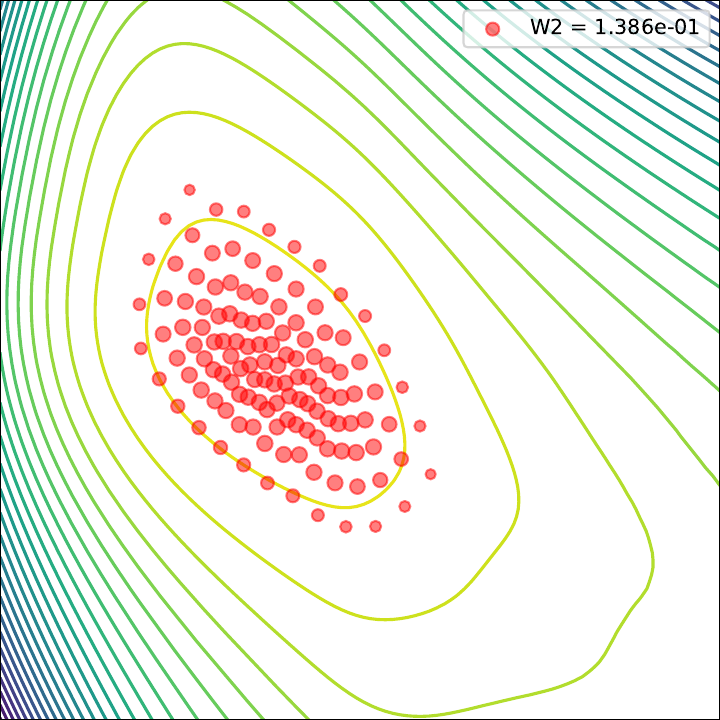}
        \caption{Iteration 500}
    \end{subfigure}
    \begin{subfigure}{.19\linewidth}
        \includegraphics[width=\linewidth,height= .6\linewidth]{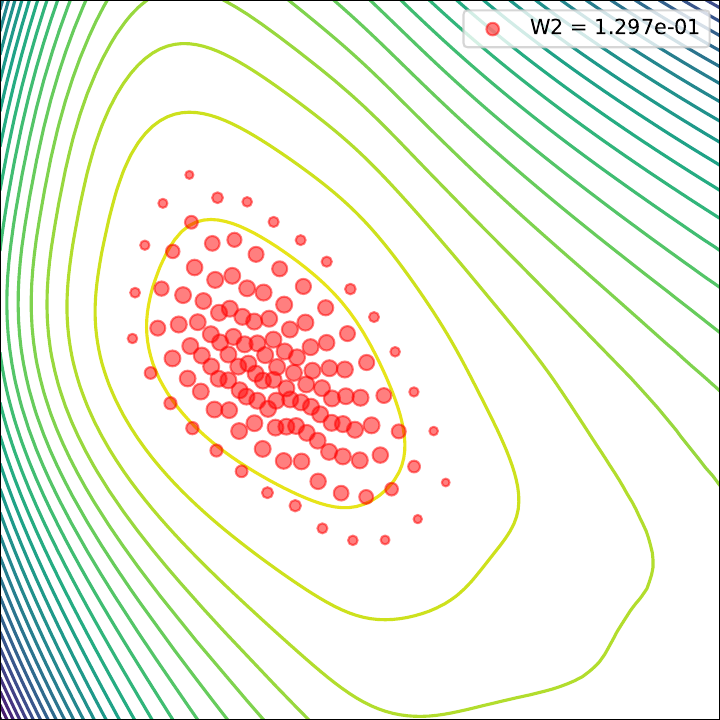}
        \caption{Iteration 2000}
    \end{subfigure}
    \begin{subfigure}{.19\linewidth}
        \includegraphics[width=\linewidth,height= .6\linewidth]{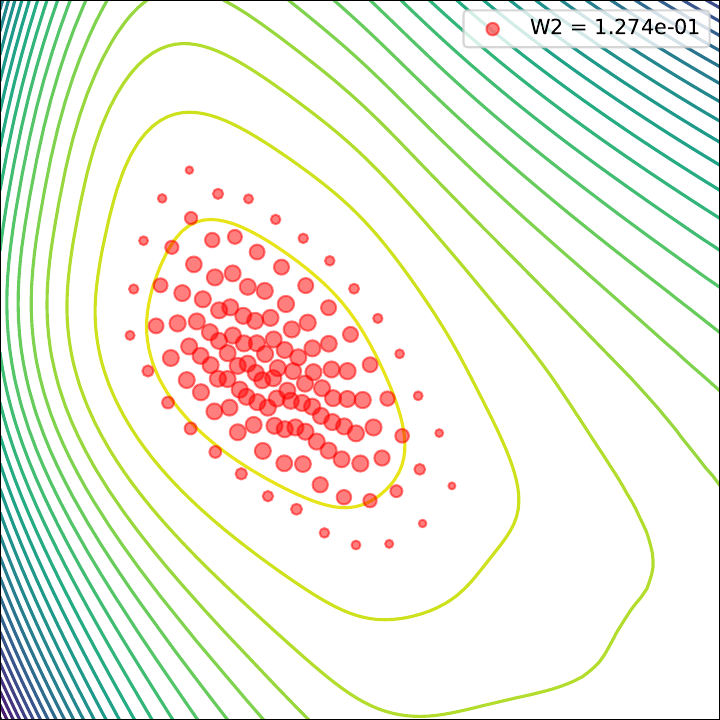}
        \caption{Iteration 10000}
    \end{subfigure}
    \vspace{-3mm}
    \caption{The contour lines of the log posterior in the Gaussian Process task.}
    \label{figure_contour}
    \vspace{-5mm}
    \end{figure*}

\subsection{Gaussian Process Regression}
The Gaussian Process (GP) model is widely adopted for the uncertainty quantification in regression problems \cite{rasmussen2003gaussian}. 
We follow the experiment setting in \cite{chen2018stein}, and use the dataset LIDAR
which consists of 221 observations.
In this task, we set the particle number to $M=128$ for all the algorithms.

We report the $W_2$ distance between the empirical distribution after 10000 iterations
and the target distribution in Table \ref{gp_w2_table}.
The target distribution is approximated by
10000 reference particles generated by the HMC method after it achieves its equilibrium\cite{brooks2011handbook}.
It can be observed that both the accelerated position update and the dynamic weight adjustment result in a decreased $W_2$ and GAD-PVI algorithms consistently achieve loweset $W_2$ to the target.
Besides, the results also show that the CA variants usually outperforms their DK counterpart, as CA is able to adjust the weight continuously on $[0,1]$ while DK set the weight either to $0$ or $1/M$.

In Figure \ref{figure_contour}, we plot the contour lines of the log posterior and the particles 
generated by four representative algorithms, namely BLOB, WAIG-BLOB, DPVI-CA-BLOB, and WGAD-CA-BLOB, 
at different iterations (0, 100, 500, 2000, 10000). 
The results indicate that the particles in WAIG-BLOB and WGAD-CA-BLOB exhibit a faster convergence to the high probability area of the target  due 
to their accelerated position updating strategy,
and the DPVI-CA and WGAD-CA algorithms finally offer a broader final coverage, 
as the CA dynamic weight adjustment strategy enables the particles to represent the region with arbitrary local density mass instead of a fixed $1/M$ mass.

\begin{table}[tb]
	\centering
	 \resizebox{0.45\textwidth}{!}{
	\begin{tabular}{c|cccc}
		\toprule
		\multirow{2}*{algorithms} & \multicolumn{4}{c}{Datasets}  \\
		~    & Concrete & kin8nm &RedWine&space \\
        \hline
        ParVI-SVGD & 	6.323ee+00  & 8.020e-02 & 6.330e-01 &9.021e-02\\
		\hline
		ParVI-BLOB & 	6.313e+00  & 7.891e-02 & 6.318e-01 &8.943e-02\\
		WAIG-BLOB & 	6.063e+00  & 7.791e-02 & 6.267e-01&8.775e-02 \\
		WNES-BLOB & 	6.112e+00  & 7.690e-02 & 6.264e-01 &8.836e-02\\
        DPVI-DK-BLOB & 	6.285e+00  & 7.889e-02 & 6.294e-01&8.853e-02 \\
        DPVI-CA-BLOB & 	6.292e+00  & 7.789e-02 & 6.298e-01&8.850e-02 \\
        \textbf{W\NAME-DK-BLOB} & 	6.058e+00  & 7.688e-02 & 6.267e-01& 8.716e-02\\
        \textbf{W\NAME-CA-BLOB} & 	\textbf{6.047e+00}  & \textbf{7.629e-02} & \textbf{6.263e-01}&\textbf{8.704e-02} \\

        \hline
		ParVI-GFSD & 	6.314e+00  & 7.891e-02 & 6.317e-01 & 8.943e-02\\
		WAIG-GFSD & 	6.105e+00  & 7.794e-02 & 6.265e-01 & 8.776e-02\\
		WNES-GFSD & 	6.123e+00  & 7.756e-02 & 6.263e-01 & 8.836e-02\\
        DPVI-DK-GFSD & 	6.291e+00  & 7.882e-02 & 6.277e-01 & 8.851e-02\\
        DPVI-CA-GFSD & 	6.290e+00  & 7.791e-02 & 6.298e-01 & 8.852e-02\\
        \textbf{W\NAME-DK-GFSD} & 	6.099e+00  & 7.726e-02 & 6.265e-01 &\textbf{8.708e-02} \\
        \textbf{W\NAME-CA-GFSD} & 	\textbf{6.088e+00}  & \textbf{7.634e-02} & \textbf{6.260e-01} &8.710e-02 \\

		\bottomrule
	\end{tabular}}
	\caption{Averaged Test $RMSE$ in the BNN task.}
	\label{bnn_rmse_w}
	\vspace{-5mm}
\end{table}



\subsection{Bayesian Neural Network}
In this experiment, we study a Bayesian regression 
task with Bayesian neural network on 4 datasets from 
UCI and LIBSVM.
We follow the experiment setting from \cite{liu2016stein,dpvi}, 
which models the output as a Gaussian distribution and uses a $\text{Gamma}(1, 0.1)$ prior for the inverse covariance.
We use a one-hidden-layer neural network with 50 hidden units and maintain 128 particles. 
For all the datasets, we set the batchsize as 128.

We present the Root Mean Squared Error (RMSE) of various ParVI algorithms in Table \ref{bnn_rmse_w}.
The results demonstrate that the combination of the accelerated position updating strategy 
and the dynamically weighted adjustment leads to a lower RMSE. 
Notably, WGAD-CA type algrithms outperform other methods in the majority of cases.

\section{Conclusion}
In this paper, we propose the General Accelerated Dynamic-Weight Particle-based Variational Inference (GAD-PVI) framework, which adopts an accelerated position update scheme and dynamic weight adjustment approach simutaneously.
Our GAD-PVI framework is developed by discretizing the Semi-Hamiltonian Information Fisher-Rao (SHIFR) flow on the novel Information-Fisher-Rao space. 
The theoretical analysis  demonstrate that  the SHIFR flow yields additional decrease on the local functional dissipation compared to the Hamiltonian flow in the vanilla information space.
We propose effective particle system which evolve the position, weight, velocity of particles via a set of odes for the SHIFR flows with different underlying information metrics.
By directly discretizing the propsed particle system, we obatain our GAD-PVI framework.
Several effective instances of the GAD-PVI framework have been provided by employing three distinct dissimilarity functionals and associated smoothing approaches under the Wasserstein/Kalman-Wasserstein/Stein metric.
Empirical studies demonstrate the faster convergence and reduced approximation error of GAD-PVI methods over the SOTAs.

\clearpage
\section{Acknowledgments}
This work is supported by National Key Research and Development Program of China under Grant 2020AAA0107400 and National Natural Science Foundation of China (Grant No: 62206248).
\bibliography{aaai24.bib}

\appendix
\onecolumn

\section{Appendix A}
\subsection{A.1 Definition of Information Metric in Probability Space}\label{definition}
To make general information gradient flow on probability space defined, we briefly review the 
definition of information metrics in probability space \cite{ambrosio2008gradient}:
\begin{definition}\label{definition_metric}
    (Information metric in probability space). Denote the tangent space at $\mu \in \mathcal{P}(\RBB^n )$ by
    $T_\mu\mathcal{P}(\RBB^n) = \{\sigma \in C(\RBB^n):\int \sigma dx = 0 \}$. The cotangent
    space at $\mu$ is denoted as $T_\mu^* \mathcal{P}(\RBB^n)$, can be treated as the quotient space $C(\RBB^n)/\mathbb{R}$.
    An information metric tensor $G(\mu)[\cdot]:T_\mu \mathcal{P}(\RBB^n)\to T_\mu^* \mathcal{P}(\RBB^n)$ is an invertible mapping from
    $\mathcal{P}(\RBB^n)$ to $ T_\mu^* \mathcal{P}(\RBB^n)$. This information metric tensor defines the information metric (as well as inner product)
    on tangent space $\mathcal{P}(\RBB^n)$, for $\sigma_1,\sigma_2\in T_\mu \mathcal{P}(\RBB^n)$ 
    and $\Phi_i = G(\mu)[\sigma_i],i=1,2$ as
    \begin{align}\label{metric_prob}
        g_\mu(\sigma_1,\sigma_2)=\int\sigma_1G(\mu)\sigma_2dx = \int\Phi_1G(\mu)^{-1}\Phi_2dx,
    \end{align}
\end{definition}
then we denote the general information probability space as $(\mathcal{P}(\RBB^n),G(\mu))$.
As long as a metric is specified, the probability space $\mathcal{P}(\RBB^n)$ together with the metric
can be taken as an infinite dimensional Riemannian manifold which is so-called density manifold \cite{density_manifold},
which enables the definition of gradient flow.

\subsection{A.2 Full Hamiltonian Flow on the IFR Space and the Fisher-Rao Kinetic Energy}\label{full_Hamiltonian}
To develop ParVI methods which possess both accelerated position update
and dynamical weight adjustment simultaneously,
a natural choice is to directly simulate the Hamiltonian flow on the augmented IFR space
By substituting the IFR metric \eqref{ifr_metric} into the general Hamiltonian flow \eqref{AIG_general}, we derive the Full Hamiltonian flow on the IFR space as the direct accelerated probabilistic flow on the IFR space with the form
\begin{align}\label{DAIFR_flow}
    \left\{
    \begin{aligned}
        &\partial_t \mu_t \!=G(\mu_t)^{-1}\left[ \Phi_t\right] \!
        +\! \left(\Psi_t \!-
        \! \int{\Psi_t}d\mu_t\right)\mu_t\\
        &\partial_t\Phi_t \! + \! \gamma_t\Phi_t \! + 
        \underbrace{\frac{1}{2}\frac{\delta}{\delta\mu}(\int\Phi_tG(\mu_t)^{-1}\left[\Phi_t\right] dx)}_{\text{Information kinetic energy}}\! 
        +\underbrace{\frac{1}{2}\Phi_t^2-\Phi_t\int\Phi_t d\mu_t}_{\text{Fisher-Rao kinetic energy}}
        \! + \! \frac{\delta \mathcal{F}(\mu_t) }{\delta \mu} \!=\! 0.
    \end{aligned}\right.
\end{align}
where the $\Phi_t$ is the velocity field; 
the $\frac{\delta \mathcal{F}(\mu_t) }{\delta \mu}$ represent the potential energy dissipation;
the $\frac{1}{2}\frac{\delta}{\delta\mu}(\int\Phi_tG(\mu_t)^{-1}\left[\Phi_t\right] dx)$ represent the kinetic energy dissipation of information transport;
the $\frac{1}{2}\Phi_t^2-\Phi_t\int\Phi_t d\mu_t$ represent the kinetic energy dissipation of Fisher-Rao distortion.
As far as we know, the particle formulation of the Full Hamiltonian flow on the IFR space \eqref{DAIFR_flow} is intractable due to the Fisher-Rao kinetic energy term. 
When deriving particle systems, the $\nabla\Phi_t$ can be straightforwardly approximated by the velocities $\vb^i_t$, but the Hamiltonian field $\Phi_t$ is hard to be approximated by finite points and iteratively updated. 
Actually, even the particle formulation of the accelerated Fisher-Rao flow has not been derived due to great difficulty \cite{AIG}.
Therefore, we ignore the influence of the Fisher-Rao kinetic energy on the Hamiltonian field and derive the SHIFR \eqref{GAD-FLOW}. We point out that the Fisher-Rao kinetic energy would vanish as the flow converge to the equilibrium of $(\mu_{\infty}=\pi,\Phi_{\infty}=\mathbf{0})$ which fit the behavior of kinetic energy in physical dynamic system. Therefore, the ignorance of the Fisher-Rao kinetic energy is tenable and the SHIFR still have the target distribution as its stationary distribution which will be shown in Proposition \ref{proposition_stationary}.

\subsection{A.3 Proof of Proposition \ref{proposition_1}}\label{app:proof_of_p1}
First we Introduce a technical lemma for the proof of Proposition \ref{proposition_1}.
\begin{lemma}\label{lemma_1}
     The following probability flow dynamic formulation and particle system formulation is equivalent:
     \begin{align}\label{WAIG-euler}
        \left\{
            \begin{aligned}
                &\partial_t \mu_t \!+  \nabla\cdot\left(\mu_t  \nabla \Phi_t\right) = 0 \!
                \\
                &\partial_t\Phi_t + \gamma_t\Phi_t
                +\frac{1}{2}\left\lVert \nabla\Phi_t\right\rVert ^2
                +\frac{\delta\mathcal{F}(\mu_t)}{\delta\mu}=0 
            \end{aligned}\right.
       \end{align}
    \begin{align}\label{WAIG-lagrangian}
        \left\{
        \begin{aligned}
            &\frac{d}{dt}X_t=V_t\\
            &\frac{d}{dt}V_t=-\gamma_tV_t-\nabla(\frac{\delta\mathcal{F}(\mu_t)}{\delta\mu})(X_t) 
        \end{aligned}\right.
    \end{align}
\end{lemma}
\begin{proof}
    We start with the calculation of gradient of kinetic term. 
    For a twice differentiable $\Phi(x)$, we have:
    \begin{align}\label{p1_cancel}
        \frac{1}{2}\nabla\left\lVert \nabla \Phi\right\rVert ^2 = 
        \nabla^2\Phi\nabla\Phi
        =(\nabla\Phi\cdot\nabla)\nabla\Phi
	\end{align}
    From \eqref{WAIG-euler}, we have:
    \begin{align*}
        \partial_t\mu_t + \nabla\cdot(\mu_t\nabla\Phi_t)= 0
    \end{align*}
    which is the continuity equation of $\mu_t$ under vector field $\nabla\Phi_t$ \cite{santambrogio2017euclidean}. 
    Hence, we have following equation on partical level (denoting $V_t=\nabla\Phi_t(X_t)$):
    \begin{align*}
        \frac{d}{dt}X_t&=\nabla\Phi_t(X_t)\\
        &=V_t
    \end{align*}
    Then, the vector field shall follow:
    \begin{align*}
        \frac{d}{dt}V_t&=\frac{d}{dt}\nabla\Phi_t(X_t)\\
        &\stackrel{\mathrm{(1)}}{=}(\partial_t+\nabla\Phi_t(X_t)\cdot\nabla)\nabla\Phi_t(X_t)\\
        &\stackrel{\mathrm{(2)}}{=}-\gamma_t\nabla\Phi_t(X_t)
        -\frac{1}{2}\nabla\left\lVert \nabla \Phi_t(X_t)\right\rVert ^2
        -\nabla(\frac{\delta\mathcal{F}(\mu_t)}{\delta\mu})(X_t) 
        +(\nabla\Phi_t(X_t)\cdot\nabla)\nabla\Phi_t(X_t)\\
        &\stackrel{\mathrm{(3)}}{=}-\gamma_t\nabla\Phi_t(X_t)-\nabla(\frac{\delta\mathcal{F}(\mu_t)}{\delta\mu})(X_t) \\
        &\stackrel{\mathrm{(4)}}{=}-\gamma_tV_t-\nabla(\frac{\delta\mathcal{F}(\mu_t)}{\delta\mu})(X_t) 
    \end{align*}
    where equation (1) becomes valid from material derivative in fluid dynamic \cite{von2004mathematical},
    equation (2) comes from the PDEs of $\Phi_t$ in \eqref{WAIG-euler},
    equation (3) comes from cancelling terms on each side of \eqref{p1_cancel},
    equation (4) comes from the definition of $V_t$.
\end{proof}
Now we are ready to give the proof of Proposition \ref{proposition_1}.

\begin{proof}(Proof of Proposition \ref{proposition_1})
    Let $\Psi:\mathcal{P}_2(\mathbb{R}^d)\to \mathbb{R}$ be a functional on the probability space, 
    $\mut_t^M$ be the distribution produced by the continuous-time composite flow \eqref{particle_formulation_w} at time $t$.
     With $\mu_t$ denoting the mean-field limit of $\mut_t^M$ as $M \to \infty$, we have 
	\begin{align*}
		\partial_t \Psi[\mu_t] = (\mathcal{L} \Psi)[\mu_t],
	\end{align*}
	where, 
	\begin{align}\label{lpsi} 
		\mathcal{L} \Psi[\mu] =\int \langle \nabla \Phi(\xb),
         \nabla_\xb \frac{\delta \Psi(\mu)}{\delta \mu }(\xb)  \rangle \mu(\xb)\mathrm{d} \xb
		-\int \frac{\delta \Psi(\mu)}{\delta \mu }(\xb) \left( \frac{\delta\mathcal{F}(\mu)}{\delta\mu}(\xb) 
        - \mathbb{E}_{\mu} [\frac{\delta\mathcal{F}(\mu)}{\delta\mu}(\xb)]\right)\mu(\xb) \dr \xb,
	\end{align}
    in which $\Phi_t$ abides
    \begin{align*}
		\left\{
	\begin{aligned}
        &\partial_t\Phi_t+\gamma_t\Phi_t
        +\frac{1}{2}\left\lVert \nabla\Phi_t\right\rVert ^2
        +\frac{\delta\mathcal{F}(\mu_t)}{\delta\mu}=0\\
        &\Phi_0=\mathbf{0}
    \end{aligned}\right.
    \end{align*}
	and $\frac{\delta \Psi(\mu)}{\delta \mu }(\cdot)$ denotes the first variation 
    of functional $\Psi$ at $\mu$ satisfying 
	\begin{align*}
		\int \frac{\delta \Psi(\mu)}{\delta \mu }(\xb)  \xi(\xb) \dr \xb = \lim_{\epsilon \to 0} \frac{\Psi(\mu + \epsilon \xi) - \Psi(\mu)}{\epsilon}
	\end{align*}
	for all signed measure $\int \xi(\xb)\dr \xb =0$. Let $(\LM^{Pos} \Psi)[\mu]$ be the 
    first term of \eqref{lpsi}, and $(\LM^{Wei} \Psi)[\mu]$ be the second term of \eqref{lpsi},
	We have:
    \begin{align*}
		\mathcal{L} \Psi[\mu] = (\LM^{Pos} \Psi)[\mu]+(\LM^{Wei} \Psi)[\mu]
    \end{align*}
	
    For the measure valued composite flow $\mut_t^M$ \eqref{particle_formulation_w}, the infinitesimal  generator of $\Psi$ w.r.t. $ \tilde \mu_t^M$ is defined as follows: 
	\begin{align*}
		(\LM_M \Psi)[\mut^M] := \lim_{t\to 0^+} \frac{\EBB _{\mut_0^M = \mut^M}(\Psi[\mut_t^M]) - \Psi(\mut^M)}{t},
	\end{align*}
	where $\EBB _{\mut_0^M = \mut^M}(\Psi[\mut_t^M])$ denotes the expectation of the functional $\Psi$ evaluated along the trajectory $\mut_t^M$ taken conditional on the initialization $\mut_0^M = \mut^M$.
	As 
	\begin{align*}
		\EBB _{\mut_0^M = \mut^M}(\Psi[\mut_t^M]) - \Psi(\mut^M)
		=&\underbrace{\EBB _{\mut_0^M = \mut^M}(\Psi[\sum_{i=1}^M w^i_t \delta_{\xb^i_t}]) -
         \EBB _{\mut_0^M = \mut^M}( \Psi[\sum_{i=1}^M w^i_0 \delta_{\xb^i_t}]) }
         _{\text{weight adjusting infinitesimal}}
		\\ &+\underbrace{\EBB _{\mut_0^M = \mut^M}( \Psi[ \sum_{i=1}^M w^i_0 \delta_{\xb^i_t}])  -
         \Psi(\sum_{i=1}^M w^i_0 \delta_{\xb^i_0})}_{\text{position adjusting infinitesimal}},	
	\end{align*}	
	we follow the same idea from \cite{mroueh2020unbalanced,rotskoff2019global,dpvi} to 
    divide $(\LM_M \Psi)[\mut^M]$ into two parts $(\LM_M^{Pos} \Psi)[\mut^M]$ and
     $(\LM_M^{Wei} \Psi)[\mut^M]$ corresponding to the position update and the weight adjustment, respectively.
	According to the definition of first variation, it can be calculated that 
	\begin{align*}
		(\LM_M^{Wei} \Psi)[\mut^M]  
		=&	\lim_{t\to 0^+} \frac{\EBB _{\mut_0^M = \mut^M}(\Psi[\sum_{i=1}^M w^i_t \delta_{\xb^i_t}]) - \EBB _{\mut_0^M = \mut^M}(\Psi[\sum_{i=1}^M w^i_0 \delta_{\xb^i_t}])}{t}
		\\
		=& \int \frac{\delta \Psi(\mu^M)}{\delta \mu }(\xb) \sum_{i=1}^M  \partial_t \rho(w^i_t \xb^i_0)  \dr \xb
		\\
		=&- \int \frac{\delta \Psi(\mu^M)}{\delta \mu }(\xb) (\frac{\delta\mathcal{F}(\mu^M)}{\delta\mu}(\xb) - 
        \mathbb{E}_{\mu^M} [\frac{\delta\mathcal{F}(\mu^M)}{\delta\mu}(\xb)])\mu^M(\xb) \dr \xb.
	\end{align*}
	and 
    \begin{align*}
		(\LM_n^{Pos} \Psi)[\mut^M]  
		=&\lim_{t\to 0^+} \frac{\EBB _{\mut_0^M = \mut^M}(\Psi[\sum_{i=1}^M w^i_0 \delta_{\xb^i_t}]) - \Psi(\sum_{i=1}^M w^i_0 \delta_{\xb^i_0})}{t}
		\\
		=&\int \frac{\delta \Psi(\mu^M)}{\delta \mu }(\xb) \sum_{i=1}^M w^i_0 \partial_t \rho(\xb^i_t)  \dr \xb
		\\
		=&\int \langle  V_{\mu^M}(\xb), \nabla_\xb \frac{\delta \Psi(\mu^M)}{\delta \mu }(\xb)  \rangle \mu^M(\xb)\mathrm{d} \xb
	\end{align*}
    where $V_{\mu^M}$ abides 
    \begin{align*}
		\left\{
	\begin{aligned}
        &\frac{dV_{\mu^M}}{dt} = -\gamma_t V_{\mu^M}-\nabla\frac{\delta\mathcal{F}(\mu^M)}{\delta\mu}\\
        &V_{\mu^M_0}=\mathbf{0}
    \end{aligned}\right.
    \end{align*}
	Combining the above equalities, we have
	\begin{align*}
		\mathcal{L}_M \Psi[\mu_M] =& \mathcal{L}_M^{Pos} \Psi[\mu_M] + \mathcal{L}_M^{Wei} \Psi[\mu_M]
	\end{align*}
	
    If we take the limit of $\mathcal{L}_M \Psi[\mu_M]$  as $M \to \infty$ 
    on a sequence such that $\mu^M \rightharpoonup \mu$ 
    (i.e. $\mu^M$ weakly converges to $\mu$) a.s., 
    and $\frac{\delta\mathcal{F}(\mu^M)}{\delta\mu} \rightharpoonup 
    \frac{\delta\mathcal{F}(\mu)}{\delta\mu}$ a.s., we can deduce that 
    $ V_{\mu^M}\rightharpoonup 
    \nabla \Phi$ in light of Lemma \ref{lemma_1}. Those allow to conclude that 
    $\mathcal{L}^{Wei}_M \Psi[\mu_M] \to \mathcal{L}^{Wei} \Psi[\mu]$
	and $\mathcal{L}^{Pos}_M \Psi[\mu_M] \to \mathcal{L}^{Pos}\Psi[\mu]$, 
    thus $\mathcal{L}_M \Psi[\mu_M] \to \mathcal{L} \Psi[\mu]$.

	Since $\partial_t \Psi(\mu_t^M) = \mathcal{L}_M \Psi[\mu^M] $ and $\partial_t \Psi(\mu_t) = \mathcal{L}_M \Psi[\mu_t]$,
	we have 
	\begin{align*}
		\lim_{n\to \infty} \Psi(\mu_t^M) = \Psi (\mu_t),
	\end{align*}
	which indicates that $\mu_t^M \rightharpoonup \mu_t$ if $\mu_0^M \rightharpoonup \mu_0$.
	Since $\mu_t$ satisfying $\partial_t \Psi(\mu_t) = \mathcal{L} \Psi[\mu_t]$ 
    solves the partial differential equation \eqref{Gad-W-flow}, we conclude that the path 
    of \eqref{particle_formulation_w} starting from $\mut^M_0$ 
	weakly converges to a solution of the partial differential equation \eqref{Gad-W-flow} 
    starting from $\mu_0$ as $M\to\infty$.
    
\end{proof}

\subsection{A.4 The Extra Decrease of Functional Dissipation of the SHIFR Flow \eqref{GAD-FLOW}}\label{app:descending}
Here, we investigate the extra decrease property in terms of functional dissipation of the SHIFR gradient flow \eqref{GAD-FLOW} 
comparing to the Hamiltonian flow \eqref{AIG_general} in vanilla information space
in the following proposition. Here we only illustrate the Wasserstein case for readers can easily check for general information space case in the same routine.
\begin{proposition}\label{proposition_dissipation}

    For arbitrary $\overline{\mu}\in \mathcal{P}(\RBB^n) $ and $\overline{\Phi}\in C(\RBB^n)$,
    The local dissipation of functional $\frac{\mathrm{d}\mathcal{F}(\mu_t)}{\mathrm{d}t}$
    following the SHIFR gradient flow \eqref{GAD-FLOW} starting from $(\overline{\mu},\overline{\Phi})$
    has an additional functional dissipation term comparing to the ones following the Hamiltonian flow in non-augmented space \eqref{AIG_general}.

    Take the Wasserstein case as an example.
    Denote the probability path starting from $(\overline{\mu},\overline{\Phi})$ 
    following W-SHIFR flow as $(\mu^{ SHIFR}_t,\Phi^{ SHIFR}_t)$,
    following Hamiltonian flow in vanilla space as $(\mu^{H}_t,\Phi^{H}_t)$.
    We have:
    \begin{align}\label{p2_1}
        \left.\frac{\mathrm{d}\mathcal{F}(\mu^{ SHIFR}_t)}{\mathrm{d}t}\right|_{t=0}
        \left.\leq \frac{\mathrm{d}\mathcal{F}(\mu^{H}_t)}{\mathrm{d}t}\right|_{t=0}
    \end{align}

\end{proposition}

\begin{proof}
	
    For W-SHIFR case, according to the result in \eqref{lpsi},   
	the following eqaulity holds for any functional $\Psi:\mathcal{P}_2(\mathbb{R}^d)\to \mathbb{R}$ on the probability space $\mathcal{P}_2(\mathbb{R}^d)$,
    where, 
	\begin{align*}
		\partial_t \Psi[\mu^{ SHIFR}_t] &=\int \langle \nabla \Phi(\xb),
         \nabla_\xb \frac{\delta \Psi(\mu^{ SHIFR}_t)}{\delta \mu }(\xb)  \rangle \mu^{ SHIFR}_t(\xb)\mathrm{d} \xb
		\\ &-\int \frac{\delta \Psi(\mu^{ SHIFR}_t)}{\delta \mu }(\xb) \left( \frac{\delta\mathcal{F}(\mu^{ SHIFR}_t)}{\delta\mu}(\xb) 
        - \mathbb{E}_{\mu} [\frac{\delta\mathcal{F}(\mu^{ SHIFR}_t)}{\delta\mu}(\xb)]\right)\mu^{ SHIFR}_t(\xb) \dr \xb,
	\end{align*}
    in which $\left.\mu_t\right|_{t=0}=\overline{\mu}$ and $\Phi_t$ abides
    \begin{align*}
		\left\{
	\begin{aligned}
        &\partial_t\Phi_t+\gamma_t\Phi_t
        +\frac{1}{2}\left\lVert \nabla\Phi_t\right\rVert ^2
        +\frac{\delta\mathcal{F}(\mu_t)}{\delta\mu}=0\\
        &\left.\Phi_t\right|_{t=0}=\overline{\Phi}
    \end{aligned}\right.
    \end{align*}

	By substituting $\Psi(\mu^{ SHIFR}_t) = \FM(\mu^{ SHIFR}_t)$,
    $U_{\mu^{ SHIFR}_t} = \frac{\delta \FM(\mu^{ SHIFR}_t)}{\delta\mu }$
    and $t=0$ in the above equality, we have:
	\begin{align}\label{dissi_1}
		\left.\frac{\mathrm{d}\FM(\mu^{ SHIFR}_t)}{\mathrm{d}t}\right|_{t=0} =
        & - \int \langle \nabla_\xb \frac{\delta \FM(\overline{\mu})}{\delta\mu }(\xb),\nabla \overline{\Phi}(\xb)\rangle \overline{\mu}(\xb)\mathrm{d} \xb 
        - \int \frac{\delta \FM(\overline{\mu})}{\delta\mu }(\xb) (U_{\overline{\mu}}(\xb) - \mathbb{E}_{\overline{\mu}} [U_{\overline{\mu}}(\xb)])\overline{\mu}(\xb) \dr \xb \nonumber\\
		= &- \int \langle \nabla_\xb \frac{\delta \FM(\overline{\mu})}{\delta\mu }(\xb),\nabla \overline{\Phi}(\xb)\rangle \overline{\mu}(\xb)\mathrm{d} \xb 
        -\left(\int{\left(\frac{\delta\mathcal{F}(\overline{\mu})}{\delta\mu}(\xb)\right)^2\overline{\mu}(\xb)\mathrm{d}\xb} - \left(\int{\frac{\delta\mathcal{F}(\overline{\mu})}{\delta\mu}(\xb)\overline{\mu}(\xb)\mathrm{d}\xb}\right)^2\right).
	\end{align}   
	Similarly, we can get following result for the Hamiltonian flow in the non-augmented Wasserstein space case:
    \begin{align}\label{dissi_2}
		\left.\frac{\mathrm{d}\FM(\mu^{H}_t)}{\mathrm{d}t}\right|_{t=0} 
		= &- \int \langle \nabla_\xb \frac{\delta \FM(\overline{\mu})}{\delta\mu }(\xb),\nabla \overline{\Phi}(\xb)\rangle \overline{\mu}(\xb)\mathrm{d} \xb .
    \end{align} 
    Since the second term of \eqref{dissi_1} is always less or equal to zero
    and the fisrt term of \eqref{dissi_1} is the same as the first term of \eqref{dissi_2}, we can reach to the conclusion
    that the local dissipation of SHIFR flow has an additional functional dissipation term compared to the Hamiltonian flow in the non-augmented space:
    \begin{align*}
        \left.\frac{\mathrm{d}\mathcal{F}(\mu^{ SHIFR}_t)}{\mathrm{d}t}\right|_{t=0}
        \left.\leq \frac{\mathrm{d}\mathcal{F}(\mu^{H}_t)}{\mathrm{d}t}\right|_{t=0}
    \end{align*}
    For general information space, readers can follow the same routine to get the extra functional dissipation property.
\end{proof}

\subsection{A.5 The stationary analysis of the SHIFR Flow \eqref{GAD-FLOW}}\label{app:stationary}
Following proposition establish the stationary property of the  SHIFR Flow \eqref{GAD-FLOW} with
dissimilarity functional $\mathcal D(\cdot|\pi)$ which vanishes at $\mu=\pi$.
\begin{proposition}\label{proposition_stationary}

    The target distribution and zero-velocity $(\mu_{\infty}=\pi,\Phi_{\infty}=\mathbf{0})$ ($\mathbf{0}$ means a funtion
    defined on $\RBB^n$ that always map to zero)
    is the stationary distribution of the  SHIFR flow \eqref{GAD-FLOW} with dissimilarity functional
     $\mathcal D(\cdot|\pi)$ which satisfy $\mathcal D(\pi|\pi)=0$ with any Information metric tensor $G^I(\mu)[\cdot]$.
\end{proposition}
    
\begin{proof}
The  SHIFR flow under an Information metric writes:
	\begin{align}\label{GAD-FLOW-I}
        \left\{
        \begin{aligned}
            &\partial_t\mu_t=G^I(\mu_t)^{-1}\left[ \Phi_t\right]-\! \left(\frac{\delta\mathcal{F}(\mu_t)}{\delta\mu} \!-
            \! \int{\frac{\delta\mathcal{F}(\mu_t)}{\delta\mu}\mathrm{d}\mu_t}\right)\mu_t ,\\
            &\partial_t\Phi_t+\gamma_t\Phi_t+\frac{1}{2}\frac{\delta}{\delta\mu}(\int\Phi_tG^I(\mu_t)^{-1}\left[\Phi_t\right] dx)+\frac{\delta \mathcal{F}(\mu_t) }{\delta \mu}=0.
        \end{aligned}\right.
    \end{align}
    Because the functional $\mathcal{F}(\cdot)$ is specified as some dissimilarity functional $\mathcal D(\cdot|\pi)$,
    we have:
    \begin{align*}
		\frac{\delta \mathcal{F}(\mu_\infty) }{\delta \mu}
        =\frac{\delta \mathcal{F}(\pi) }{\delta \mu}
        =\frac{\delta\mathcal D(\pi,\pi) }{\delta \mu}
        =\mathbf{0}.
    \end{align*}   
    From the element of gradient flow, we also have:
    \begin{align*}
		G^I(\mu)^{-1}\left[ \Phi_\infty\right]
        =G^I(\mu)^{-1}\left[ \mathbf{0}\right]
        =\mathbf{0}.
    \end{align*}  
    Substituting into \eqref{GAD-FLOW-I} that $(\mu_{\infty}=\pi,\Phi_{\infty}=\mathbf{0})$,
    we can get :
    \begin{align}\label{stationary_1}
		&\left.\partial_t\Phi_t\right|_{t=\infty}=
        -\gamma_\infty \Phi_\infty 
        -\frac{1}{2}\frac{\delta}{\delta\mu}(\int\Phi_\infty G^I(\mu_t)^{-1}\left[\Phi_\infty\right] dx)
        -\frac{\delta \mathcal{F}(\mu_\infty) }{\delta \mu}=\mathbf{0},\\
        &\left.\partial_t\mu_t\right|_{t=\infty}=G^I(\mu_\infty)^{-1}\left[ \Phi_\infty\right]-\! \left(\frac{\delta\mathcal{F}(\mu_\infty)}{\delta\mu} \!-
            \! \int{\frac{\delta\mathcal{F}(\mu_\infty)}{\delta\mu}\mathrm{d}\mu_\infty}\right)\mu_\infty=\mathbf{0}.
    \end{align}   
    These suffice for proof.

\end{proof}
\clearpage
\section{Appendix B}\label{app:b}
\subsection{B.1 Kalman-Wasserstein-SHIFR Flow and KWGAD-PVI Algorithms}
Combining the Kalman filter to estimate the probability distributions of a dynamic system over time
and the Wasserstein metric to measure the difference between these estimated distributions,
the Kalman-Wasserstein metric is proposed in ensemble Kalman sampling literature \cite{garbuno2020interacting}.
The inverse of Kalman-Wasserstein metric tensor write:
\begin{align}\label{kw_inverse_tensor}
    G^{KW}(\mu)^{-1}\left[\Phi\right] = -\nabla \cdot (\mu C^\lambda(\mu)\nabla\Phi),\Phi \in T_\mu^* \mathcal{P}(\RBB^n),
\end{align}
where $\Phi \in T^*_\mu\mathcal{P}(\RBB^n)$,
substituting into \eqref{gf_general}, the gradient flow of Kalman-Wasserstein metric writes:
\begin{align}\label{gf_kw}
    \partial_t \mu_t = G^{KW}(\mu_t)^{-1}\left[\frac{\delta \mathcal{F}(\mu_t) }{\delta \mu}\right] = 
    -\nabla \cdot (\mu_t C^\lambda(\mu_t) \nabla \frac{\delta \mathcal{F}(\mu_t) }{\delta \mu}) .
\end{align}
where $\lambda \geq 0$ is the regularization constant and $C^\lambda(\mu)$ is the linear transformation follows:
\begin{align}\label{C_matrix_kw}
    C^\lambda(\mu) =
    \int (x-m(\mu))(x-m(\mu))^T\mu dx +\lambda I,m(\mu) = \int x\mu dx.
\end{align}
Substituting the Kalman-Wasserstein metric into the  SHIFR flow \eqref{GAD-FLOW} gives
the Kalman-Wasserstein-SHIFR flow:
\begin{align}\label{Gad-KW-flow}
    \left\{
    \begin{aligned}
        &\partial_t \mu_t \!=-  \nabla\cdot\left(\mu_t C^\lambda(\mu_t) \nabla \Phi_t\right) \!
        -\! \left(\frac{\delta\mathcal{F}(\mu_t)}{\delta\mu} \!-
        \! \int{\frac{\delta\mathcal{F}(\mu_t)}{\delta\mu}\mathrm{d}\mu_t}\right)\mu_t,\\
        &\partial_t\Phi_t + \gamma_t\Phi_t
        +\frac{1}{2}((x-m(\mu_t))^T \int \nabla\Phi_t\nabla\Phi_t^Td\mu_t (x-m(\mu_t))
            +\nabla\Phi_t^TC^\lambda(\mu_t)\nabla\Phi_t
        )
        +\frac{\delta\mathcal{F}(\mu_t)}{\delta\mu}=0 .
    \end{aligned}\right.
\end{align}
We claim that the finite-particles formulation of the Kalman-Wasserstein-SHIFR flow \eqref{Gad-KW-flow}
evolves the positions $x^i$'s, the 
weights $w^i$'s of $M$ particles and velocity field $\vb$ following:
\begin{align}\label{particle_formulation_kw}
    \left\{
	\begin{aligned}
		\mathrm{d}\xb^i_t &=  C^\lambda(\tilde{\mu}_t) \vb^i_t  \mathrm{d}t, \\
        \mathrm{d} \vb^i_t &= (-\gamma \vb^i_t 
        -\EBB[\vb_t\vb_t^T](\xb^i-\EBB[\xb] )
        - \nabla\frac{\delta\mathcal{F}(\tilde{\mu}_t)}{\delta\mu}(\xb^i_t) )\mathrm{d}t, \\
		\mathrm{d} w^i_t &=\textstyle -\left(\frac{\delta\mathcal{F}(\tilde{\mu}_t)}{\delta\mu}(\xb^i_t) 
        - \sum_{i=1}^M w^i_t \frac{\delta\mathcal{F}(\tilde{\mu}_t)}{\delta\mu}(\xb^i_t)\right)w^i_t\mathrm{d}t,\\
		\tilde{\mu}_t   &= \textstyle\sum_{i=1}^M w^i_t\delta_{\xb^i_t}.
	\end{aligned}\right.
\end{align}
Here the expectation is taken over the empirical distribution of particles.

Then, the proposition below show the mean-field limit of the finite-particles formulation \eqref{particle_formulation_kw}
is exactly the Kalman-Wasserstein-SHIFR flow \eqref{Gad-KW-flow}.
\begin{proposition}\label{proposition_2}
    Suppose the empirical distribution $\mut^M_0$ of $M$ weighted particles weakly converges to a 
    distribution $\mu_0$ when $M \to \infty$. 
    Then, the path of \eqref{particle_formulation_kw} starting from $\mut^M_0$ and $\Phi_0$ with initial velocity $\mathbf{0}$
    weakly converges to a solution of the Kalman-Wasserstein-SHIFR gradient flow \eqref{Gad-KW-flow}
    starting from $\mu_t|_{t=0} = \mu_0 = $ and $\Phi_t|_{t=0} = \mathbf{0}$ as $M\to\infty$:
\end{proposition}
Similar to the proof scheme of proposition \ref{proposition_1},
we start from proofing a technical lemma first:
\begin{lemma}\label{lemma_2}
    The following fluid dynamic formulation and particle dynamic formulation is equivalent:
    
    (Suppose that $X_t\sim \mu_t$ and $V_t=\nabla\Phi_t(X_t)$, expectation is taken over particles)
   \begin{align}\label{KWAIG-euler}
    \left\{
        \begin{aligned}
            &\partial_t \mu_t \!+  \nabla\cdot\left(\mu_t C^\lambda(\mu_t) \nabla \Phi_t\right) = 0 ,\!
            \\
            &\partial_t\Phi_t + \gamma_t\Phi_t
            +\frac{1}{2}((x-m(\mu_t))^T \int \nabla\Phi_t\nabla\Phi_t^Td\mu_t (x-m(\mu_t))
                +\nabla\Phi_t^TC^\lambda(\mu_t)\nabla\Phi_t
            )
            +\frac{\delta\mathcal{F}(\mu_t)}{\delta\mu}=0 .
        \end{aligned}\right.
   \end{align}
   \begin{align}\label{KWAIG-lagrangian}
    \left\{
    \begin{aligned}
        &\frac{d}{dt}X_t=C^\lambda(\mu_t) V_t,\\
        &\frac{d}{dt}V_t=-\gamma_tV_t
        -\EBB[V_tV_t^T](X_t-\EBB[X_t])
        -\nabla(\frac{\delta\mathcal{F}(\mu_t)}{\delta\mu})(X_t) .
    \end{aligned}\right.
\end{align}

\end{lemma}

\begin{proof}
First we establish two equations, For $i=1\dots n$, we have:
\begin{align}\label{lemma2_1}
    \begin{aligned}
        (C^\lambda(\mu_t)\nabla\Phi_t\cdot\nabla)\nabla_i\Phi_t(X_t)
        &=\sum_{j = 1}^{n}(C^\lambda(\mu_t)\nabla\Phi_t)_j  \nabla_j\nabla_i\Phi_t(X_t)\\
        &=\sum_{j = 1}^{n}\nabla_{ij}\Phi_t(X_t)(C^\lambda(\mu_t)\nabla\Phi_t)_j\\
        &=(\nabla^2\Phi_tC^\lambda(\mu_t)\nabla\Phi_t)_i.
    \end{aligned}
\end{align}
Them, according to chain rule, we have:
\begin{align}\label{lemma2_2}
    \begin{aligned}
        \nabla(\nabla\Phi_t(x)^TC^\lambda(\mu_t)\nabla\Phi_t(x))
        =2\nabla^2\Phi_t(x)C^\lambda(\mu_t)\nabla\Phi_t(x).
    \end{aligned}
\end{align}
Since the first equation of \eqref{KWAIG-euler} is actually the continuity equation with velocity 
field $C^\lambda(\mu_t) \nabla\Phi_t$, it is obvious to have $\frac{d}{dt}X_t=C^\lambda(\mu_t) V_t$.
Then we deduce the second equation of \eqref{KWAIG-lagrangian}:
\begin{align*}
    \frac{d}{dt}V_t&=\frac{d}{dt}\nabla\Phi_t(X_t)\\
    &\stackrel{\mathrm{(1)}}{=}(\partial_t+C^\lambda(\mu_t)\nabla\Phi_t\cdot\nabla)\nabla\Phi_t(X_t)\\
    &\stackrel{\mathrm{(2)}}{=}\partial_t\nabla\Phi_t+
    \nabla^2\Phi_tC^\lambda(\mu_t)\nabla\Phi_t
    \\
    &\stackrel{\mathrm{(3)}}{=}-\gamma_t\nabla\Phi_t
    -\int \nabla\Phi_t\nabla\Phi_t^Td\mu_t (x-m(\mu_t))
    -\frac{1}{2}\nabla(\nabla\Phi_t(x)^TC^\lambda(\mu_t)\nabla\Phi_t(x))
    -\nabla(\frac{\delta\mathcal{F}(\mu_t)}{\delta\mu})(X_t)
    +\nabla^2\Phi_tC^\lambda(\mu_t)\nabla\Phi_t
    \\&\stackrel{\mathrm{(4)}}{=}-\gamma_t\nabla\Phi_t
    -\int \nabla\Phi_t\nabla\Phi_t^Td\mu_t (x-m(\mu_t))
    -\nabla(\frac{\delta\mathcal{F}(\mu_t)}{\delta\mu})(X_t)
    \\&\stackrel{\mathrm{(5)}}{=}-\gamma_tV_t
    -\EBB[V_tV_t^T](X_t-\EBB[X_t])
    -\nabla(\frac{\delta\mathcal{F}(\mu_t)}{\delta\mu})(X_t). 
\end{align*}
where equation (1) becomes valid from material derivative in fluid dynamic \cite{von2004mathematical},
equation (2) comes from the equation \eqref{lemma2_1},
equation (3) comes from the PDE \eqref{KWAIG-euler},
equation (4) comes from cancelling terms on each side of \eqref{lemma2_2},
equation (5) comes from the definition of $V_t$ and $X_t$.
\end{proof}

\begin{proof}(Proof of Proposition \ref{proposition_2})
Substituting Lemma \ref{lemma_1} by Lemma \ref{lemma_2},
the proof scheme of proposition \ref{proposition_2}
is the same with the proof scheme of proposition \ref{proposition_1}.
\end{proof}

By discretizing \eqref{particle_formulation_kw}, we derive the KWGAD-PVI algorithms which
update the positions of particles according to the following rule: 
\begin{align}\label{dynamic_pos_kw}
    \xb^i_{k+1} = \xb^i_{k} + \eta_{pos}  [C^\lambda_k\vb_{k}]^i,
\end{align}
and adjusts the velocity field as following:
\begin{align}\label{dynamic_ac_kw}
	\vb_{k+1}^i=(1-\gamma\eta_{vel})\vb_k^i
    -\frac{\eta_{vel}}{M}[\sum_{j=1}^{N}w_k^j(V_k^j)(V_k^j)^T](\xb_k^i-m_k)
    -\eta_{vel}\nabla U_{\tilde{\mu}_k}(\xb_k).
\end{align}
Here, $C^\lambda_k$ and $m_k$ are calculated at each round by:
\begin{align}\label{dynamic_auxiliary_kw}
	m_k= \frac{1}{N}\sum_{i=1}^{M}w_k^i x^i_k,
    C^\lambda_k=\frac{1}{N-1}\sum_{i=1}^{M}
    w_k^i(\xb_k^i-m_k) (\xb_k^i-m_k)^T+\lambda I.
\end{align}

\subsection{B.2 Stein-SHIFR Flow and SGAD-PVI Algorithms}
Involving reproducing kernel Hilbert space norm into probability space,
the Stein metric is proposed for geometrical analysis \cite{nusken2021stein}.
The gradient flow of Stein metric writes:
\begin{align}\label{gf_s}
    \partial_t \mu_t = G^{S}(\mu_t)^{-1}\frac{\delta \mathcal{F}(\mu_t) }{\delta \mu} = 
    -\nabla \cdot (\mu_t\int k(\cdot,y)\mu_t(y)\nabla_y\frac{\delta \FM(\mu_t)}{\delta\mu}(y)dy ).
\end{align}
Substituting the Stein metric into the  SHIFR flow \eqref{GAD-FLOW} gives
the Stein-SHIFR flow:
\begin{align}\label{Gad-S-flow}
    \left\{
    \begin{aligned}
        &\partial_t \mu_t \!=-\nabla \cdot (\mu_t\int k(\cdot,y)\mu_t(y)\nabla_y\Phi_t(y)dy )
        -\! \left(\frac{\delta\mathcal{F}(\mu_t)}{\delta\mu} \!-
        \! \int{\frac{\delta\mathcal{F}(\mu_t)}{\delta\mu}\mathrm{d}\mu_t}\right)\mu_t,\\
        &\partial_t\Phi_t + \gamma_t\Phi_t
        +\int \nabla\Phi_t(\cdot)^T\nabla\Phi_t(y)k(\cdot,y)\mu_t(y)dy
        +\frac{\delta\mathcal{F}(\mu_t)}{\delta\mu}=0 .
    \end{aligned}\right.
\end{align}
We claim that the finite-particles formulation of the Stein-SHIFR flow \eqref{Gad-S-flow}
evolves the positions $x^i$'s, the 
weights $w^i$'s of $M$ particles and velocity field $\vb$ following:
\begin{align}\label{particle_formulation_s}
    \left\{
	\begin{aligned}
		\mathrm{d}\xb^i_t &=  [\int k(\xb_t,y)\nabla\Phi_t(y)\tilde{\mu}_t(y) \mathrm{d}y]^i\mathrm{d}t, \\
        \mathrm{d} \vb^i_t &= (-\gamma \vb^i_t 
        -[\int \vb_t^T\nabla\Phi_t(y)\nabla_xk(\xb_t,y)\tilde{\mu}_t(y)\mathrm{d}y]^i
        - \nabla\frac{\delta\mathcal{F}(\tilde{\mu}_t)}{\delta\mu}(\xb^i_t) )\mathrm{d}t, \\
		\mathrm{d} w^i_t &=\textstyle -\left(\frac{\delta\mathcal{F}(\tilde{\mu}_t)}{\delta\mu}(\xb^i_t) 
        - \sum_{i=1}^M w^i_t \frac{\delta\mathcal{F}(\tilde{\mu}_t)}{\delta\mu}(\xb^i_t)\right)w^i_t\mathrm{d}t,\\
		\tilde{\mu}_t   &= \textstyle\sum_{i=1}^M w^i_t\delta_{\xb^i_t}.
	\end{aligned}\right.
\end{align}

Then, the proposition below show the mean-field limit of the finite-particles formulation \eqref{particle_formulation_s}
is exactly the Stein-SHIFR flow \eqref{Gad-S-flow}.
\begin{proposition}\label{proposition_3}
    Suppose the empirical distribution $\mut^M_0$ of $M$ weighted particles weakly converges to a 
    distribution $\mu_0$ when $M \to \infty$. 
    Then, the path of \eqref{particle_formulation_s} starting from $\mut^M_0$ and $\Phi_0$ with initial velocity $\mathbf{0}$
    weakly converges to a solution of the Stein-SHIFR gradient flow \eqref{Gad-S-flow}
    starting from $\mu_t|_{t=0} = \mu_0 = $ and $\Phi_t|_{t=0} = \mathbf{0}$ as $M\to\infty$:
\end{proposition}

Similarly, fist proof a technical lemma:
\begin{lemma}\label{lemma_3}
    The following fluid dynamic formulation and particle dynamic formulation is equivalent:
    
    (Suppose that $X_t\sim \mu_t$ and $V_t=\nabla\Phi_t(X_t)$)
   \begin{align}\label{SAIG-euler}
    \left\{
    \begin{aligned}
        &\partial_t \mu_t \!+\nabla \cdot (\mu_t\int k(\cdot,y)\mu_t(y)\nabla_y\Phi_t(y)dy )=0,
        \\
        &\partial_t\Phi_t + \gamma_t\Phi_t
        +\int \nabla\Phi_t(\cdot)^T\nabla\Phi_t(y)k(\cdot,y)\mu_t(y)dy
        +\frac{\delta\mathcal{F}(\mu_t)}{\delta\mu}=0 .
    \end{aligned}\right.
   \end{align}
   \begin{align}\label{SAIG-lagrangian}
    \left\{
    \begin{aligned}
        &\frac{d}{dt}X_t = \int k(X_t,y)\nabla\Phi_t(y)\mu_t(y)dy,
        \\
        &\frac{d}{dt}V_t=-\gamma_tV_t
        -\int V_t^T\nabla\Phi_t(y)\nabla_xk(X_t,y)\mu_t(y)dy
        -\nabla(\frac{\delta\mathcal{F}(\mu_t)}{\delta\mu})(X_t) .
    \end{aligned}\right.
\end{align}

\end{lemma}

\begin{proof}
First we notice the following equation:
\begin{align}\label{lemma3_1}
    \begin{aligned}
        \nabla (\int \nabla\Phi(x)^T\nabla\Phi(y)k(x,y)\mu_t(y)dy)
        &=\nabla^2\Phi(x)\int\nabla\Phi(y)k(x,y)\mu(y)dy
        +\int \nabla\Phi (x)^T\nabla\Phi (y)\nabla_xk(x,y)\mu(y)dy.
    \end{aligned}
\end{align}
Since the first equation of \eqref{SAIG-euler} is actually the continuity equation with velocity 
field $\int k(\cdot,y)\mu_t(y)\nabla_y\Phi_t(y)dy$, it is obvious to have $\frac{d}{dt}X_t = \int k(X_t,y)\nabla\Phi_t(y)\mu_t(y)dy$.
Then we deduce the second equation of \eqref{SAIG-lagrangian}:
\begin{align*}
    \frac{d}{dt}V_t&=\frac{d}{dt}\nabla\Phi_t(X_t)\\
    &\stackrel{\mathrm{(1)}}{=}\partial_t\nabla\Phi_t(X_t)
    +\nabla^2\Phi_t(X_t)(\int k(X_t,y)\nabla\Phi_t(y)\mu_t(y)dy)
    \\ &\stackrel{\mathrm{(2)}}{=}
    -\gamma_tV_t
        -\nabla (\int \nabla\Phi(x)^T\nabla\Phi(y)k(x,y)\mu_t(y)dy)
        -\nabla(\frac{\delta\mathcal{F}(\mu_t)}{\delta\mu})(X_t)
    +\nabla^2\Phi_t(X_t)(\int k(X_t,y)\nabla\Phi_t(y)\mu_t(y)dy)
    \\ &\stackrel{\mathrm{(3)}}{=}
    -\gamma_tV_t
        -\int \nabla\Phi (x)^T\nabla\Phi (y)\nabla_xk(x,y)\mu(y)dy
        -\nabla(\frac{\delta\mathcal{F}(\mu_t)}{\delta\mu})(X_t)
    \\ &\stackrel{\mathrm{(4)}}{=}   
        -\gamma_tV_t
        -\int V_t^T\nabla\Phi_t(y)\nabla_xk(X_t,y)\mu_t(y)dy
        -\nabla(\frac{\delta\mathcal{F}(\mu_t)}{\delta\mu})(X_t) .
\end{align*}
where equation (1) becomes valid from material derivative in fluid dynamic \cite{von2004mathematical},
equation (2) comes from the PDE \eqref{SAIG-euler},
equation (3) comes from leveraging \eqref{lemma3_1},
equation (4) comes from the definition of $V_t$.
\end{proof}
\begin{proof}(Proof of Proposition \ref{proposition_3})
    Substituting Lemma \ref{lemma_1} by Lemma \ref{lemma_3},
    the proof scheme of proposition \ref{proposition_3}
    is the same with the proof scheme of proposition \ref{proposition_1}.
\end{proof}

By discretizing \eqref{particle_formulation_s}
Stein-GAD-PVI algorithm 
updates the positions of particles according to the following rule: 
\begin{align}\label{dynamic_pos_s}
    \xb^i_{k+1} = \xb^i_{k} +  \frac{\eta_{pos}}{M}  \sum_{j=1}^{M}K(\xb^i_{k},\xb^j_{k}) \vb_{k}^i,
\end{align}
and adjusts the velocity field as following:
\begin{align}\label{dynamic_ac_s}
	\vb_{k+1}^i=(1-\gamma\eta_{vel})\vb_k^i
    -\frac{\eta_{vel}}{M} \sum_{j=1}^{M}(\vb_{k}^i)^T\vb_{k}^j\nabla_1K(\xb^i_{k},\xb^j_{k})
    -\eta_{vel}\nabla U_{\tilde{\mu}_k}(\xb_k).
\end{align}

\subsection{B.3 GAD-PVI Algorithms in details}\label{app:detail_algo}
\paragraph{various Dissimilarity Functionals and Smoothing Approaches}
To develop practical GAD-PVI methods, we must first select a dissimilarity functional $\FM$. 
Once a dissimilarity functional $\FM$ has been chosen, 
we need to select a smoothing approach to approximate 
the first variation of the empirical approximation, 
as the value of $\frac{\delta\FM(\cdot)}{\mu}$ at an 
empirical distribution $\tilde{\mu} = \sum_{i=1}^M w^i \delta_{\xb^i}$ is generally not well-defined. 
Smoothing strategies allow us to approximate the first variation value at the discrete empirical distribution.
Generally, a smoothed approximation to the first total variation is denoted as $U_{\tilde{\mu}}(\cdot) \thickapprox \frac{\delta\mathcal{F}(\tilde{\mu})}{\delta\mu}(\cdot)$.
The commonly BLOB (with KL-divergence as $\FM$) \cite{craig2016blob} has been introduced in \eqref{blob-u},
now we give detailed formulations of GFSD (with KL-divergence as $\FM$) \cite{liu2019understanding} and KSDD (with Kernel Stein Discrepancy as $\FM$) \cite{korba2021kernel}, which are all compatible with our GAD-PVI framework.
\begin{algorithm}[hp]
    \caption{General Accelerated Dynamic-weight Particle-based Variational Inference (GAD-PVI) framework in details}
    \label{gadpvi_variants}
    \textbf{Input}: 
    Initial distribution $\tilde{\mu}_0 = \sum_{i=1}^M w^i_0 \delta_{\xb^i_0}$, position adjusting step-size $\eta_{pos}$, weight adjusting step-size $\eta_{wei}$, velocity field adjusting step-size $\eta_{vel}$, velocity damping parameter $\gamma$.	
    \begin{algorithmic}[1] 
    \STATE Choose a suitable functional $\FM$ and 
    its smoothing strategy $U_{\tilde{\mu}} \thickapprox \frac{\delta\mathcal{F}(\tilde{\mu})}{\delta\mu}$
     from KL-BLOB/KL-GFSD/KSD-KSDD\\
     \begin{align*}
        U_{\tilde{\mu}}(\xb)\thickapprox \frac{\delta\mathcal{F}(\tilde{\mu})}{\delta\mu}(\xb) = \left\{ 
        \begin{aligned}
            &-\log{\pi(\xb)} + \frac{\sum_{i=1}^M K(\xb, \xb^i_k)}{\sum_{i=1}^M  K(\xb,\xb^i_k)} + \sum_{i=1}^M{\frac{K(\xb,\xb^i_k)}{\sum_{l=1}^M K(\xb^i_k,\xb^l_k)}}\ (\text{KL-BLOB}), \\
            &-\log{\pi(\xb)} + \frac{\sum_{i=1}^M K(\xb, \xb^i_k)}{\sum_{i=1}^M K(\xb,\xb^i_k)}\ (\text{KL-GFSD}),\\
            &-\textstyle \frac{1}{M}\sum_{i=1}^M  k_\pi(\xb^i_k, \xb)\ (\text{KSD-KSDD}).
        \end{aligned}\right.
    \end{align*}
    \FOR{$k = 0,1,...,T-1$}
    \FOR{$i = 1,2,...,M$}
    \STATE 
    \begin{align*}
        \text{Update positions $\xb^i_k$ according to } \left\{ 
        \begin{aligned}
            &\text{\eqref{dynamic_pos_w}(WGAD)},\\
            &\text{\eqref{dynamic_pos_kw}(KWGAD)},\\
            &\text{\eqref{dynamic_pos_s}(SGAD)}.
        \end{aligned}\right.
    \end{align*}
    \begin{align*}
        \text{Update positions $\vb^i_k$ according to } \left\{ 
        \begin{aligned}
            &\text{\eqref{dynamic_ac_w}(WGAD)},\\
            &\text{\eqref{dynamic_ac_kw}(KWGAD)},\\
            &\text{\eqref{dynamic_ac_s}(SGAD)}.
        \end{aligned}\right.
    \end{align*}
    \ENDFOR 
    \IF{Adopt CA strategy for weight adjustment}
    \STATE Update weights $w^i_k$ according to \eqref{dynamic_weight}
    \ENDIF
    \IF{Adopt DK strategy for weight adjustment}
    \FOR{$i = 1,2,...,M$}
    \STATE Calculate the duplicate/kill rate: $R^i_{k+1} = -\lambda\eta\left(U_{\tilde{\mu}_k}(\xb^i_{k+1}) - \frac{1}{M}\sum_{i=1}^M U_{\tilde{\mu}_k}(\xb^i_{k+1})\right)$
    \ENDFOR 
    \FOR{$i = 1,2,...,M$}
    \IF{$R^i_{k+1}>0$}
        \STATE Duplicate the particle $\xb^i_{k+1}$ with probability $1-\exp(-R^i_{k+1})$ and kill one which is uniformly chosen from the rest.
    \ELSE 
        \STATE Kill the particle $\xb^i_{k+1}$ with probability $1-\exp(R^i_{k+1})$ and duplicate one which is uniformly chosen from the rest.
    \ENDIF
    \ENDFOR

    \ENDIF

    \ENDFOR
    \STATE \textbf{Output}:$\tilde{\mu}_T = \frac{1}{M} \sum_{i=1}^M \delta_{\xb^i_T}$.
    \end{algorithmic}
    \end{algorithm}
\emph{KL-GFSD}
In order to deal with the intractable $\log{\mu(\xb)}$ of the first variation of the KL divergence, GFSD directly approximate $\mu$ by smoothing the empirical distribution $\tilde{\mu}$ with a kernel function $K$: 
    $\hat{\mu} = \tilde{\mu}*K  = \sum_{i=1}^Mw^iK(\cdot, \xb^i)$, which leads to the following approximations: 
    \begin{align}
        U_{\tilde{\mu}_k} (\xb)  = & - \log{\pi(\xb)} + \textstyle\log{\sum_{i=1}^M w^i_k K(\xb,\xb^i_{k})},\label{gfsd-u}
        \\
        \nabla U_{\tilde{\mu}_k} (\xb) = 
        &-\nabla\log{\pi(\xb)} + \frac{\sum_{i=1}^M w^i_k\nabla_{\xb}K(\xb, \xb^i_k)}{\sum_{i=1}^M w^i_k K(\xb,\xb^i_k)}. \label{gfsd-v}
    \end{align}
    \begin{remark}\label{blob-better}
        In the above approximations, we call the terms defined through the interaction with other particles
        as the repulsive terms.
        It can be observed that the BLOB-type approximations \eqref{blob-u} have an extra repulsive term (the term in the second line)
        compared to the GFSD-type approximations \eqref{gfsd-u} and \eqref{gfsd-v}.
        Practically, this extra repulsive term would drive the particles away from each other further,
        and result in a better exploration of particles in the probability space.
        Actually, the BLOB-type methods usually outperforms the GFSD-type methods empirically. 
    \end{remark}
    \begin{remark}\label{rbf}
        Since GFSD and BLOB (partly) smooth the original empirical distribution $\tilde{\mu}$ with a kernel function $K$, the underlying evolutionary 
        distribution is actually a smoothed version of $\tilde{\mu}$. 
    To update the positions and the weights in the smoothed empirical distribution, one should solve a system of linear 
        equations to obtain the new positions $\xb^i_{k+1}$'s and weights $w^i_{k+1}$'s in the $k$-th iteration.
    Nevertheless, with a proper kernel function $K$, such as the RBF kernel, the density $\mu(\xb^i)$ at a given position $\xb^i$
        mainly comes from its corresponding weight $w^i$.
    Actually, as the RBF kernel $e^{-h\|x- x_i\|^2}$ approaches 0 when x becomes far from $x_i$ and equals 1 when $x = x_i$,
     it can be observed that the density at $x_i$ mainly from $w^i$.
    Hence, we can still update the positions and weights in a splitting scheme respectively.
    This approximation performs well in practice.
    \end{remark}

\emph{KSD-KSDD}
Except for the KL-divergence, KSD is recently adopted as the dissimilarity functional in the non-accelerated fixed-weight ParVI method KSDD \cite{korba2021kernel}, 
    whose first variation and the corresponding vector field are defined as
    \begin{align}\label{ksdd_vu}
        \begin{array}{ll}
        \frac{\delta\mathcal{F}(\mu)}{\delta\mu}(\xb) =  \mathbb{E}_{\xb'\sim\mu}\left[k_{\pi}(\xb',\xb)\right],\\
        \nabla\frac{\delta\mathcal{F}(\mu)}{\delta\mu}(\xb) =  \mathbb{E}_{\xb'\sim\mu}\left[\nabla_{\xb}k_{\pi}(\xb',\xb)\right].
        \end{array}
    \end{align}
    Here, $k_{\pi}$ denotes the Stein kernel \cite{liu2016kernelized}, and it is defined by the score of $\pi$: $s(\xb) = \nabla\log{\pi(\xb)}$ and a positive 
        semi-definite kernel function $K$:
    \begin{align*}
        k_\pi(\xb,\yb) = & s(\xb)^Ts(\yb)K(\xb,\yb) + s(\xb)^T\nabla_\yb K(\xb, \yb) + \nabla_\xb K(\xb,\yb)^T s(\yb) + 
            \nabla_\xb \cdot\nabla_\yb K(\xb,\yb).
    \end{align*}

The the first variation and its gradient in \eqref{ksdd_vu} can be directly approximated via the empirical distribution $\tilde{\mu}$.
We construct the following finite-particle approximations: 
    \begin{align*}
        U_{\tilde{\mu}_k} (\xb) = & \textstyle\sum_{i=1}^M w^i_k  k_\pi(\xb^i_{k}, \xb),\\
        \nabla U_{\tilde{\mu}_k} (\xb) =& \textstyle\sum_{i=1}^M w^i_k \nabla_\xb k_\pi(\xb^i_k, \xb).
    \end{align*}

\paragraph{The Detailed GAD-PVI algorithms} 
Adopting different underlying information metric tensors (W-metric, KW-metic and S-metric), weight adjustment approaches(CA and DK) and dissimilarity functionals/associated smoothing approaches(KL-BLOB, KL-GFSD and KSD-KSDD), we can derive 18 different intances of GAD-PVI, named as WGAD/KWGAD/SGAD-CA/DK-BLOB/GFSD/KSDD.
Here we present our General Accelerated Dynamic-weight Particle-based Variational Inference (GAD-PVI) framework, in a more detailed version of Algorithm \ref{alg:algorithm1} as Algorithm \ref{gadpvi_variants}.

\clearpage
\section{Appendix C}
In this section we list the details of experiments setting, parameter tuning and additional results of our empirical studies.
\subsection{C.1 Experiments Settings}
\paragraph{Density of the Gaussian mixture model.}
The density of the Gaussian mixture model is defined as follows: 
\begin{align*}
    \pi(\xb)\propto \frac{2}{3}\exp\left(-\frac{1}{2}\left\|\xb-\ab\right\|^2\right) + \frac{1}{3}\exp\left(-\frac{1}{2}\left\|\xb + \ab\right\|^2\right),
\end{align*}
where $\ab = 1.2 * \mathbf{1}$. 
\paragraph{Density of the Gaussian Process task.}
We follow the experiment setting in \cite{chen2018stein,dpvi}, and use the dataset LIDAR (denoted as $\mathcal{D} = \{(x_i,y_i)\}^N_{i=1}$) which consists of 221 observations 
of scalar variable $x_i$ and $y_i$.
Denote $\xb = [x_1, x_2,...,x_N]^T$ and $\yb = [y_1, y_2, ..., y_{N}]^T$, the target log-posterior w.r.t. the model parameter 
$\phi = (\phi_1, \phi_2)$ is defined as follows: 
\begin{align*}
	\log{p(\phi|\mathcal{D})}\! = \!-\frac{\yb^T \Kb^{-1}_y \yb}{2} \!-\! \frac{\log{\det{(\Kb_y)}}}{2} \!-\! \log{(1\!+\!\xb^T\xb)}.
\end{align*}
Here, $\Kb_y$ is a covariance function $\Kb_y = \Kb + 0.04\mathbf{I}$ with $\Kb_{i,j} = \exp(\phi_1)\exp(-\exp(\phi_2)(x_i - x_j)^2)$ and 
	$\mathbf{I}$ represents the identity matrix. 
\paragraph{Training/Validation/Test dataset in Bayesian neural network.}
For each dataset in the Bayesian neural network task, 
we split it into 90\% training data and 10\% test data randomly, 
which follows the settings from \cite{liu2016stein,zhang2020stochastic,dpvi}. 
Besides, we also randomly choose $1/5$ of the training set as the validation set for parameter tuning.
\paragraph{Initialization of particles' positions.}
In the Gaussian mixture model, we initialize particles according to the standard Gaussian distribution.
In the Gaussian process regression task, 
we initialize particles with mean vector $[0, -10]^T$ and covariance $0.09*\mathbf{I}_{2\times 2}$ for all the algorithms.
As for the Bayesian neural network task, we follow the initialization convention in \cite{liu2016stein,dpvi}.
\paragraph{Bandwidth of Kernel Function in Different Algorithms}
For all the experiments, 
we adopt RBF kernel as the kernel function $K$: $K(\xb, \yb) = \exp(-\|\xb-\yb\|^2_2/h)$, 
where the parameter $h$ is known as the bandwidth \cite{liu2019understanding}. 
We follow the convention in \cite{dpvi} and set the parameter
$h = \frac{1}{M}\sum^M_{i=1}\left(\min_{j\neq i}{\|\xb^i - \xb^j\|^2_2}\right)$ for 
GFSD-type algorithms and 
BLOB-type algorithms.

\paragraph{WNES and WAG}
\citet{liu2019understanding} follows the accelerated gradient descend methods in the Wasserstein propability space \cite{liu2017accelerated,zhang2018estimate} and derives the WNES and WAG methods,
which update the particles' positions with an extra momentum. Though their methods have WNES and WAG type, we only conduct empirical studies of WNES as baseline because the authors report WNES algorithms are usually more robust and efficient than WAG type algorithms\cite{liu2019understanding}. 


\subsection{C.2 Parameters Tuning}

\paragraph{Detailed Settings for $\eta_{pos}$, $\eta_{wei}$, $\eta_{vel}$ and $\gamma$}
Here we present the parameter settings for position adjusting step-size $\eta_{pos}$, weight adjusting step-size $\eta_{wei}$, velocity field adjusting step-size $\eta_{vel}$, velocity damping parameter $\gamma$ of different algorithms are provided in Table \ref{param:gaussian}, \ref{param:gp}, and \ref{param:bnn}. 
All the parameters are chosen by grid search. For the position adjusting step-size $\eta_{pos}$, we first find a suitable range by a coarse-grain grid search and then fine tune it. Note that, the position step-size are tuned via grid search for the fixed-weight ParVI algorithms, then used in the corresponding dynamic-weight algorithms. The acceleration parameters and weight adjustment parameters are tuned via grid
search for each specific algorithm. As a result, it can be observed that the position adjusting step-size for any specific fixed-weight ParVI algorithm, its corresponding dynamic algorithm and the DK variant are the same in these tables.
For ease of understanding, we use the rate of weight adjusting step-size $\eta_{wei}$ divided by the position adjusting step-size $\eta_{pos}$ to illustrate the tuning.
Moreover, inspired by the effective warmup strategy in tuning hyper-parameters, we follow the settings of \cite{dpvi} and construct the weight adjusting step-size parameter schedule using the hyperbolic tangent: $\lambda \tanh{(2*(t/T)^5)}$, with $t$ being the current time step and $T$ the total number of steps.

\begin{table}[tb]
	\centering

	\begin{tabular}{c|cc}
		\toprule
		\multirow{2}*{Algorithm} & \multicolumn{2}{c}{tasks}  \\
		~    & Single Gaussian & Gaussian Mixture Model \\
		\hline
		ParVI-BLOB & 	1.0e-2, -, -, - & 1.0e-2, -, -, -    \\
		WAIG-BLOB & 	1.0e-2, -, 1.0, 0.3  & 1.0e-2, -, 1.0, 0.3  \\

		WNES-BLOB & 	1.0e-2, -, 1.0, 0.2   & 1.0e-2, -, 1.0, 0.2  \\
		DPVI-CA-BLOB & 	1.0e-2, 1.0, -, -  & 1.0e-2, 1.0, -, -  \\
        DPVI-DK-BLOB & 	1.0e-2, 1.0, -, - & 1.0e-2, 1.0, -, - \\
        W\NAME-CA-BLOB & 	1.0e-2, 1.0, 1.0, 0.3  & 1.0e-2, 1.0, 1.0, 0.3  \\
        W\NAME-DK-BLOB & 	1.0e-2, 5e-2, 1.0, 0.3  & 1.0e-2, 5e-2, 1.0, 0.3  \\
        \hline
        KWAIG-BLOB & 	1.0e-2, -, -, -  & 1.0e-2, -, -, -  \\
        KW\NAME-CA-BLOB & 	1.0e-2, 5e-3, 1.0, 0.9   & 1.0e-2, 5e-3, 1.0, 0.9 \\
        KW\NAME-DK-BLOB & 	1.0e-2, 5e-2, 1.0, 0.9  & 1.0e-2, 5e-2, 1.0, 0.9   \\
        \hline
        SAIG-BLOB & 	 5.0e-2, -, -, -  & 2.5e-2, -, -, -  \\
        S\NAME-CA-BLOB & 	5.0e-2, 5e-3, 1.0, 0.9  & 2.5e-2, 5e-3, 1.0, 0.9  \\
        S\NAME-DK-BLOB & 	5.0e-2, 5e-2, 1.0, 0.9  & 2.5e-2, 5e-2, 1.0, 0.9 \\
        \hline
        \hline
        ParVI-GFSD & 	1.0e-2, -, -, - & 1.0e-2, -, -, -  \\
		WAIG-GFSD & 	1.0e-2, -, 1.0, 0.3  & 1.0e-2, -, 1.0, 0.3 \\
		WNES-GFSD & 	1.0e-2, -, 1.0, 0.2  & 1.0e-2, -, 1.0, 0.2  \\
		DPVI-CA-GFSD & 	1.0e-2, 1.0, -, -  & 1.0e-2, 0.8, -, -   \\
        DPVI-DK-GFSD & 	1.0e-2, 1.0, -, -  & 1.0e-2, 1.0, -, - \\
        W\NAME-CA-GFSD & 	1.0e-2, 1.0, 1.0, 0.3  & 1.0e-2, 0.8, 1.0, 0.3  \\
        W\NAME-DK-GFSD & 	1.0e-2, 5e-2, 1.0, 0.3  & 1.0e-2, 5e-2, 1.0, 0.3  \\
        \hline
        KWAIG-GFSD & 	1.0e-2, -, -, -  & 1.0e-2, -, -, -  \\
        KW\NAME-CA-GFSD & 	1.0e-2, 5e-3, 1.0, 0.9  & 1.0e-2, 5e-3, 1.0, 0.9  \\
        KW\NAME-DK-GFSD & 	1.0e-2, 5e-2, 1.0, 0.9  & 1.0e-2, 5e-2, 1.0, 0.9  \\
        \hline
        SAIG-GFSD & 	  5.0e-2, -, -, - & 2.5e-2, -, -, - \\
        S\NAME-CA-GFSD & 	5.0e-2, 5e-3, 1.0, 0.9  & 2.5e-2, 5e-3, 1.0, 0.9  \\
        S\NAME-DK-GFSD & 	5.0e-2, 5e-2, 1.0, 0.9  & 2.5e-2, 5e-2, 1.0, 0.9  \\
        
		\bottomrule
	\end{tabular}
	\caption{Parameters of different algorithms in SG and GMM($\eta_{pos}$, $\frac{\eta_{wei}}{\eta_{pos}}$, $\eta_{vel}$ and $\gamma$).}
	\label{param:gaussian}
	\vspace{-5mm}
\end{table}

\begin{table}[tb]
	\centering

	\begin{tabular}{c|cc}
		\toprule
		\multirow{2}*{Algorithm} & \multicolumn{2}{c}{Smoothing Approaches}  \\
		~    & BLOB &GFSD  \\
		\hline
		ParVI & 	1.0e-2, --, --, --  & 1.0e-2, --, --, --  \\
		WAIG & 	1.0e-2, --, 1.0, 0.4  & 1.0e-2, --, 1.0, 0.3  \\
		WNES & 	1.0e-2, --, 1.0, 0.4  & 1.0e-2, --, 1.0, 0.4  \\
		DPVI-CA & 	1.0e-2, 0.1, --, --  & 1.0e-2, 0.3, --, -- \\
        DPVI-DK & 	1.0e-2, 0.01, --, --  & 1.0e-2, 0.01, --, --  \\
        W\NAME-CA & 	1.0e-2, 0.1, 1.0, 0.4,   & 1.0e-2, 0.3, 1.0, 0.3  \\
        W\NAME-DK & 	1.0e-2, 0.01, 1.0, 0.4,   & 1.0e-2, 0.01, 1.0, 0.3  \\
        \hline
        KWAIG & 	5.0e-3, --, 1.0, 0.8  & 1.0e-3, --, 1.0, 0.7  \\
        KW\NAME-CA & 	5.0e-3, 0.1, 1.0, 0.8   & 1.0e-3, 0.3, 1.0, 0.7   \\
        KW\NAME-DK & 	5.0e-3, 0.01, 1.0, 0.8   & 1.0e-3, 0.01, 1.0, 0.7   \\
        \hline
        SAIG & 	 2.0e-2, --, 1.0, 0.7 & 1.0e-2, --, 1.0, 0.6  \\
        S\NAME-CA & 	2.0e-2, 0.1, 1.0, 0.7   & 1.0e-2, 0.3, 1.0, 0.6,  \\
        S\NAME-DK & 	2.0e-2, 0.01, 1.0, 0.7   & 1.0e-2, 0.01, 1.0, 0.6  \\
		\bottomrule
	\end{tabular}
	\caption{Parameters of different algorithms in GP($\eta_{pos}$, $\frac{\eta_{wei}}{\eta_{pos}}$, $\eta_{vel}$ and $\gamma$).}
	\label{param:gp}
	\vspace{-5mm}
\end{table}

\begin{table}[tb]
	\centering

	\begin{tabular}{c|cccc}
		\toprule
		\multirow{2}*{Algorithm} & \multicolumn{4}{c}{Datasets}  \\
		~    & Concrete & kin8nm & RedWine & space \\
		\hline
		ParVI-BLOB & 	4.0e-6, --, --, --  & 1.0e-6, --, --, -- & 3.4e-6, --, --, -- & 3.0e-6, --, --, -- \\
		WAIG-BLOB & 	4.0e-6, --, 1.0, 0.2  & 1.0e-6, --, 1.0, 0.3 & 3.4e-6, --, 1.0, 0.5 & 3.0e-6, --, 1.0, 0.5 \\
		WNES-BLOB & 	4.0e-6, --, 1.0, 0.3  & 1.0e-6, --, 1.0, 0.2 & 3.4e-6, --, 1.0, 0.2 & 3.0e-6, --, 1.0,  0.2 \\
		DPVI-CA-BLOB & 	4.0e-6, 1.0, --, --  & 1.0e-6, 0.8, --, -- & 3.4e-6, 0.5, --, -- & 3.0e-6, 1.0, --, -- \\
        DPVI-DK-BLOB & 	4.0e-6, 1.0, --, --  & 1.0e-6, 0.8, --, -- & 3.4e-6, 0.5, --, -- & 3.0e-6, 1.0, --, -- \\
        W\NAME-CA-BLOB & 	4.0e-6, 1.0, 1.0, 0.2   & 1.0e-6, 0.8, 1.0, 0.3 & 3.4e-6, 0.5, 1.0, 0.5 & 3.0e-6, 1.0, 1.0, 0.5 \\
        W\NAME-DK-BLOB & 	4.0e-6, 1.0, 1.0, 0.2  & 1.0e-6, 0.8, 1.0, 0.3 & 3.4e-6, 0.1, 1.0, 0.5 & 3.0e-6, 1.0, 1.0, 0.5  \\
        \hline
        KWAIG-BLOB & 	4.0e-6, --, 1.0, 0.7  & 1.0e-6, --, 1.0, 0.3 & 3.4e-6, --, 1.0, 0.5 & 2.0e-6, --, 1.0, 0.8 \\
        KW\NAME-CA-BLOB & 	4.0e-6, 1.0, 1.0, 0.7  & 1.0e-6, 0.8, 1.0, 0.3 & 3.4e-6, 0.5, 1.0, 0.5 & 2.0e-6, 1.0, 1.0, 0.8 \\
        KW\NAME-DK-BLOB & 	4.0e-6, 1.0, 1.0, 0.7  & 1.0e-6, 0.8, 1.0, 0.3 & 3.4e-6, 0.1, 1.0, 0.5 & 2.0e-6, 1.0, 1.0, 0.8  \\
        \hline
        SAIG-BLOB & 	 1.0e-5, --, 1.0, 0.6 & 2.0e-6, --, 1.0, 0.3 & 7.0e-6, --, 1.0, 0.5 & 8.0e-6, --, 1.0, 0.8 \\
        S\NAME-CA-BLOB & 	1.0e-5, 1.0, 1.0, 0.6  & 2.0e-6, 0.8, 1.0, 0.3 & 7.0e-6, 0.5, 1.0, 0.5 & 8.0e-6, 1.0, 1.0, 0.8 \\
        S\NAME-DK-BLOB & 	1.0e-5, 1.0, 1.0, 0.6   & 2.0e-6, 0.8, 1.0, 0.3 & 7.0e-6, 0.1, 1.0, 0.5 & 8.0e-6, 1.0, 1.0, 0.8 \\
        \hline
        \hline
        ParVI-GFSD & 	4.0e-6, --, --, --  & 1.0e-6, --, --, -- & 3.4e-6, --, --, -- & 3.0e-6, --, --, -- \\
		WAIG-GFSD & 	4.0e-6, --, 1.0, 0.1  & 1.0e-6, --, 1.0, 0.3 & 3.4e-6, --, 1.0, 0.5 & 3.0e-6, --, 1.0, 0.5 \\
		WNES-GFSD & 	4.0e-6, --, 1.0, 0.3  & 1.0e-6, --, 1.0, 0.3 & 3.4e-6, --, 1.0, 0.2 & 3.0e-6, --, 1.0, 0.2 \\
		DPVI-CA-GFSD & 	4.0e-6, 1.0, --, --  & 1.0e-6, 0.8, --, -- & 3.4e-6, 0.5, --, -- & 3.0e-6, 1.0, --, -- \\
        DPVI-DK-GFSD & 	4.0e-6, 1.0, --, --  & 1.0e-6, 0.8, --, -- & 3.4e-6, 0.5, --, -- & 3.0e-6, 1.0, --, -- \\
        W\NAME-CA-GFSD & 	4.0e-6, 1.0, 1.0, 0.1  & 1.0e-6, 0.8, 1.0, 0.3 & 3.4e-6, 0.5, 1.0, 0.5 & 3.0e-6, 1.0, 1.0, 0.5 \\
        W\NAME-DK-GFSD & 	4.0e-6, 1.0, 1.0, 0.1   & 1.0e-6, 0.8, 1.0, 0.3 & 3.4e-6, 0.1, 1.0, 0.5 & 3.0e-6, 1.0, 1.0, 0.5 \\
        \hline
        KWAIG-GFSD & 	4.0e-6, --, 1.0, 0.6  & 1.0e-6, --, 1.0, 0.3 & 3.4e-6, --, 1.0, 0.5 & 3.0e-6, --, 1.0, 0.3 \\
        KW\NAME-CA-GFSD & 	4.0e-6, 1.0, 1.0, 0.6  & 1.0e-6, 0.8, 1.0, 0.3 & 3.4e-6, 0.5, 1.0, 0.5 & 3.0e-6, 1.0, 1.0, 0.3 \\
        KW\NAME-DK-GFSD & 	4.0e-6, 1.0, 1.0, 0.6  & 1.0e-6, 0.8, 1.0, 0.3 & 3.4e-6, 0.1, 1.0, 0.5 & 3.0e-6, 1.0, 1.0, 0.3 \\
        \hline
        SAIG-GFSD & 	 4.0e-6, --, 1.0, 0.2 & 2.0e-6, --, 1.0, 0.3 & 7.0e-6, --, 1.0, 0.5 & 6.0e-6, --, 1.0, 0.3 \\
        S\NAME-CA-GFSD & 	4.0e-6, 1.0, 1.0, 0.2  & 2.0e-6, 0.8, 1.0, 0.3 & 7.0e-6, 0.5, 1.0, 0.5 & 6.0e-6, 1.0, 1.0, 0.3 \\
        S\NAME-DK-GFSD & 	4.0e-6, 1.0, 1.0, 0.2   & 2.0e-6, 0.8, 1.0, 0.3 & 7.0e-6, 0.1, 1.0, 0.5 & 6.0e-6, 1.0, 1.0, 0.3 \\
		\bottomrule
	\end{tabular}
	\caption{Parameters of different algorithms in BNN($\eta_{pos}$, $\frac{\eta_{wei}}{\eta_{pos}}$, $\eta_{vel}$ and $\gamma$).}
	\label{param:bnn}
	\vspace{-5mm}
\end{table}

\subsection{C.3 Additional Experiments Results}\label{additional_result}
\paragraph{Results for SG}
In this section, we give empirical results on approximating a single-mode Gaussian distribution, whose density is defined as: 
\begin{align*}
    \pi(\xb)\propto \exp(-\frac{1}{2}\xb^T\Sigma^{-1}\xb), 
\end{align*}
where $\Sigma_{ii} = 1.0$ and correlation $\Sigma_{ij, i\neq j} = 0.8$. 
To investigate the influence of  number $M$ in this task, we run all the algorithms with $M\in\{32, 64, 128, 256, 512\}$.
All the particles are initialized from a Gaussian distribution with zero mean and covariance matrix $0.5*\mathbf{I}_{10\times 10}$.

In Figure \ref{figure_sg_w} and \ref{figure_sg_kws}, we plot the $W_2$ distance to the target of the samples generated by each algorithm w.r.t. iteration and time.
We generate 5000 samples from the target distribution $\pi$ as reference to evaluate $W_2$.
As this task is a simple Single Gaussian model, the approximation error difference between the GAD-PVI algorithms and fixed-weight ones is not so obvious. When particles number gets more, the effect of the dynamic weight adjustment scheme is smaller, for such a large number of fix-weight particles is also sufficed for approximating this simple distribution. Meanwhile, the faster convergence effect of the accelerated position update strategies is quite obvious in the figures.
Moreover, we can see that the GFSD-type algorithms cannot outperform the BLOB-type algorithms and SVGD, this may due to the lack of the repulsive mechanism in GFSD which lead to particle system collapse in single-mode task. This also coincide with the discussion in Appendix B.3.

In Table \ref{sg_result}, we further report the final $W_2$ distance between the empirical distribution generated by each algorithm and the target distribution. It can be observed that, in the Wasserstein metric case, CA strategy constantly outperform their fixed-weight counterparts with the same number of particles, the DK variants are weakened due to single-modality of this task. However, in the KW or S metric case, it can be observed that the GAD-DK algorithms outperform than others in majority of cases, this is due to the poor transportation ability of KW and S metric, a direct duplicate/kill mechanism greatly enhance the transport speed from low probability region to high probability region. We also find that the KW and Stype algorithms achieve poor approximation result comparing to WGAD algorithms in terms of the Wasserstein distance to reference points, this may because the WGAD algorithms aim to implicitly minimize the Wasserstein distance, and this is not the case for other two type.

\begin{table}[tb]
	\centering
    \resizebox{1.0\textwidth}{!}{
	\begin{tabular}{c|ccccc}
		\toprule
		\multirow{2}*{Algorithm} & \multicolumn{5}{c}{Number of particles}  \\
		~    & 32 & 64 & 128 & 256 & 512 \\
		\hline
        ParVI-SVGD & 	1.320e+00 & 1.228e+00  & 1.162e+00 & 1.102e+00 & 1.027e+00\\
        \hline
		ParVI-BLOB & 1.315e+00 & 1.229e+00  & 1.164e+00 & 1.102e+00 & 1.038e+00\\
		WAIG-BLOB &  1.314e+00	& 1.230e+00  & 1.164e+00 & 1.100e+00 & 1.038e+00\\
		WNes-BLOB & 1.315e+00	& 1.230e+00  & 1.164e+00 & 1.101e+00 & 1.037e+00\\
        DPVI-DK-BLOB & 1.313e+00	& 1.229e+00  & 1.163e+00 & 1.100e+00 & 1.035e+00\\
        DPVI-CA-BLOB & 1.309e+00	&1.227e+00  & 1.162e+00 & 1.102e+00 &1.037e+00\\
        WGAD-DK-BLOB(Ours) & 	1.313e+00   &1.229e+00  & 1.163e+00 & \textbf{1.098e+00} & \textbf{1.034e+00}\\
        WGAD-CA-BLOB(Ours) & 	\textbf{1.300e+00} &\textbf{1.226e+00}  & \textbf{1.161e+00} & 1.099e+00  & 1.036e+00\\
        \hline
        KWAIG-BLOB & 	1.943e+00  & 1.922e+00                          & 1.889e+00 &1.827e+00 & 1.777e+00\\
        KWGAD-DK-BLOB(Ours) & 	\textbf{1.920e+00}  & \textbf{1.884e-02} & \textbf{1.848e+00}& \textbf{1.798e+00}& \textbf{1.757e+00} \\
        KWGAD-CA-BLOB(Ours) & 	1.942e+00  & 1.901e+00                 & 1.865e+00& 1.816e+00 & 1.764e+00\\

        \hline
        SAIG-BLOB & 	 1.451e+00 & 1.412e+00                      & 1.429e+00 &1.436e+00 & 1.479e+00\\
        SGAD-DK-BLOB(Ours) &\textbf{1.435e+00}  & \textbf{1.355e+00} & \textbf{1.341e+00}& \textbf{1.219e+00} & \textbf{1.143e+00} \\
        SGAD-CA-BLOB(Ours) & 	1.444e+00  & 1.396e+00              & 1.412e+00& 1.405e+00 & 1.407e+00\\

        \hline
        \hline
		ParVI-GFSD & 	1.453e+00  & 1.353e+00 & 1.267e+00 & 1.198e+00 & 1.136e+00\\
		WAIG-GFSD & 	1.449e+00  & 1.353e+00 & 1.264e+00 & 1.196e+00 & 1.134e+00\\
		WNes-GFSD & 	1.450e+00  & 1.353e+00 & 1.265e+00 & 1.197e+00 & 1.135e+00\\
        DPVI-DK-GFSD & 	1.448e+00  & 1.347e+00 & 1.267e+00 & 1.197e+00 & 1.135e+00\\
        DPVI-CA-GFSD & 	1.446e+00  & 1.349e+00 & 1.259e+00 & 1.195e+00 & 1.133e+00\\
        WGAD-DK-GFSD(Ours) & 	1.448e+00  & 1.345e+00 & 1.264e+00 & 1.195e+00 & 1.134e+00\\
        WGAD-CA-GFSD(Ours) & 	\textbf{1.398e+00}  & \textbf{1.332e+00} & \textbf{1.252e+00} & \textbf{1.191e+00} & \textbf{1.131e+00}\\
        \hline
        KWAIG-GFSD & 	2.246e+00  & 2.171e+00        & 2.134e+00           & 2.082e-02 & 2.046e+00 \\
        KWGAD-DK-GFSD(Ours) & 	2.215e+00  & \textbf{2.154e+00} & \textbf{2.092e+00} &\textbf{2.073e-02} & \textbf{2.022e+00} \\
        KWGAD-CA-GFSD(Ours) &\textbf{2.204e+00} &2.160e+00 & 2.110e+00 &2.081e-02 & 2.035e+00 \\
        
        \hline
        SAIG-GFSD & 	 1.823e+00              & 1.789e+00 & 1.760e+00       & 1.626e-02 & 1.609e+00\\
        SGAD-DK-GFSD(Ours) & 1.881e+00  & \textbf{1.782e+00} & \textbf{1.609e+00} & \textbf{1.437e-02} & \textbf{1.401e+00} \\
        SGAD-CA-GFSD(Ours) & 	\textbf{1.820e+00}        & 1.783e+00    & 1.720e+00      & 1.602e-02 & 1.592e+00 \\
		\bottomrule
	\end{tabular}}
	\caption{Averaged Test $W_2$ distances for different ParVI methods in SG task.}
	\label{sg_result}
	\vspace{-5mm}
\end{table}

\begin{figure}[ht!]
    \centering
    \begin{subfigure}{\linewidth}
        \includegraphics[width=.45\linewidth]{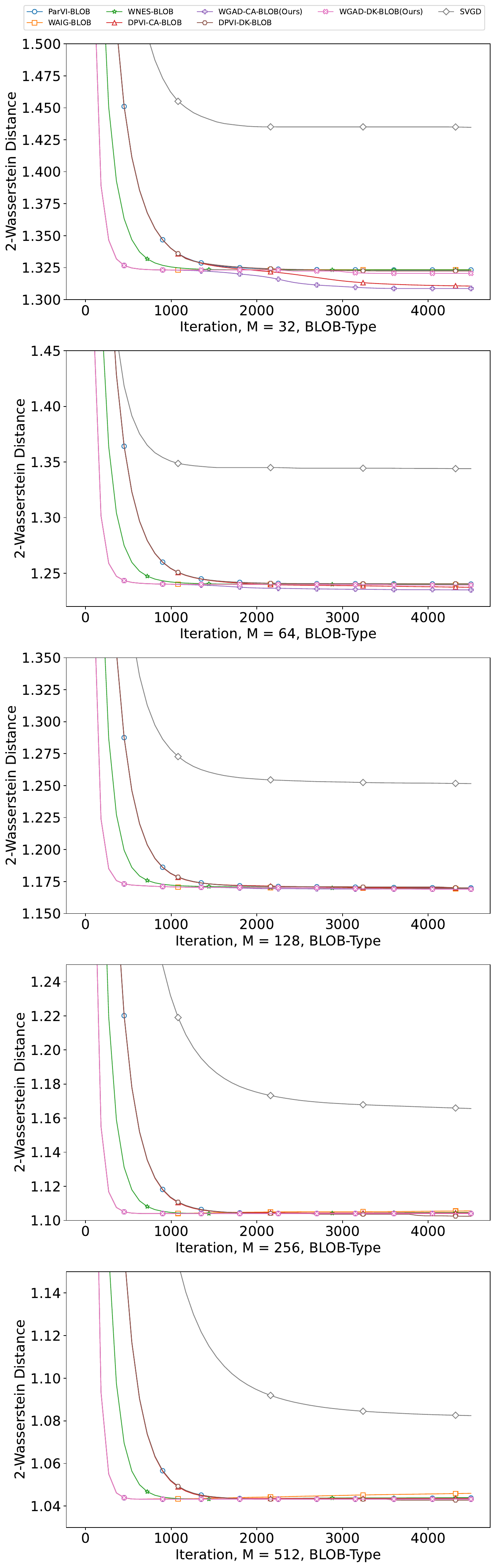}
        \includegraphics[width=.45\linewidth]{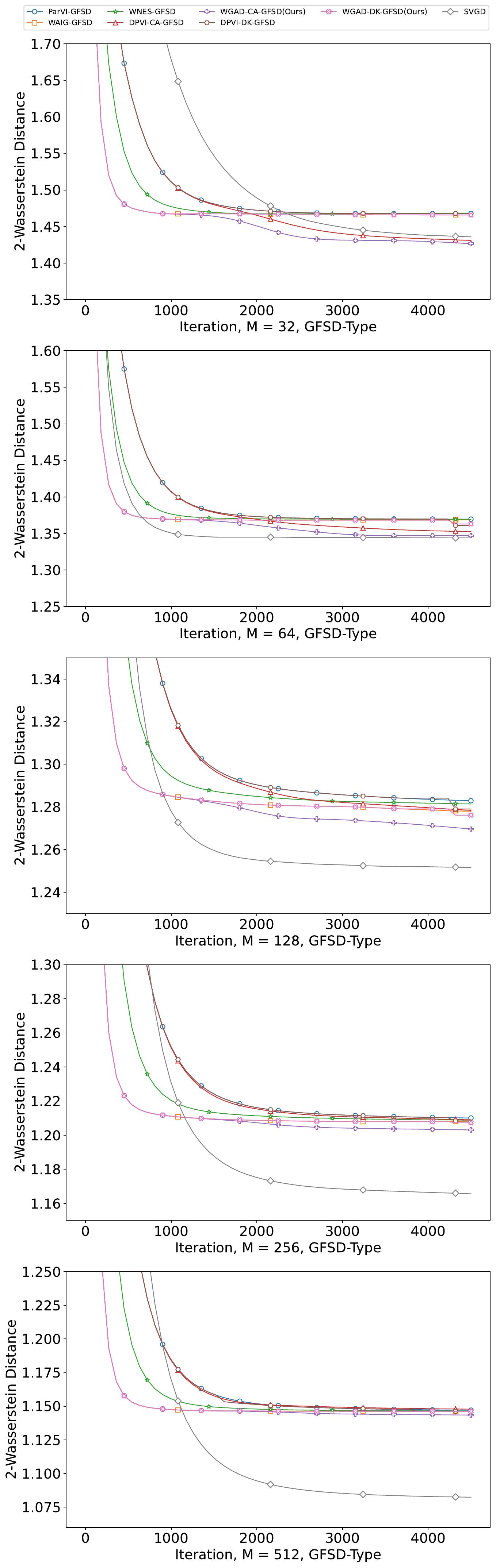}
    \end{subfigure}\qquad
    \caption{Averaged Test $W_2$ distance to the target w.r.t. iterations in the SG task for algorithms(W-Type).}
    \vspace{-5mm}
    \label{figure_sg_w}
    \end{figure}
\begin{figure}[ht!]
    \centering
    \begin{subfigure}{\linewidth}
        \includegraphics[width=.24\linewidth]{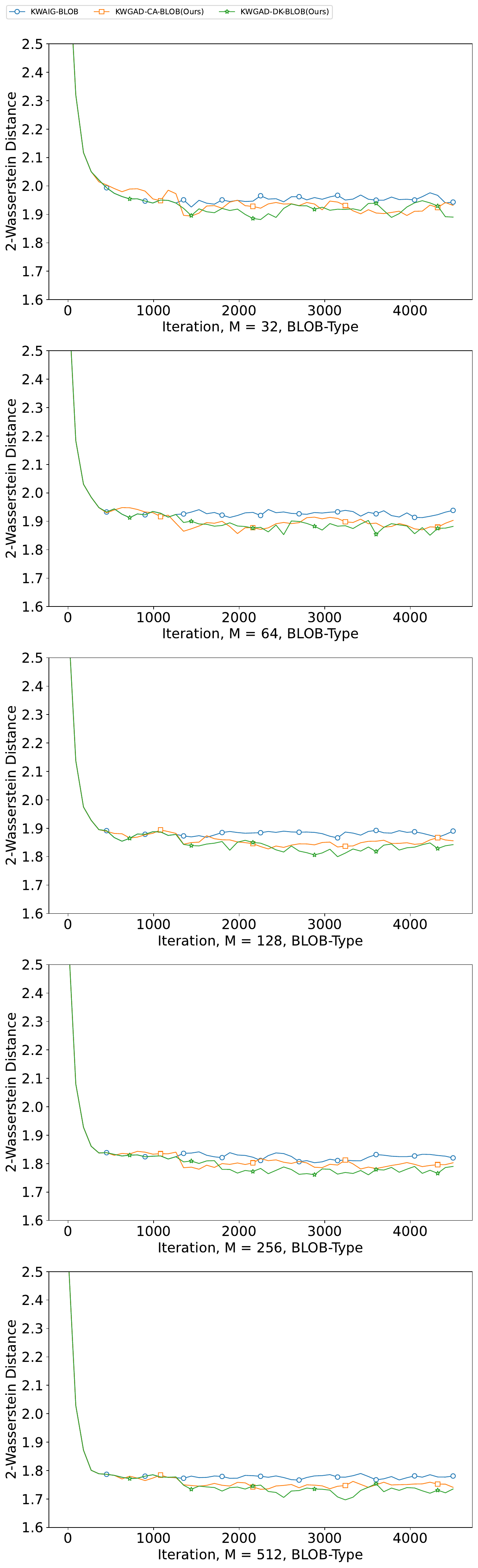}
        \includegraphics[width=.24\linewidth]{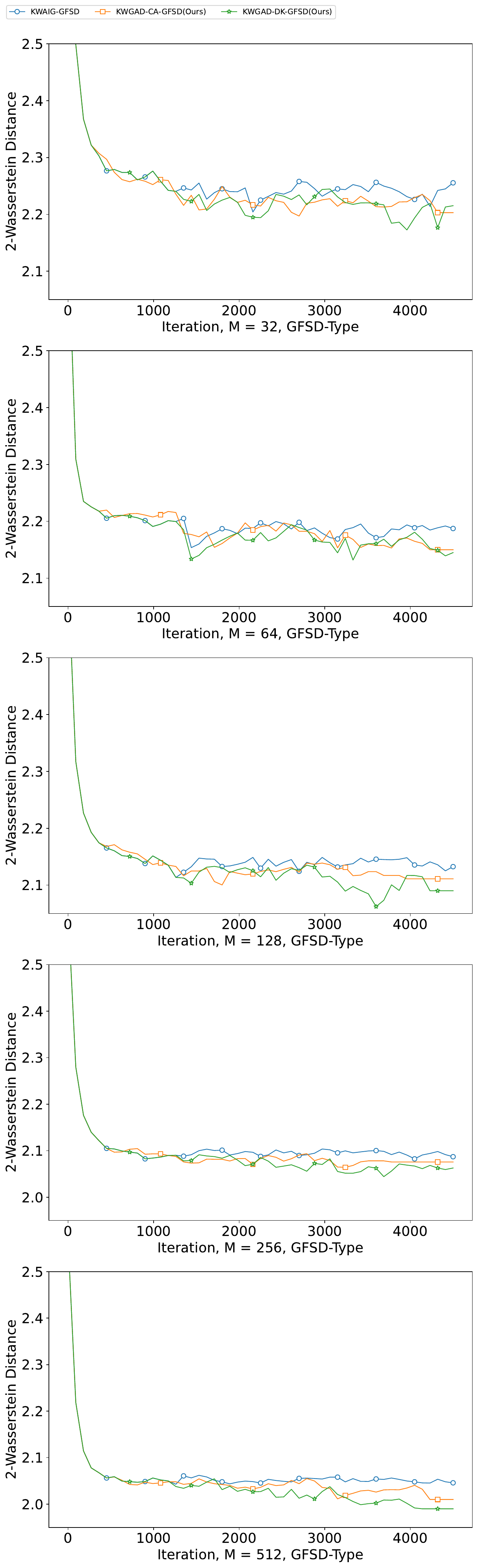}
        \includegraphics[width=.24\linewidth]{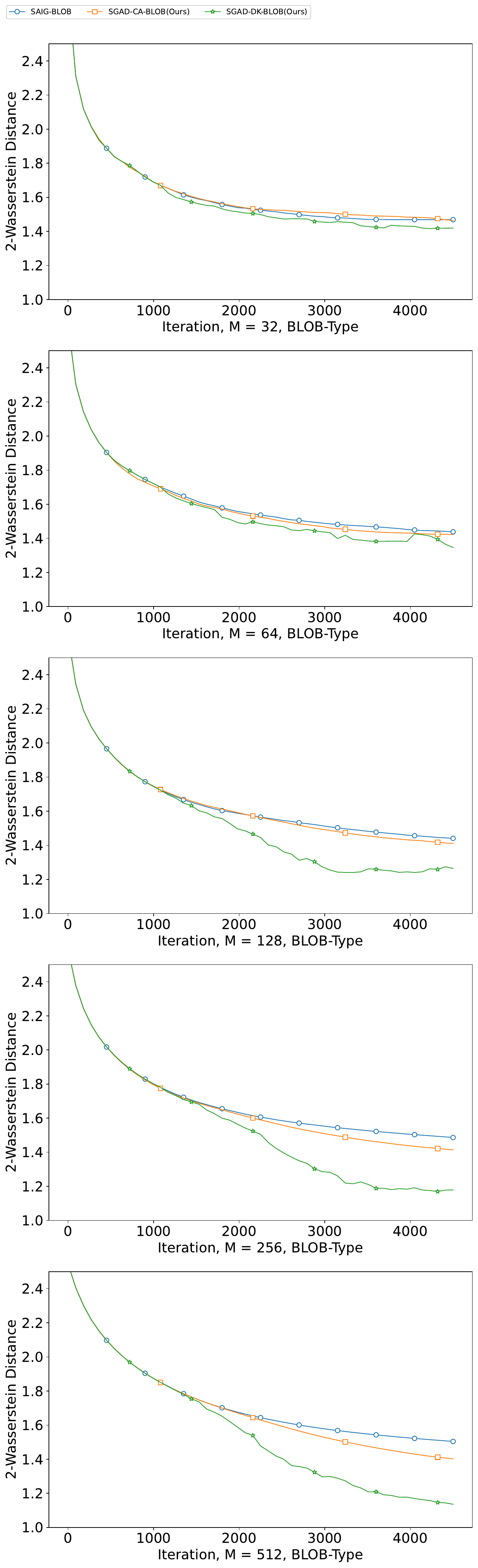}
        \includegraphics[width=.24\linewidth]{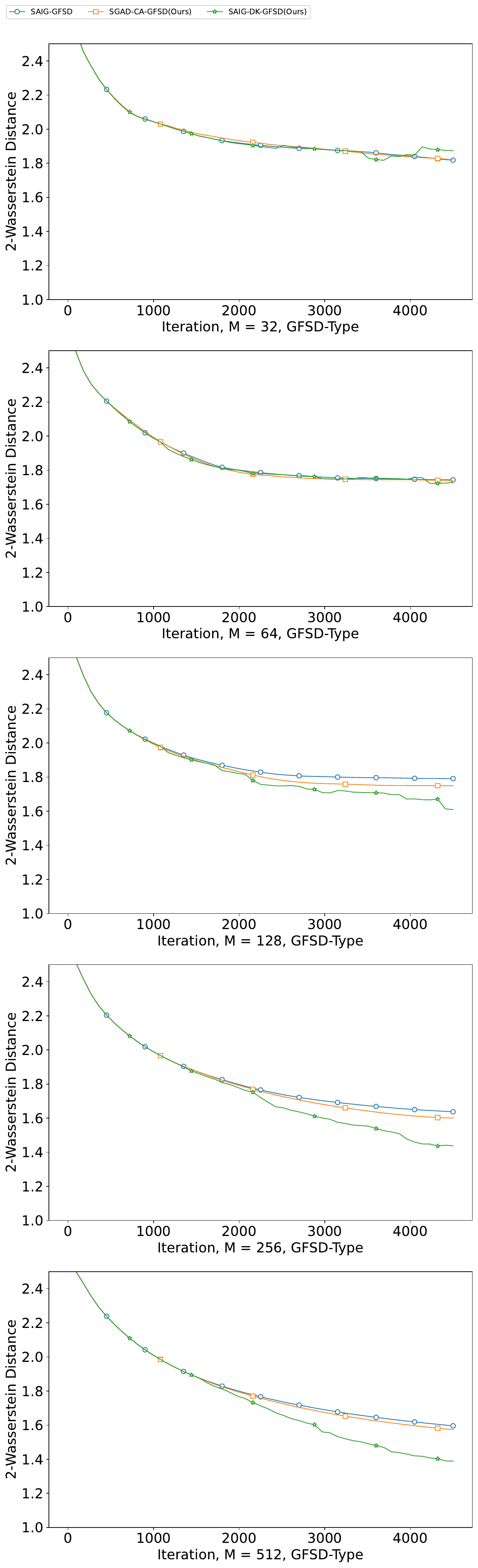}
    \end{subfigure}\qquad
    \caption{Averaged Test $W_2$ distance to the target w.r.t. iterations in the SG task for algorithms(KW/S-Type).}
    \vspace{-5mm}
    \label{figure_sg_kws}
    \end{figure}

\paragraph{Additional Results for GMM}
We provide additional results for Gaussian Mixture Model experiment. 
In Figure \ref{figure_mg} and Figure \ref{figure_mg_kws}, we plot the $W_2$ w.r.t iteration of each algorithm for all $M=\{32,63,128,256,512\}$. From these figures, we can observe that, compared with the baseline ParVI algorithms,
our GAD-PVI of either CA strategy or DK variants
result in a better performance, i.e., less approximation error and faster convergence.
Actually, as we have discussed in the methodology part, while the weight-adjustment step in GAD-PVI greatly enhances the expressiveness of particles' empirical distribution, the accelerated position update strategy also bring in faster convergence. From Figure \ref{figure_mg}, we find that DK variants decrease quite fast at first in the WGAD case, that is because the duplicate/kill scheme will greatly enhance the particle transport ability thus move more particles to high probability region at first comparing to moving particle step by step. In the Figure \ref{figure_mg_kws}, we observe that DK variants also show better result comparing to other algorithms, this may because of the slow transportation property of KW and S type algorithms.

In Table \ref{mg_result}, we further report the final $W_2$ distance between the empirical distribution generated by each algorithm and the target distribution. It can be observed that GAD-PVI algorithms constantly achieve better approximation result than existing algorithms. Notably, for this complex multi-mode task, DK variants show their advantage as the duplicate/kill operation allows transferring particles from low-probability region to distant high-probability area (e.g. among different local modes) especially in KW/S case. However, the DK variants in Wasserstein case are not so competitive with CA algorithms, this difference could lie in that the KW/S metric space are much more influenced by the potential barrier and need DK scheme to transport particles among different local modes but Wasserstein metric are more robust from multi-modality and CA strategy suffice. Besides, the GAD-PVI algorithms with the CA strategy are usually more stable than their counterpart with DK, which may be ascribed to the fluctuations induced by the discrete weight adjustment (0 or $1/M$) in DK.

\begin{table}[tb]
	\centering
    \resizebox{1.0\textwidth}{!}{
	\begin{tabular}{c|ccccc}
		\toprule
		\multirow{2}*{Algorithm} & \multicolumn{5}{c}{Number of particles}  \\
		~    & 32 & 64 & 128 & 256 & 512 \\
		\hline
        ParVI-SVGD & 	2.175e+00 & 2.101e+00  & 2.088e+00 & 2.026e+00 & 2.044e+00\\
        \hline
		ParVI-BLOB & 	2.317e+00 & 2.779e+00  & 2.292e+00 & 2.440e+00 & 2.294e+00\\
		WAIG-BLOB &  2.317e+00	& 2.775e+00  & 2.039e+00 & 1.976e+00 & 1.845e+00\\
		WNes-BLOB & 2.317e+00	& 2.777e+00  & 2.213e+00 & 2.329e+00 & 2.180e+00\\
        DPVI-DK-BLOB & 2.065e+00	& 2.066e+00  & 1.859e+00 & 1.735e+00 & 1.704e+00\\
        DPVI-CA-BLOB & 2.039e+00	&1.934e+00  & 1.825e+00 & 1.727e+00 & 1.633e+00\\
        WGAD-DK-BLOB(Ours) & 	2.064e+00 &1.933e+00  & 1.831e+00 & 1.727e+00 & 1.650e+00\\
        WGAD-CA-BLOB(Ours) & 	\textbf{2.037e+00} &\textbf{1.929e+00}  & \textbf{1.824e+00} & \textbf{1.725e+00}  & \textbf{1.632e+00}\\
        \hline
        KWAIG-BLOB & 	4.937e+00  & 4.674e+00                          & 4.600e+00 & 4.242e+00 & 4.221e+00\\
        KWGAD-DK-BLOB(Ours) & \textbf{2.854e+00} & \textbf{2.622e+00} & \textbf{2.400e+00}& \textbf{2.542e+00}& \textbf{2.248e+00} \\
        KWGAD-CA-BLOB(Ours) & 	4.565e+00  & 4.206e+00                 & 4.094e+00& 3.836e+00 & 3.767e+00\\

        \hline
        SAIG-BLOB & 	 5.070e+00 & 4.632e+00                      & 4.554e+00 & 4.140e+00 & 4.032e+00\\
        SGAD-DK-BLOB(Ours) &\textbf{2.863e+00}  & \textbf{2.914e+00} & 2.760e+00& \textbf{1.1890e+00} & \textbf{2.247e+00} \\
        SGAD-CA-BLOB(Ours) & 	3.406e+00  & 3.051e+00              & \textbf{2.659e+00}& 2.595e+00 & 2.492e+00\\

        \hline\hline
		ParVI-GFSD & 	2.427e+00  & 2.888e+00 & 2.331e+00 & 2.567e+00 & 2.398e+00\\
		WAIG-GFSD & 	2.425e+00  & 2.885e+00 & 2.328e+00 & 2.206e+00 & 2.094e+00\\
		WNes-GFSD & 	2.426e+00  & 2.888e+00 & 2.330e+00 & 2.494e+00 & 2.337e+00\\
        DPVI-DK-GFSD & 	2.151e+00  & 2.025e+00 & 1.924e+00 & 1.837e+00 & 1.769e+00\\
        DPVI-CA-GFSD & 	2.134e+00  & 2.025e+00 & 1.928e+00 & 1.838e+00 & 1.755e+00\\
        WGAD-DK-GFSD(Ours) & 	2.150e+00  & \textbf{2.017e+00} & 1.924e+00 & \textbf{1.834e+00} & 1.755e+00\\
        WGAD-CA-GFSD(Ours) & 	\textbf{2.120e+00}  & 2.019e+00 & \textbf{1.923e+00} & 1.835e+00 & \textbf{1.754e+00}\\
        \hline
        KWAIG-GFSD & 	5.072e+00  & 4.780e+00                        & 4.706e+00           & 4.381e+00 & 4.359e+00 \\
        KWGAD-DK-GFSD(Ours) &\textbf{3.086e+00} & \textbf{2.682e+00} & \textbf{2.817e+00} &\textbf{2.385e+00} & \textbf{2.393e+00} \\
        KWGAD-CA-GFSD(Ours) & 4.597e+00  & 4.389e+00                 & 4.262e+00    &  3.960e+00 & 3.929e+00 \\
        
        \hline
        SAIG-GFSD & 	 4.828e+00              & 4.577e+00          & 4.555e+00       & 4.151e+00 & 4.075e+00\\
        SGAD-DK-GFSD(Ours) & \textbf{2.994e+00}  & \textbf{3.411e+00} & \textbf{2.881e+00} & \textbf{2.347e+00} & \textbf{2.324e+00} \\
        SGAD-CA-GFSD(Ours) & 	3.937e+00        & 4.142e+00    & 3.676e+00      & 3.617e+00 & 4.007e+00 \\
		\bottomrule
	\end{tabular}}
	\caption{Averaged Test $W_2$ distances for different ParVI methods in GMM task.}
	\label{mg_result}
	\vspace{-5mm}
\end{table}

\begin{figure}[ht!]
    \centering
    \begin{subfigure}{\linewidth}
        \includegraphics[width=.24\linewidth]{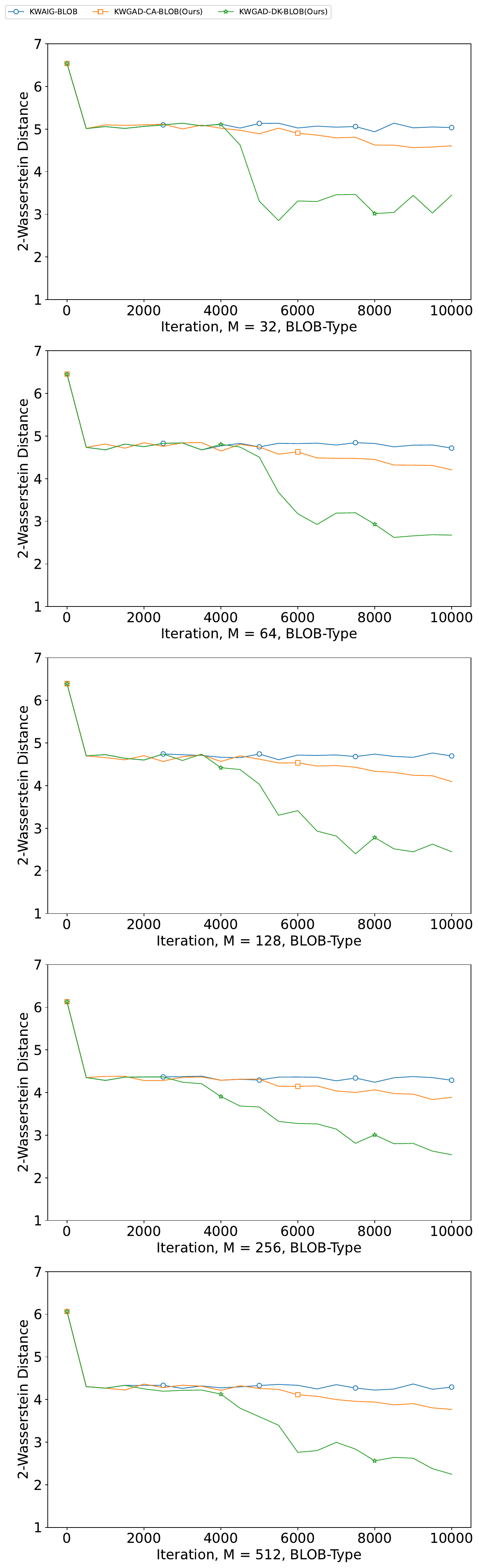}
        \includegraphics[width=.24\linewidth]{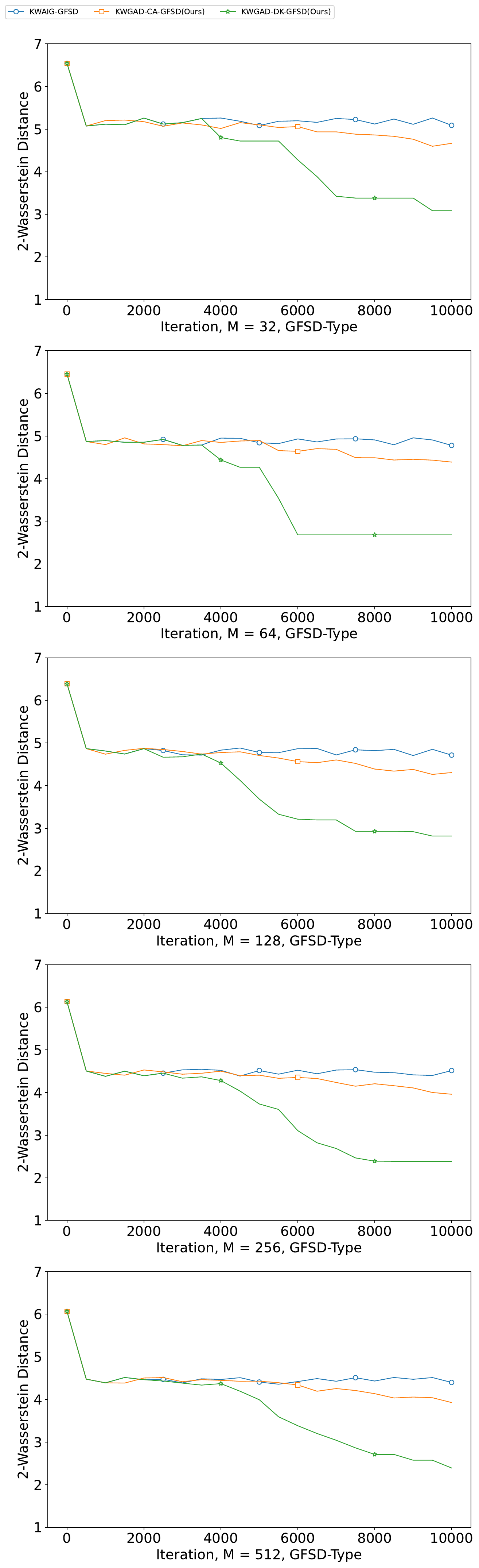}
        \includegraphics[width=.24\linewidth]{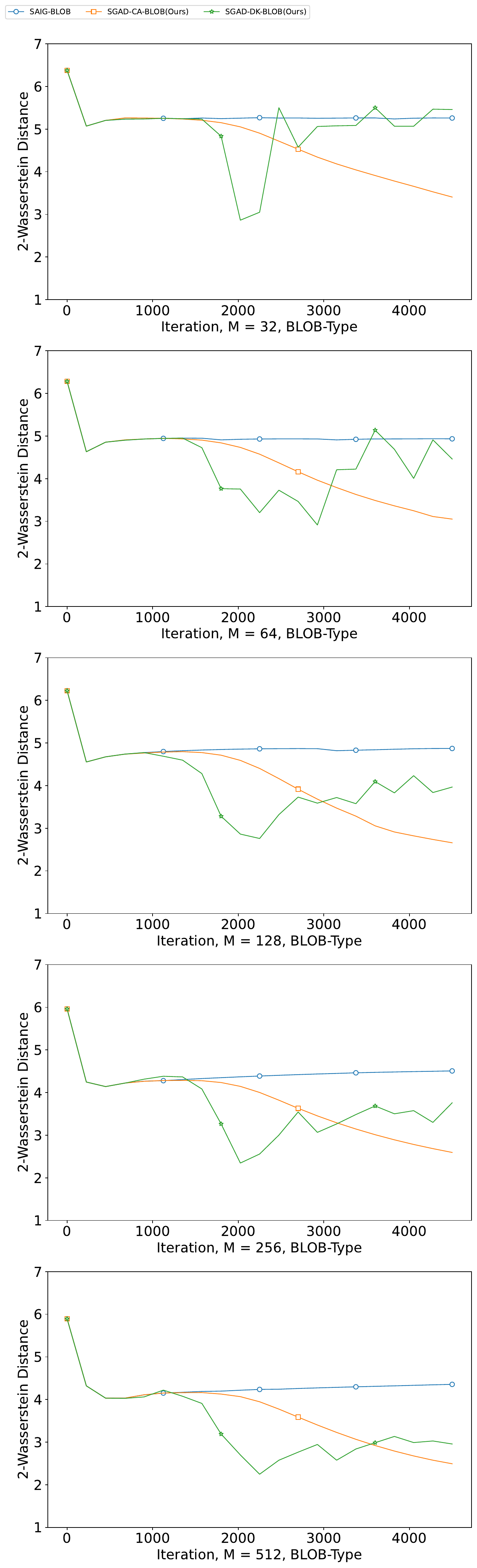}
        \includegraphics[width=.24\linewidth]{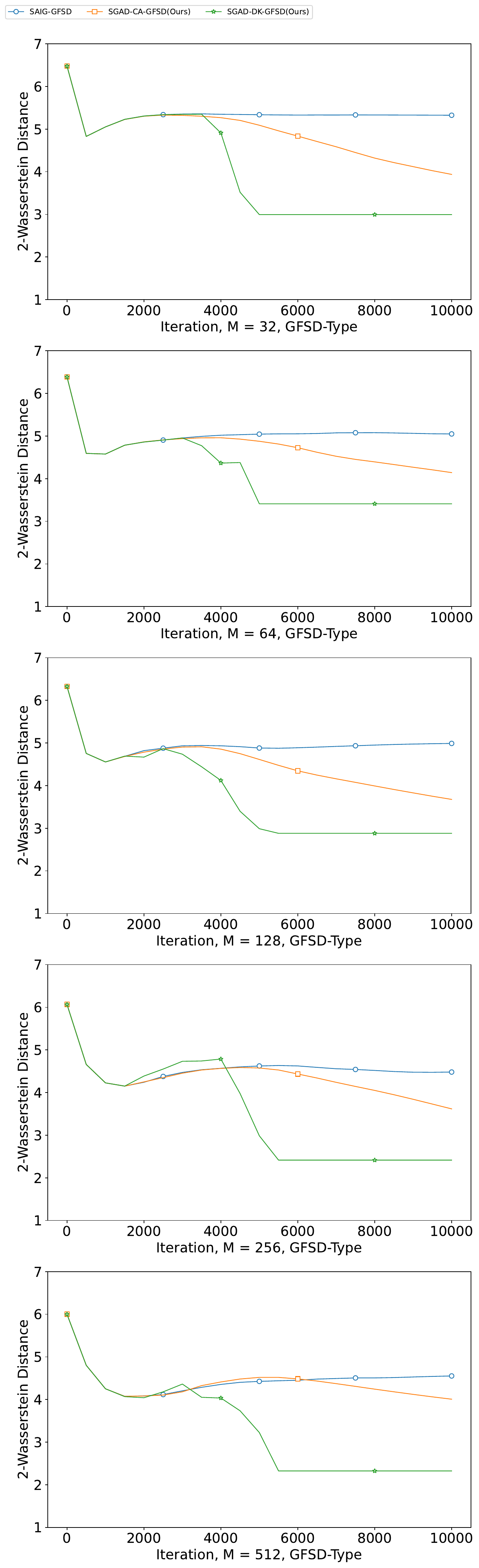}
    \end{subfigure}\qquad
    \caption{Averaged Test $W_2$ distance to the target w.r.t. iterations in the GMM task for algorithms(KW/S-Type).}
    \vspace{-5mm}
    \label{figure_mg_kws}
    \end{figure}

\paragraph{Additional Results for GP}
Here, we provide additional results of KW/S-type method for Gaussian Process Regression task in Table \ref{gp_w2_table_kw}.
The result is quite similar to the Wasserstein case, i.e., both the accelerated position update and the dynamic weight adjustment result in a decreased $W_2$ and GAD-PVI algorithms consistently achieve loweset $W_2$ to the target. Note that difference between DK type GAD-PVI and their fixed weight counterpart is not that obvious, due to the fact that the one-mode nature of GP greatly weaken the advantage of DK, i.e., transferring particles from low-probability region to distant high-probability area (e.g. among different local modes).

\begin{table}[tb]
	\centering

	\begin{tabular}{c|cc}
		\toprule
		\multirow{2}*{Algorithm} & \multicolumn{2}{c}{Smoothing Strategy}  \\
		~    & BLOB &GFSD \\

        \hline
        KWAIG & 	1.571e-01 $\pm $ 2.190e-04  & 2.146e-01 $\pm $ 8.608e-04  \\
        \NAME-KW-DK & 	1.566e-01 $\pm $ 1.791e-03  & 2.072e-01 $\pm $ 1.772e-03  \\
        \NAME-KW-CA & 	\textbf{1.341e-01 $\pm $ 1.494e-04}   & \textbf{1.991e-01 $\pm $ 4.415e-04} \\
        
        \hline
        SAIG & 	 1.570e-01 $\pm $ 3.791e-04 & 2.084e-01 $\pm $ 7.575e-03  \\   
        \NAME-S-DK & 	1.552e-01 $\pm $ 1.090e-02  & 2.012e-01 $\pm $ 4.960e-03  \\
        \NAME-S-CA & 	\textbf{1.236e-01 $\pm $ 2.872e-04}  & \textbf{1.691e-01 $\pm $ 3.499e-03 }  \\
		\bottomrule
	\end{tabular}
	\caption{Averaged $W_2$ distances after 10000 iterations 
    with different KW/S-type algorithms in the GP task with dataset LIDAR.}
	\label{gp_w2_table_kw}
	\vspace{-5mm}
\end{table}

\paragraph{Additional Results for BNN}
We provide additional test Negative Log-likelihood results for Bayesian Neural Network experiment for all algorithms in Table \ref{bnn_ll} and the test RMSE result under KW/S-Type algorithms in Table \ref{bnn_rmse_kws}.
The results demonstrate that the combination of the accelerated position updating strategy and the dynamically weighted adjustment leads to a lower NLL and RMSE under difference specific IFR space, and WGAD-PVI algorithms with CA usually achieve the best performance in Wasserstein case while KW/SGAD-PVI algorithms with CA or DK are comparable to each others.
Note that the position step-size of GAD-PVI are set to the value tuned for their fixed weight counterpart. Actually, if we retune the position step-size for all GAD-PVI algorithms, they are expected to achieve a better performance than existing result. 

\begin{table}[tb]
	\centering

	\begin{tabular}{c|cccc}
		\toprule
		\multirow{2}*{Algorithm} & \multicolumn{4}{c}{Datasets}  \\
		~    & Concrete & kin8nm & RedWine & space \\
		\hline
        ParVI-SVGD & 	1.738e+00 &1.160e+00  & 6.943e+00 & 2.739e+00 \\
        \hline
        \hline
		ParVI-BLOB & 	1.849e+00 & 1.122e+00  & 6.900e+00 & 2.742e+00 \\
		WAIG-BLOB &  1.710e+00	&1.093e+00  & 6.790e+00 & 2.638e+00 \\
		WNES-BLOB & 1.734e+00	& 1.065e+00  & 6.799e+00 & 2.679e+00 \\
        DPVI-DK-BLOB & 1.833e+00	&1.122e+00  & 6.848e+00 & 2.687e+00 \\
        DPVI-CA-BLOB & 1.837e+00	&1.093e+00  & 6.856e+00 & 2.685e+00 \\
        W\NAME-DK-BLOB(Ours) & 	1.703e+00 &1.065e+00  & 6.785e+00 & 2.602e+00 \\
        W\NAME-CA-BLOB(Ours) & 	\textbf{1.697e+00} &\textbf{1.048e+00}  & \textbf{6.782e+00} & \textbf{2.595e+00} \\

        \hline
        KWAIG-BLOB & 1.794e+00	&1.199e+00  & 9.322e-01 & 2.751e+00 \\
        KW\NAME-DK-BLOB(Ours) & 	\textbf{1.788e+00}&\textbf{1.169e+00}  & \textbf{9.300e-01} & 2.743e+00 \\
        KW\NAME-CA-BLOB(Ours) & 1.789e+00	&1.173e+00  & 9.304e-01 & \textbf{2.736e+00} \\
        
        \hline
        SAIG-BLOB & 1.832e+00	 &1.136e+00 & 9.349e-01 & 2.711e+00 \\
        S\NAME-DK-BLOB(Ours) & \textbf{1.821e+00}	&\textbf{1.068e+00}  & 9.292e-01 & 2.708e+00 \\
        S\NAME-CA-BLOB(Ours) & 1.824e+00	&1.119e+00  & \textbf{9.270e-01} & \textbf{2.642e+00} \\

        \hline
        \hline
		ParVI-GFSD & 	1.850e+00  & 1.122e+00 & 6.899e+00 & 2.742e+00\\
		WAIG-GFSD & 	1.724e+00  & 1.094e+00 & 6.785e+00 & 2.638e+00\\
		WNES-GFSD & 	1.740e+00  & 1.108e+00 & 6.781e+00 & 2.679e+00\\
        DPVI-DK-GFSD & 	1.836e+00  & 1.120e+00 & 6.812e+00 & 2.687e+00\\
        DPVI-CA-GFSD & 	1.836e+00  & 1.093e+00 & 6.857e+00 & 2.687e+00\\
        W\NAME-DK-GFSD(Ours) & 	1.722e+00  & 1.075e+00 & 6.780e+00 & \textbf{2.593e+00} \\
        W\NAME-CA-GFSD(Ours) & 	\textbf{1.720e+00}  & \textbf{1.050e+00} & \textbf{6.774e+00} & 2.597e+00 \\

        \hline
        KWAIG-GFSD & 	1.820e+00  & 1.199e+00 & 9.337e-01 & 2.759e+00\\
        KW\NAME-DK-GFSD(Ours) & 	1.813e+00  & 1.176e+00 & \textbf{9.305e-01} & \textbf{2.740e+00} \\
        KW\NAME-CA-GFSD(Ours) & 	\textbf{1.812e+00}  & \textbf{1.169e+00} & 9.319e-01 & 2.746e+00 \\
        
        \hline
        SAIG-GFSD & 	1.814e+00 & 1.116e+00 & 9.444e-01 & 2.782e+00\\
        S\NAME-DK-GFSD(Ours) & 	1.809e+00  & \textbf{1.062e+00} & 9.391e-01 & \textbf{2.555e+00}\\
        S\NAME-CA-GFSD(Ours) & 	\textbf{1.800e+00}  & 1.098e+00 & \textbf{9.359e-01} & 2.745e+00\\

		\bottomrule
	\end{tabular}
	\caption{Averaged Test $NLL$ distances for different ParVI methods in BNN task.}
	\label{bnn_ll}
	\vspace{-5mm}
\end{table}

\begin{table}[tb]
	\centering
	\begin{tabular}{c|cccc}
		\toprule
		\multirow{2}*{Algorithm} & \multicolumn{4}{c}{Datasets}  \\
		~    & Concrete & kin8nm & RedWine & space \\
        \hline
        KWAIG-BLOB & 	6.217e+00  & 8.159e-02 & 6.916e+00 &8.829e-02\\
        KW\NAME-DK-BLOB(Ours) & 	\textbf{6.207e+00}  & 8.069e-02 & 6.910e+00& 8.760e-02\\
        KW\NAME-CA-BLOB(Ours) & 	6.208e+00  & \textbf{8.058e-02} & \textbf{6.896e+00}& \textbf{8.753e-02}\\

        \hline
        SAIG-BLOB & 	 6.323e+00 & 7.936e-02 & 6.856e+00 &8.899e-02\\
        S\NAME-DK-BLOB(Ours) & 	6.293e+00  & \textbf{7.695e-02} & 6.828e+00& 8.867e-02\\
        S\NAME-CA-BLOB(Ours) & 	\textbf{6.269e+00}  & 7.872e-02 & \textbf{6.811e+00}& \textbf{8.783e-02}\\

        \hline
        KWAIG-GFSD & 	6.330e+00  & 8.158e-02 & 6.924e+00 & 8.948e-02\\
        KW\NAME-DK-GFSD(Ours) & 	6.274e+00  & 8.079e-02 & \textbf{6.854e+00} &\textbf{8.918e-02} \\
        KW\NAME-CA-GFSD(Ours) & 	\textbf{6.251e+00}  & \textbf{8.057e-02} & 6.917e+00 &8.927e-02 \\
        
        \hline
        SAIG-GFSD & 	 6.266e+00 & 7.864e-02 & 6.861e+00 & 8.998e-02\\
        S\NAME-DK-GFSD(Ours) & 	6.257e+00  & \textbf{7.679e-02} & 6.804e+00 & \textbf{8.638e-02}\\
        S\NAME-CA-GFSD(Ours) & 	\textbf{6.228e+00}  & 7.801e-02 & \textbf{6.793e+00} & 8.931e-02\\
		\bottomrule
	\end{tabular}
	\caption{Averaged Test $RMSE$ for different ParVI methods in BNN task.}
	\label{bnn_rmse_kws}
	\vspace{-5mm}
\end{table}

\paragraph{Results for GAD-KSDD}
Newly derived KSDD methods proposed by \cite{korba2020non} evolve particle system according to the direct minimizing the Kernel Stein Discrepancy(KSD) of particles w.r.t. the target distribution. The KSDD method is the first ParVI that introduce the dissimilarity functional whose first variation is well-defined at discrete empirical distribution, thus result in no approximation error when particles number is infinite. Though theoretically impressive, the experimental performance of KSDD is not satisfying, for they are more computationally expensive and have been widely reported to be less stable. Furthermore, KSDD are also reported to make particles easily trapped at saddle points and demanding for convexity of task and sensitive to parameters.

As shown in Figure \ref{figure_sg_ksdd}, we make simple Wasserstein experiments of KSDD type algorithms on SG task to illustrate that our GAD-PVI framwork is compatible with the KSD-KSDD approach. The bandwidth of KSDD are reported that should be carefully determined\cite{korba2020non,dpvi},  we follow the conventions in \cite{lu2019accelerating} and \cite{korba2021kernel}, and set the parameter $h$ via grid search. From Figure \ref{figure_sg_ksdd}, it can be observed that our GAD-PVI algorithms achieve the best result in different particles number settings. These illustrate our framework can corporate with this new smoothing approach. However, due to the limit of the KSDD itself, it is not realistic to fine tune parameters and conduct empirical studies of KSDD ParVI algorithms on complex tasks. Additionally, in this simple SG task, the final result of KSDD type algorithms is not competitive to other methods at all. So we exclude KSDD type experiments in GMM, GP and BNN.

\begin{figure}[h]
    \centering
    \begin{subfigure}{\linewidth}
        \includegraphics[width=.33\linewidth]{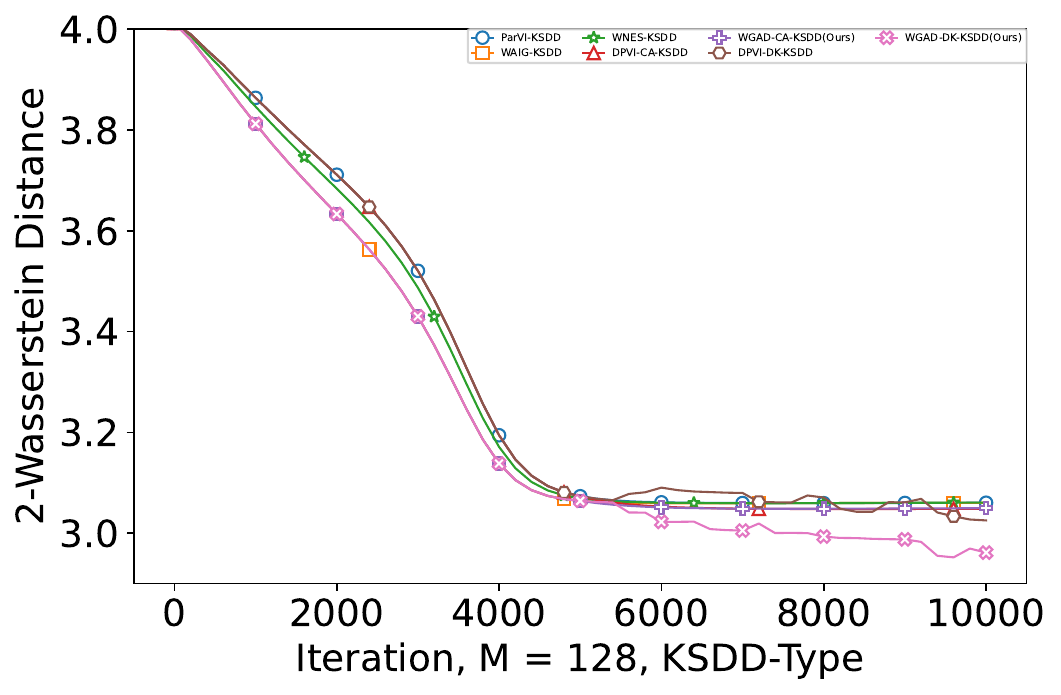}
        \includegraphics[width=.33\linewidth]{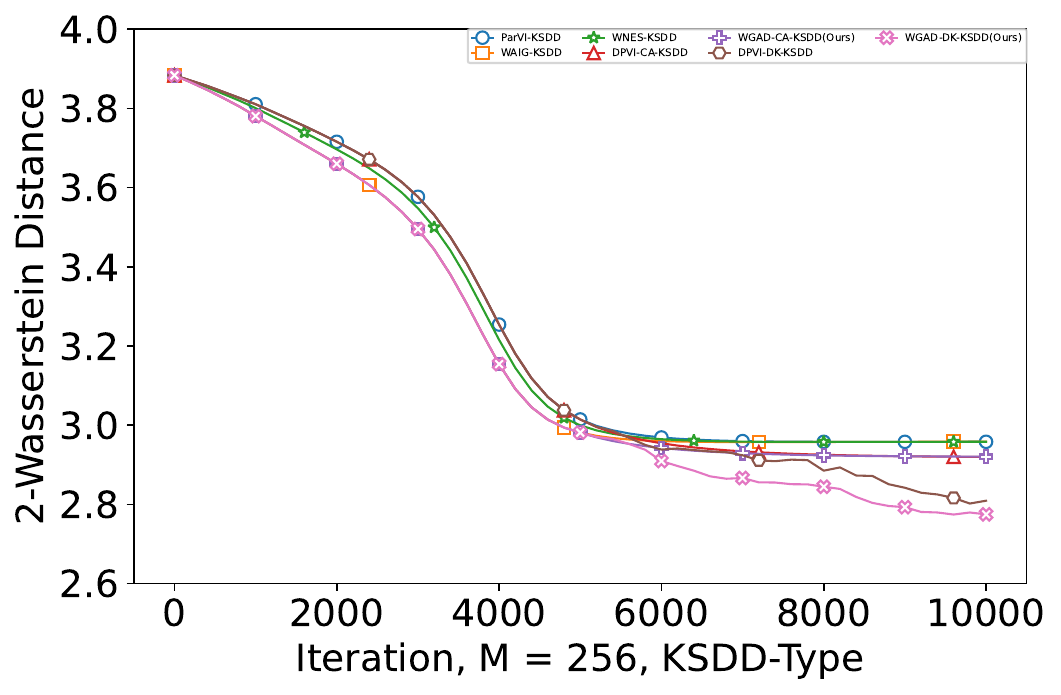}
        \includegraphics[width=.33\linewidth]{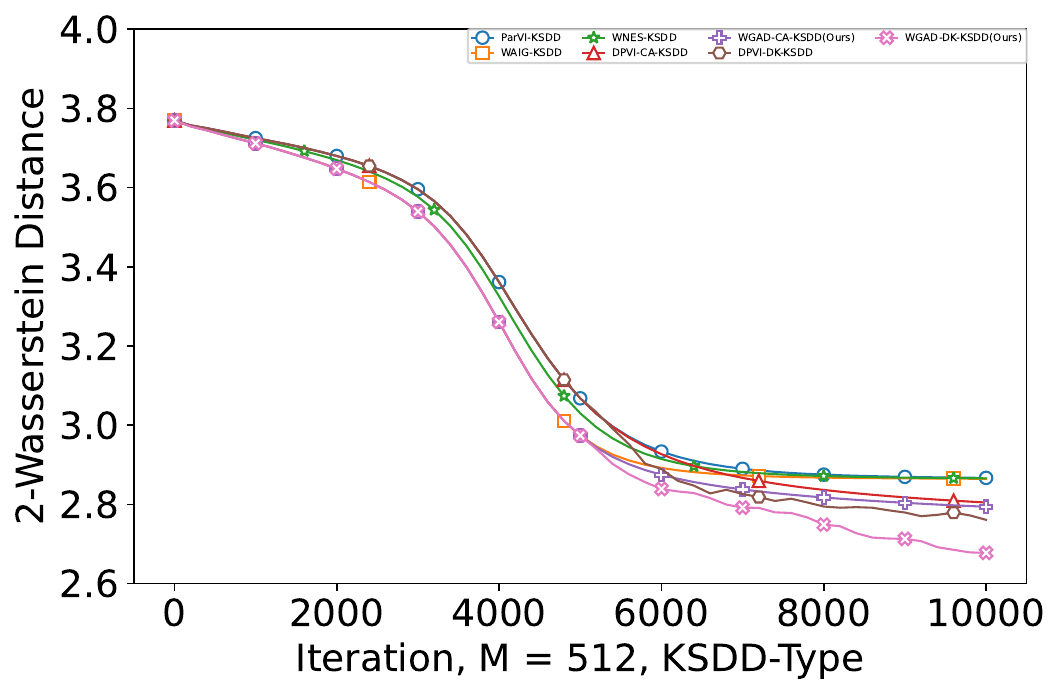}
    \end{subfigure}\qquad
    \caption{Averaged Test $W_2$ distance to the target w.r.t. iterations in the SG task for algorithms (WGAD-KSDD-Type).}
    \vspace{-5mm}
    \label{figure_sg_ksdd}
    \end{figure}

\end{document}